\documentclass{article}
\usepackage{style,times}

\usepackage{amsmath,amsfonts,bm}

\def\eqref#1{equation~\ref{#1}}

\def\1{\bm{1}}

\DeclareMathAlphabet{\mathsfit}{\encodingdefault}{\sfdefault}{m}{sl}
\SetMathAlphabet{\mathsfit}{bold}{\encodingdefault}{\sfdefault}{bx}{n}

\def\sJ{{\mathbb{J}}}

\newcommand{\R}{\mathbb{R}}

\DeclareMathOperator*{\argmin}{arg\,min}

\usepackage{url}
\usepackage{tikz}
\usetikzlibrary{arrows.meta, positioning, fit, backgrounds, calc}
\usepackage{amsmath, amssymb}
\usepackage{graphicx}
\usepackage{caption}
\usepackage{subcaption}
\usepackage{booktabs}
\usepackage{algorithm}
\usepackage{algpseudocode}
\usepackage{float}     
\usepackage{capt-of}
\usepackage{makecell}
\usepackage{enumitem}
\usepackage{subcaption}   
\usepackage{siunitx}
\usepackage{tabularx}
\newcolumntype{C}[1]{>{\centering\arraybackslash}p{#1}}
\usepackage{xcolor}
\usepackage{bbm}
\definecolor{cShared}{HTML}{4C72B0}
\definecolor{cAE}{HTML}{55A868}
\definecolor{cDiff}{HTML}{C44E52}
\definecolor{cPost}{HTML}{8172B2}

\def\cB{{\mathcal B}}

\def \I{\mathbf I}
\def \0{\mathbf 0}

\def \cU{\mathcal U}

\def \bb1{\mathbbm{1}}

\def\md{{\mathrm d}}
\def\sfW{{\mathsf W}}

\def\cZ{\mathcal{Z}}

\def\b0{\mathbf 0}

\def \Rb {\mathbb R}
\def \Eb {\mathbb E}

\def \Lip{\mathrm {Lip}}

\def \Pb{\mathbb P}

\def \I {\mathrm I}
\def \cT {\mathcal{T}}
\def \Id {\mathsf{Id}}
\def \sJ {\mathsf{J}}
\def \sfd {\mathsf{d}}
\def \bP {\mathbb{P}}
\def \bQ {\mathbb{Q}}
\def \sfD {\mathsf{D}}
\def \sfE {\mathsf{E}}
\def \sfV{\mathsf{V}}

\def \sfG{\mathsf{G}}
\def \sfH{\mathsf{H}}
\def \sfP{\mathsf{P}}

\def\sfMMD{\mathsf{MMD}}

\usepackage{amsthm}
\newtheorem{theorem}{Theorem}[section]
\newtheorem{lemma}[theorem]{Lemma}

\newtheorem{example}[theorem]{Example}

\newtheorem{corollary}[theorem]{Corollary}

\usepackage{hyperref}

\title{Generative Priors Conditioned on Natural Language for Bayesian Inversion in PDEs}

\author{Pengyu Zhang, Mark Girolami,  Arnaud Vadeboncoeur \\
Department of Engineering\\
University of Cambridge\\
Cambridge, CB3 0FA, UK \\
\texttt{\{pz281, mag92, av537\}@cam.ac.uk} \\
}

\iclrfinalcopy
\begin{document}

\maketitle

\begin{abstract}
Inferring quantities of interest (QoI) from data is a central task in Science and Engineering.
In such contexts, we often have access to both quantitative data and qualitative data. Quantitative data may be represented by noisy sensor measurements, simulation data, re-analysis data; qualitative data may be in the form of text descriptions of experimental setups, expected experiment outcomes, and human-perceived system behaviours. The task we address in this paper is the following. Given a training set of paired qualitative text and quantitative QoI data, we learn to exploit the inherent correlation between the two modalities to learn a highly informative data-driven natural-language-conditional Bayesian prior, such that when presented with a new physical system, we can coherently combine (i) the training dataset, (ii) qualitative text describing the new system, and (iii) a small number of noisy sensor readings from that new system, to perform inference and uncertainty quantification (UQ) over the QoI. To achieve this task, we develop two parallel approaches, one uses conditional diffusion and the other conditional autoencoders, and compare both against classical Bayesian methodology, unconditional generative models and deterministic supervised methods.
Each approach has specific strengths and tradeoffs; conditional autoencoder offers theoretical tractability, allows for fast posterior sampling, and provides better-calibrated UQ,
whereas conditional diffusion is explored for greater expressiveness and capturing complex posteriors with irregular QoI fields.
The approach is tested on the steady-state heat equation, damped Helmholtz equation, and UK weather reanalysis data.
\end{abstract}

\section{Introduction}

Recovering a full field of quantities of interest (QoI) from a small number of noisy measurements is a ubiquitous problem setup in science and engineering applications, such as in weather prediction \citep{kalnay2003atmospheric, bannister2017review, manshausen2025generative}, medical image reconstruction \citep{hammernik2018learning, arridge2019solving, song2021solving}, and structural health monitoring \citep{girolami2021statistical, febrianto2022digital}.  Such problems are typically severely ill-posed as the observation process is not a one-to-one map. Leveraging data helps attenuate the ill-posedness of these tasks. 
A direct approach is to learn the inverse map in a supervised fashion, regressing the full field directly from observations \citep{erichson2020shallow, fukami2021global, santos2023development}. Such methods are fast at test time, however, the map they learn is tied to the observation process it was trained under; they are also deterministic, whereas uncertainty quantification (UQ) is essential for decision making. 
Bayesian inference \citep{stuart2010inverse, cotter2010approximation, bui2014solving} provides UQ by construction: a likelihood function measuring data fit after an observation process and a prior encoding prior beliefs on QoI data, together give a posterior distribution. 

One of the challenges in Bayesian inference is prior specification, with very sparse observations the prior determines most of the reconstruction, and one that is too vague leaves the posterior diffuse with relatively poor accuracy. 
Prior elicitation \citep{o2006uncertain, mikkola2024prior} converts subjective domain knowledge into a prior distribution. Large Language Models (LLMs) \citep{brown2020language} have recently been prompted to supply that knowledge in place of an expert \citep{capstick2024autoelicit, selby2025had, huang2025llm}. However, such methods are confined to a few scalar parameters and are correspondingly insufficient to capture the complex spatial structures in physical fields. 
A promising direction is to replace hand-crafted priors with learned generative models trained on historical dataset \citep{daras2024survey}. Such priors are expressive and scalable, and are now used widely in imaging \citep{bora2017compressed, chung2022diffusion, feng2023score} and in physical setups \citep{shu2023physics, shysheya2024conditional, akyildiz2025efficient,glyn2025primer,vadeboncoeur2026efficient}.

System-specific information can also improve the generative prior, such as a system's geometry \citep{vadeboncoeur2025geometric}. Domain expert knowledge can also serve as additional information. In many practical scenarios, prior information can be naturally expressed through human knowledge in natural language, such as a description of the experimental setup (``a strong forcing acts in the north-east''), or a human-perceived physical condition (``the north-west is cool with a light north-westerly breeze''). Such statements are imprecise, but they carry high-level physical structural patterns.

In this work, we propose a methodology that encodes such domain knowledge via text into a data-driven physical prior. Trained on paired text and field data, a conditional generative model learns the correlation between the two modalities, so that a qualitative description of a new system induces an informative prior over its field. The learned prior constrains the search space, which is further used in posterior inference from sparse observational data to enable high-fidelity field reconstruction with UQ. Contributions in this paper are as follows.

\begin{enumerate}[label=C\arabic*., ref=C\arabic*, leftmargin=*]
    \item We propose to learn a text-conditional prior on physical fields using generative models, combining qualitative descriptions with sparse noisy quantitative measurements to perform posterior inference over quantities of interest.

    \item We provide a complete pipeline: fine-tuning a pretrained text encoder, training the text-conditional prior, and posterior inference.

    \item Two parallel approaches based on autoencoders and diffusion models are developed. With autoencoders, Bayesian inference is performed in a low-dimensional learned latent space, where, given the pushforward form of the prior, the posterior admits asymptotically exact sampling; with diffusion models, posterior inference relies on approximate likelihood guidance.

    \item We analyse the advantages and tradeoffs of two models. For autoencoders we provide theoretical guarantees and error bounds, and obtain more consistently calibrated posteriors at a lower sampling cost, while diffusion models offer greater expressiveness on irregular fields.

    \item We test the methods on the steady-state heat equation, the damped Helmholtz equation, and UK weather reanalysis data, comparing against Gaussian processes, unconditional data-driven prior models and supervised models, with both accuracy and posterior coverage reported.
\end{enumerate}

Section \ref{sec:setup} illustrates our problem setup; Section \ref{sec:related} reviews related works; Section \ref{sec:methodology} presents the proposed pipeline of two parallel approaches; Section \ref{sec:implementation} provides additional implementation details; Section \ref{sec:numerics} reports and discusses results on three experiments; Section \ref{sec:conclusion} draws final conclusions.

\subsection{Problem Setup} \label{sec:setup}

Consider a collection of $N$ high-fidelity experimental setups where each experiment produces quantitative data $u \in\cU$, which is associated with a set of physical parameters $\theta$,\footnote{The relation between $u$ and $\theta$ is left unspecified: $\theta$ is not a fixed representation of the system, and could take any form such as a parameterisation of the experimental setup or a coarse property of $u$ itself.} of which only vague and qualitative information is available in the form of the text $s \in \mathcal{S}$, where $\mathcal{S}$ is a set of text sequences. We think of $s$ as being obtained from a human expert relating an observed physical quantity to a text description. Importantly, $s$ is an imprecise and random specification of information of relevance to $u$. The set \smash{$\{u^{(n)}, s^{(n)}\}_{n=1}^N$} will be our training dataset. One may then Euclideanise text using a map $\mathsf{G}:\mathcal{S}\rightarrow \cT$, where $\cT$ is $d_\tau$-dimensional Euclidean space. In our experiments, $\sfG$ is a fine-tuned text encoder. Table~\ref{tab:notation} summarises our notation, with an example of each symbol.

Now consider a new, partially observed experiment, of which we observe a small amount of quantitative data, \smash{$y^{(o)} \in \Rb^{d_y}$} for $o>N$, and text description \smash{$s^{(o)}$}. We wish to leverage the dataset \smash{$\{u^{(n)}, s^{(n)}\}_{n=1}^N$} in the form of a text-conditional prior, to help infer \smash{$u^{(o)}$ from $(y^{(o)}, s^{(o)})$} where
\begin{equation}
    \label{eq:generative_model}
    y^{(o)} = \sfH^{(o)} u^{(o)}+\xi^{(o)}, \quad \xi^{(o)} \sim \eta,
\end{equation}
with \smash{$\sfH^{(o)}$} representing the system-dependent observation operator, $\eta$ a known noise
distribution, taken to be \smash{$\mathcal{N}(0, \sigma_y^2 \,\I_{d_y})$} in our experiments. During inference, we observe \smash{$y^{(o)}, s^{(o)}$}, we know \smash{$\sfH^{(o)}, \eta$}, and do not know \smash{$u^{(o)},\xi^{(o)}$}. Example \ref{ex:heat_example} is a toy example of our problem setup.

\begin{example}
    \label{ex:heat_example}
    The quantity of interest which we would like to estimate, $u^{(o)}$, is the field which describes the heat distribution across the surface of a plate; the observed quantitative data, \smash{$y^{(o)}$}, is a small number of noisy point-wise sensor measurements; the unknown physical parameter \smash{$\theta^{(o)}$} is the temperature at the boundaries of the plate, of which we only know a vague, qualitative description such as \texttt{``The left boundary is quite hot, while the bottom boundary is cold''}, etc., this is \smash{$s^{(o)}$}. The task: from a corpus of data, \smash{$\{u^{(n)}, s^{(n)}\}_{n=1}^N$}, learn a text-conditioned prior for $u$, and use this to infer \smash{$u^{(o)}$ from $(y^{(o)},s^{(o)})$}.
\end{example}

\subsection{Related Works} \label{sec:related}

\paragraph{Generative Priors for Inverse Problems.}
Generative models are widely used as data-driven priors in inverse problems. One approach uses Generative Adversarial Networks (GANs) \citep{goodfellow2020generative}, learning a generator that maps a low-dimensional latent vector to the field and performing inference in that latent space. Approaches include gradient descent \citep{bora2017compressed} and Markov chain Monte Carlo (MCMC) sampling \citep{patel2021gan, patel2022solution, meng2022learning}. GABI \citep{vadeboncoeur2025geometric} follows a similar idea with a graph autoencoder. There exists theory about the relation between learned priors and the resulting posteriors~\citep{hosseini2026error}.
Diffusion models \citep{ho2020denoising,song2020score} are the dominant alternative as a generative prior, with many training-free schemes for conditioning on measurements at inference time \citep{song2021solving, kawar2022denoising, feng2023score, mardani2023variational, feng2024variational}. In our diffusion approach we follow \citet{chung2022diffusion} and approximate the likelihood score via Tweedie's formula \citep{efron2011tweedie, song2023pseudoinverse, chung2023decomposed, rozet2023score, boys2024tweedie}, which yields a biased posterior approximation; asymptotically exact SMC alternatives \citep{wu2023practical,cardoso2024monte} instead require maintaining an ensemble of particles at a cost.
In all of these works the prior itself is unconditional, and we propose to condition on
text descriptions. Furthermore, mostly developed for imaging, these methods are also seldom
assessed for UQ, which in physical settings is of paramount importance.

\paragraph{Conditional Generative Models.}
Block triangular maps \citep{zech2022sparse,zech2022sparse2} provide a principled construction for conditional sampling, realised with monotone GANs \citep{baptista2024conditional} and Wasserstein autoencoders \citep{tolstikhin2018wasserstein,al2026conditional}. Our autoencoder method draws inspiration from these works. 
Furthermore, conditional diffusion models have shown superior performance in text-to-image generation \citep{rombach2022high, ramesh2022hierarchical, saharia2022photorealistic}, establishing the paradigm of conditioning a generative model on text embeddings, though these target forward generation. A few works bring text into inverse problems \citep{chung2023prompt, kim2023regularization}, which use a pretrained latent diffusion model (LDM) \citep{rombach2022high} as a prior and introduce the text at inference time as a guidance or regularisation. Both target imaging rather than physical fields.

\paragraph{Natural Language and Physical Systems.}
In the scientific domain, although generative models have been widely applied to forward and inverse problems in physical systems \citep{shu2023physics, huang2024diffusionpde, shysheya2024conditional, bastek2024physics, jiang2024resolution, jacobsen2025cocogen, zheng2025inversebench, shan2026regularization}, the integration of natural language with physical simulation remains underexplored. Emerging multimodal approaches including PROSE \citep{liu2024prose} and the Multimodal PDE Foundation Model \citep{negrini2025multimodal} combine numerical and symbolic or textual data to approximate solution operators for ODEs and PDEs, and Text2PDE \citep{zhou2025text2pde} uses LDM to perform physical simulations from linguistic prompts, all targeting forward simulation. On the inverse side, \citet{saito2026language} recently condition a Darcy-flow solver on geological descriptions encoded by a sentence embedding, using a deterministic U-Net without UQ. Learning a natural-language-conditioned prior to guide the inverse recovery of physical fields, in a form that supports UQ, remains open.  

\section{Methodology} \label{sec:methodology}

This section presents our full pipeline, from fine-tuning a pretrained text encoder (Step 1) to conditional prior learning and posterior inference
(Steps 2 and 3). Step 1 is shared by both models; Steps 2 and 3 are then described for autoencoder (AE) and score-based diffusion model (SD) in turn.

\subsection{Step 1: Fine-Tuning} 

As discussed in Section \ref{sec:setup}, the training data consists of pairs of physical fields and text descriptions \smash{$\{u^{(n)}, s^{(n)}\}_{n=1}^N$} drawn from a joint distribution $\bP_{u,s}$. To make the text usable by a model we encode it with \smash{$\sfG^{\phi}:\mathcal{S}\rightarrow\cT$}, mapping the text to a fixed-dimensional latent representation $\tau$. \smash{This induces a joint distribution over fields and text embeddings, $\bP_{u,\tau} =(\Id\times\sfG^\phi)_\#\bP_{u, s}$}. \footnote{Notational conventions are collected in
Appendix~\ref{appen:notation}.} We take $\sfG^{\phi}$ to be a pretrained sentence transformer \citep{reimers2019sentence}.\footnote{We use \textit{all-MiniLM-L6-v2} which provides text embeddings of 384 dimensions from \url{https://huggingface.co/sentence-transformers/all-MiniLM-L6-v2}} However, proximity in $\cT$ reflects semantic similarity rather than physical similarity, so we fine-tune the encoder on the training dataset before training prior models. Parameters $\phi$ in the encoder are obtained by
\begin{equation}
    \label{eq:fine-tune}
    \phi^\star \in \argmin_\phi\, \Eb_{(u,s), (u', s')\sim\bP_{u,s}}{\left| \mathsf{score}_{\text{field}} (u, u') - \mathsf{score}_{\text{text}}(\sfG^\phi(s),\sfG^\phi(s'))\right|^2},
\end{equation}
a cosine-similarity loss in which $\mathsf{score}_{\text{text}}(\tau,\tau ') = ({\tau^\top \tau '})/({\|\tau\|\|\tau '\|})$ is the cosine similarity between text embeddings. $\mathsf{score}_{\text{field}}(\bullet,\bullet)$ is chosen appropriately for each physical setup, detailed in Appendix \ref{appen:results}, in each case it returns a value in $[0,1]$ quantifying how similar two physical fields are.

\subsection{Autoencoder}

\subsubsection{Step 2: Learning Conditional Prior -- Autoencoder}

A map 
$\sfP:\cZ\times\cT\rightarrow\cU\times\mathcal{X}$  with
$\sfP(z,\tau)=\big(\sfP_\cU(z,\sfP_\mathcal{X}(\tau)),\,\sfP_\mathcal{X}(\tau)\big)$
is block-triangular~\citep{baptista2024conditional, zech2022sparse, zech2022sparse2}. Our autoencoder consists of an encoder and a decoder,
\begin{equation}\label{eq:enc_dec}
    \sfE:\cU\times\cT\rightarrow\cZ,
    \qquad
    \sfD:\cZ\times\cT\rightarrow\cU,
\end{equation}
both taking the text embedding as input; these are the networks that are trained and evaluated. Distributions, however, are compared on the joint space, for which we write the maps
\begin{equation}\label{eq:lifts}
    \widetilde{\sfE}:\cU\times\cT\rightarrow\cZ\times\cT,
    \qquad
    \widetilde{\sfD}:\cZ\times\cT\rightarrow\cU\times\cT,
\end{equation}
so that $\widetilde{\sfE}(u,\tau)=(\sfE(u,\tau),\tau)$ and $\widetilde{\sfD}(z,\tau)=(\sfD(z,\tau),\tau)$, with no new parameters because the text is passed through unchanged. They have the form of a block-triangular map with $\sfP_\mathcal{X} =\Id_\cT $. The maps $\sfE, \sfD$ are learned by minimising
\begin{equation}\label{eq:ae_objective}
    \sJ_{\mathsf{AE}}(\sfD,\sfE) = \Eb_{u,\tau\sim\bP_{u,\tau}}
    \big\|u-(\sfD\circ\widetilde{\sfE})(u,\tau)\big\|^2
    + \sfd\big(\widetilde{\sfE}_\#(\bP_{u,\tau}),\ \bQ_z\times\bP_\tau\big),
\end{equation}

where $\sfd(\bullet,\bullet)$ is a discrepancy on
probability measures over $\cZ\times\cT$, and we choose maximum mean discrepancy (MMD), detailed in Appendix~\ref{appen:mmd}. $\bQ_z$ is a probability measure on the latent space, taken to be $\mathcal{N}(0,\I)$. 
\begin{theorem}\label{thm:obj_zero}
Let $\cU,\cT,\cZ$ be separable Banach spaces, let $(\cU\times\cT,\cB(\cU\times\cT),\bP_{u,\tau})$ and $(\cZ,\cB(\cZ),\bQ_z)$ be probability spaces, and let $\sfD,\sfE$ in (\ref{eq:enc_dec}) be continuous. If $\sJ_{\mathsf{AE}}(\sfD,\sfE)=0$, then $\bP_{u\mid\tau}=\sfD(\bullet,\tau)_\#\bQ_z$ for $\bP_\tau$-a.e.\ $\tau$.
\end{theorem}
Theorem \ref{thm:obj_zero} shows that if the loss (\ref{eq:ae_objective}) is brought to zero, the AE method recovers the correct conditional distribution $\bP_{u|\tau}$ from the joint data distribution $\bP_{u,\tau}.$ Proof is provided in Appendix \ref{appen:thm_2_1}. 
\begin{theorem}\label{thm:sqWass_bound} We have the following bound on the squared 2-Wasserstein distance between the data distribution and the generative model:
    \begin{align}
        \label{eq:theorem2}
        \begin{split}
        \sfW_2^2(\bP_{u,\tau}, \widetilde{\sfD}_\#(\bQ_z\times\bP_\tau))\leq &2\Eb_{u,\tau\sim\Pb_{u,\tau}}\|u-(\sfD\circ \widetilde{\sfE})(u,\tau)\|^2    \\ &+  
        2({\|\sfD\|_{\Lip}^2+ 1})\;\sfW_2^2(\widetilde{\sfE}_\#(\Pb_{u,\tau}), \bQ_z\times\bP_\tau).
        \end{split}
    \end{align}
\end{theorem}
This bounds the error between the joint distributions by terms related to ones that appear in our loss (\ref{eq:ae_objective}). The next results bridge the gap between MMD and Wasserstein, for distributions with compact support. Proofs are available in Appendix \ref{appen:thm_2_2} and \ref{appen:coro_2_3}.
\begin{corollary}\label{cor:mmd_bound}
We have the following bound for latent distributions with compact support:  
    \begin{align}
        \begin{split}
        \sfW_2^2(\bP_{u,\tau}, \widetilde{\sfD}_\#(\bQ_z\times\bP_\tau))\leq &2\Eb_{u,\tau\sim\Pb_{u,\tau}}\|u-(\sfD\circ \widetilde{\sfE})(u,\tau)\|^2 \\ &+2 C({\|\sfD\|_{\Lip}^2+ 1})\sfMMD_{\mathrm{RBF}}^{2\alpha}(\widetilde{\sfE}_\#(\Pb_{u,\tau}), \bQ_z\times\bP_\tau),
        \end{split}
    \end{align}
    for some $C>0,\alpha\in(0,1]$, where $\alpha$ depends critically on $\mathrm{dim}(\cZ\times\cT)$.
\end{corollary}

\subsubsection{Step 3: Bayesian Inference -- Autoencoder}
After learning the conditional prior $\bP_{u|\tau}$, given observations $y$ from a new system obtained from (\ref{eq:generative_model}), the conditional autoencoder approach allows for approximate sampling of the posterior $\bP_{u|y,\tau}$ by first sampling the posterior in the latent space $\bP_{z|y,\tau}$ using standard MCMC tools, such as NUTS \citep{hoffman2014no}, and then pushing the obtained samples through the decoder. This result follows directly from~\cite{vadeboncoeur2025geometric}. Algorithm~\ref{alg:ae} shows the inference process for AE. 
\begin{lemma}
    \label{lemma:posterior}
    Let $\bP_{u|y,\tau}$ be the posterior on $\cU$, with likelihood proportional to $\exp (-\Phi(u;y))$, \footnote{$\Phi(u;y)=\tfrac{1}{2\sigma_y^2}\lVert y-\sfH u\rVert_2^2$ is the potential associated with the likelihood in (\ref{eq:generative_model}).} and prior $\bP_{u\mid\tau}=\sfD(\bullet,\tau)_\#\bQ_z$ as in Theorem \ref{thm:obj_zero}. Then, the posterior $\bP_{u|y,\tau}=\sfD(\bullet, \tau)_\#\bP_{z|y,\tau}$ with the latent posterior $\md\bP_{z|y,\tau}(z) \propto \exp(-\Phi(\sfD(z, \tau); y))\md\bQ_z(z).$ 
\end{lemma}

\begin{figure}[t]
\begin{minipage}[t]{0.49\textwidth}
\begin{algorithm}[H]
\footnotesize
\caption{Posterior sampling --- autoencoder}
\label{alg:ae}
\begin{algorithmic}[1]
\Require trained $\sfD,\sfG^\phi$; text $s$, observation $y$, operator $\sfH$,
         noise $\sigma_y$; $M_o$ NUTS warmup steps, $M$ NUTS posterior draws
\State $\tau\gets\sfG^\phi(s)$
\State $z_0\sim\bQ_z$
\For{$i=1$ \textbf{to} $M_o+M$}
    \State $z_i\gets\textsc{NUTS-Step}\big(\bP_{z\mid y,\tau},\ z_{i-1}\big)$
       \Statex \hfill $\triangleright$ posterior target defined in Lemma~\ref{lemma:posterior}
\EndFor
\For{$m=1$ \textbf{to} $M$}
    \State $u^{(m)}\gets\sfD(z_{M_o+m},\tau)$
\EndFor
\State \Return $\{u^{(m)}\}_{m=1}^{M}$
\end{algorithmic}
\vspace{1.12\baselineskip}
\end{algorithm}
\end{minipage}\hfill
\begin{minipage}[t]{0.49\textwidth}
\begin{algorithm}[H]
\footnotesize
\caption{Posterior sampling --- diffusion}
\label{alg:dps}
\begin{algorithmic}[1]
\Require trained $\sfV,\sfG^\phi$; text $s$, observation $y$, operator $\sfH$, noise $\sigma_y$; grid $T=t_L>\dots>t_0=\varepsilon$, cap $c$, $M$ posterior draws; write $u_i$ for $u_{t_i}$
\State $\tau \gets \sfG^\phi(s)$
\For{$m=1$ \textbf{to} $M$}
  \State $u_L\sim\mathcal{N}(0,\I)$
  \For{$i=L$ \textbf{to} $1$}
    \State $\hat u_0\gets \textsc{Tweedie Estimate}(u_i,t_i,\tau)$
    \State $\delta\gets \textsc{Likelihood Guidance}(\hat u_0,y,c)$
    \State $u_{i-1}\gets \textsc{SDE Step}(u_i,t_i,\tau,\delta)$
  \EndFor
  \State $u^{(m)} \leftarrow \textsc{Tweedie Estimate}(u_0, t_0, \tau)$
\EndFor
\State \Return $\{u^{(m)}\}_{m=1}^{M}$
\end{algorithmic}
\end{algorithm}
\end{minipage}
\end{figure}

\subsection{Score-Based Diffusion Model}

\subsubsection{Step 2: Learning Conditional Prior -- Score-Based Diffusion Model}

The second construction targets $\bP_{u\mid\tau}$ via diffusion. We adopt the variance-preserving (VP)-SDE of \citet{song2020score}, the continuous-time limit of DDPM \citep{ho2020denoising}, which progressively perturbs a field $u_0\sim  \bP_{u|\tau}$ into noise,
\begin{equation}\label{eq:fwd_sde}
    \md u = -\tfrac{1}{2}\beta(t)\,u\,\md t + \sqrt{\beta(t)}\,\md w,
    \qquad t\in[0,T],
\end{equation}
where $\beta(t)>0$ is the noise schedule and $w$ a Wiener process on $\cU$. Linearity of (\ref{eq:fwd_sde}) gives a Gaussian perturbation kernel
\smash{$p_{0t}(u_t| u_0)=\mathcal{N}(\sqrt{\bar\alpha(t)}u_0,(1-\bar\alpha(t))\I)$ with $\bar\alpha(t)=\exp(-\int_0^t\beta(r)\md r)$}, so that $\bar\alpha(T)\approx0$ and $u_T \sim p_T\approx\mathcal{N}(0,\I)$. The forward process does not depend on $\tau$ and the text enters only through the reverse process. Samples from $\bP_{u|\tau}$ are obtained by integrating the reverse-time SDE from $u_T$ down to $t=0$, following
\begin{equation}\label{eq:rev_sde}
    \md u = \left[-\tfrac{1}{2}\beta(t)\,u - \beta(t)\,\nabla_{u_t}\log p_t(u_t|\tau)\right]\md t + \sqrt{\beta(t)}\,\md\bar{w},
\end{equation}
which requires the conditional score $\nabla_{u_t}\log p_t(u_t|\tau)$. We approximate it with a network $\sfV:\cU\times[\varepsilon,T]\times\cT\rightarrow\cU$ trained by denoising score matching, with objective
\begin{equation}\label{eq:dsm}
    \sJ_{\mathsf{SD}}(\sfV)=\Eb_{t\sim\mathrm{Unif}[\varepsilon,T]}\Big\{\lambda(t)\,
    \Eb_{u_0,\tau \sim\bP_{u,\tau}}\ \Eb_{u_t\mid u_0}
    \Big[\big\|\sfV(u_t,t,\tau)-\nabla_{u_t}\log p_{0t}(u_t| u_0)\big\|_2^2\Big]\Big\},
\end{equation}
where $\varepsilon \simeq 0$ is a small positive constant to avoid exploding. In practice, we train the equivalent $\epsilon$-parameterisation $\epsilon_\sfV=-\sigma(t)\,\sfV$ with $\lambda(t)=\sigma^2(t)=1-\bar\alpha(t)$, which avoids the variance blow-up of the score target as $t\to0$ \citep{ho2020denoising, zhang2022fast,rozet2023score}, but for notational convenience we retain the score notation $\sfV$ throughout. At inference, (\ref{eq:rev_sde}) is discretised into $L$ steps via the Euler-Maruyama scheme on a uniform temporal grid $T=t_L>\cdots>t_0=\varepsilon$, and we write $u_i$ for $u_{t_i}$.

\subsubsection{Step 3:  Bayesian Inference -- Diffusion Posterior Sampling}

To sample from $\bP_{u\mid y,\tau}$, we replace the prior score in (\ref{eq:rev_sde}) by the posterior score using Bayes' rule
\begin{equation}\label{eq:post_score}
    \nabla_{u_t}\log p_t(u_t| y,\tau)
    = \underbrace{\nabla_{u_t}\log p_t(u_t|\tau)}_{\approx\;\sfV(u_t,t,\tau)}
    + \nabla_{u_t}\log p_t(y| u_t, \tau).
\end{equation}
The first term is the trained network; the second is intractable, since \smash{$p_t(y| u_t,\tau)=\Eb_{u_0| u_t,\tau}\!\left[p(y| u_0)\right]$} requires marginalising the denoising posterior $p(u_0| u_t,\tau)$, and DPS \citep{chung2022diffusion} evaluates the likelihood at its Tweedie point estimate $\hat u_0(u_t,\tau)=\Eb[u_0| u_t,\tau]$, available in closed form as,
\begin{equation}\label{eq:tweedie}
    \hat u_0(u_t,\tau)=\tfrac{1}{\sqrt{\bar\alpha(t)}}
    \big(u_t+\left(1-\bar\alpha(t)\right)\,\sfV(u_t,t,\tau)\big),
\end{equation}
so that, under the Gaussian likelihood given in (\ref{eq:generative_model}),
\begin{equation}\label{eq:dps_grad}
    \nabla_{u_t}\log p_t(y| u_t,\tau)\;\approx\;
    \nabla_{u_t}\log p\big(y|\hat u_0(u_t,\tau)\big)
    \;=\;-\tfrac{1}{2\sigma_y^2}\,\nabla_{u_t}\big\|y-\sfH\,\hat u_0(u_t,\tau)\big\|_2^2.
\end{equation}
Substituting (\ref{eq:dps_grad}) into (\ref{eq:post_score}) and discretising the resulting reverse SDE gives our sampler. Standard DPS replaces the factor \smash{$1/(2\sigma_y^2)$} in (\ref{eq:dps_grad}) by a tuned step
size \smash{$\zeta_i=\zeta'/\|y-\sfH\hat u_0\|_2$}, but we retain (\ref{eq:dps_grad}) as derived, so that prior and likelihood keep their relative scale. The cost is that the approximation of the denoising posterior as a point estimate is crude at large $t$, which over-weights the data misfit and can cause divergence.
We therefore apply a norm clip, controlled by a constant $c$, to the likelihood guidance contribution at each sampling step, which is only active at large $t$. Algorithm~\ref{alg:dps} provides the inference process, with a detailed version in Algorithm~\ref{alg:dps_detailed} in Appendix. Tables \ref{tab:heat_inference_full} and \ref{tab:helm_inference_full} in Appendix~\ref{appen:results} show this yields better calibrated posteriors than standard DPS, though the approximation in (\ref{eq:dps_grad}) still leaves diffusion posteriors worse calibrated than autoencoders, shown further in experiments. 

\section{Implementation} \label{sec:implementation}

\paragraph{Text generation.}
Text descriptions are generated synthetically by prompting LLMs with information from the corresponding field. They are free-form natural language, rather than a fixed template, so as to imitate the way a person would describe a QoI. The same QoI could be described in different words and sentence structures. Descriptions in the test set are held out from fine-tuning and training. More details about text generation and prompts used in each experiment are in Appendix \ref{appen:results}.

\paragraph{Network Architectures.}
The autoencoder uses modified Fourier neural operators (FNO) \citep{li2020fourier} in steady-state heat and Helmholtz experiments and a convolutional network in the UK weather experiment; the score network $\sfV$ is a U-Net \citep{ronneberger2015u} throughout. Text embeddings enter the autoencoder by concatenation; in the score network they are concatenated with the sinusoidal embedding of the diffusion time $t$, as the input to FiLM-style adaptive group normalisation (AdaGN) \citep{perez2018film, dhariwal2021diffusion}. 
This work focuses on how natural language is incorporated into the prior and in the training and inference pipeline, rather than architectural developments. 
We compare our proposed methodology (\textbf{Cond AE}, \textbf{Cond Diffusion}) against Gaussian processes (\textbf{GP}) \citep{rasmussen2003gaussian}, and a deterministic supervised model \citep{fukami2021global}, which is a direct map $(y,\tau) \mapsto u$. We also test text-free and text-conditioned variants (\textbf{Uncon Supervised}, \textbf{Cond Supervised}, more details in Appendix \ref{appen:supervised}), and unconditional generative priors trained without text (\textbf{Uncon AE}, \textbf{Uncon Diffusion}).

\paragraph{Inference-time Sampling.}
Observations are drawn at different locations for each system. $M$ posterior draws are kept per system. For the autoencoder, inference runs in the latent space with NUTS \citep{hoffman2014no} using Pyro \citep{bingham2019pyro}, with a single chain of $M_o$ warmup steps followed by $M$ retained draws. Both warmup and sampling are counted in the reported inference cost. For the diffusion model, the $M$ draws come from $M$ independent integrations of the guided reverse SDE, each of $L$ steps. More details about training and inference are in Appendix \ref{appen:results}.

\section{Numerics} \label{sec:numerics}
\subsection{Steady-State Heat Equation}
\begin{figure}[t!]
    \begin{center}
    \includegraphics[scale=0.36]{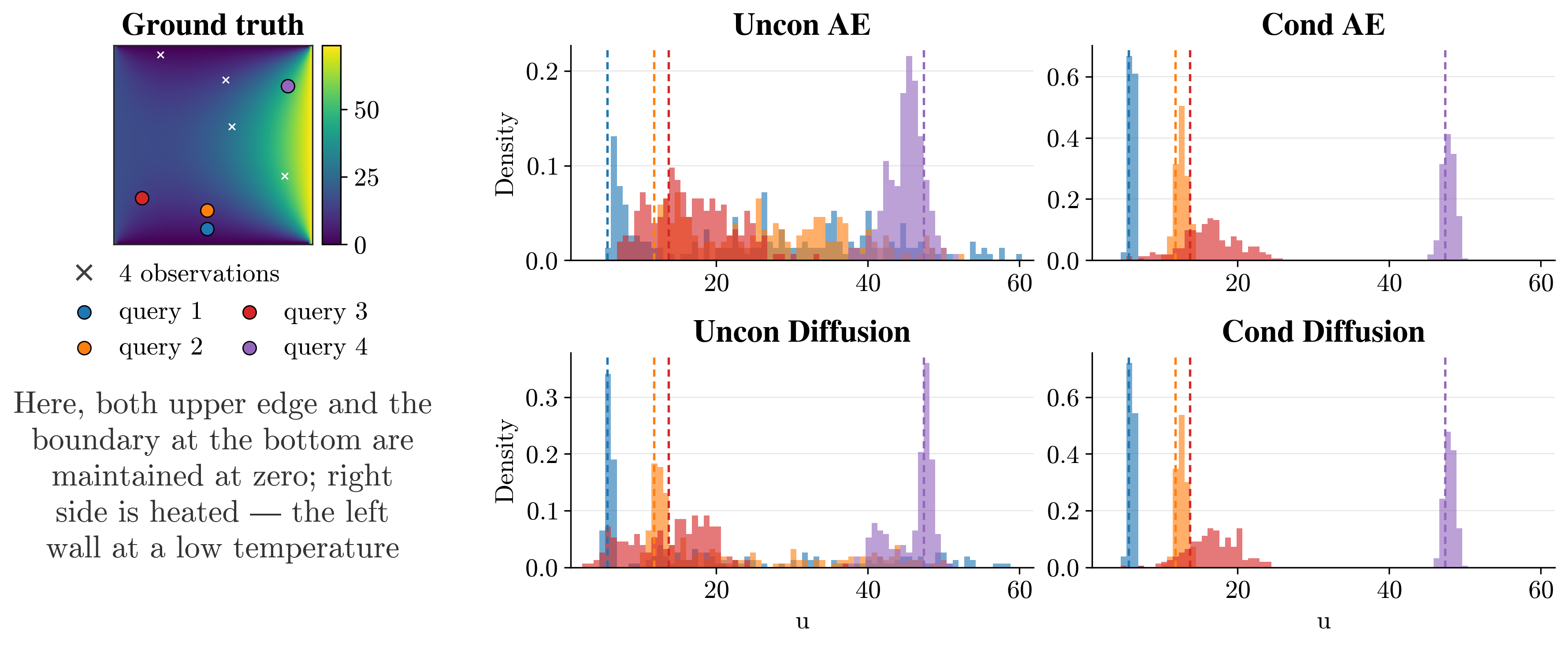}
    \end{center}
    \caption{\textbf{Posterior marginals of  steady-state heat equation 
    at four unobserved query locations for $\mathbf{n_{\mathrm{\textbf{obs}}}=4}$}. Dashed vertical lines denote the ground truth at four query points.}
    \label{fig:heat_posterior_histograms}
\end{figure}
\begin{table}[t!]
  \caption{\textbf{Inference results on the steady-state heat equation at $\mathbf{n_{\textbf{obs}}=\{4,8\}}$.} We report NMSE of posterior mean (with quartiles in Appendix Table~\ref{tab:heat_nmse_quantiles}), and coverages of the ground truth by the centered $68\%$ and $95\%$ posterior
  intervals (target values: $68\%$ and $95\%$). NMSE entries are mean $\pm$ 1 std across $200$ test fields; coverages report the mean, values in parentheses are the mean absolute difference from the target coverage (smaller is better). Pred. is the mean per-field inference time.}
  \label{tab:heat_inference}
  \centering
  \setlength{\tabcolsep}{3pt}
  \resizebox{\textwidth}{!}{%
  \begin{tabular}{l |
    S[table-format=3.2]@{\,$\pm$\,}S[table-format=3.2]
    S[table-format=2.1]@{\,}c
    S[table-format=2.1]@{\,}c |
    S[table-format=2.2]@{\,$\pm$\,}S[table-format=2.2]
    S[table-format=2.1]@{\,}c
    S[table-format=2.1]@{\,}c |
    C{1.4cm} C{1.4cm} C{1.4cm}}
    \toprule
    & \multicolumn{6}{c|}{$n_{\text{obs}}=4$}
    & \multicolumn{6}{c|}{$n_{\text{obs}}=8$}
    & \multicolumn{3}{c}{Cost} \\
    \cmidrule(lr){2-7} \cmidrule(lr){8-13} \cmidrule(lr){14-16}
    \textbf{Model}
    & \multicolumn{2}{c}{\makecell{NMSE $(\times 10^{-3})$}}
      & \multicolumn{2}{c}{\makecell{$\%C_{68}$ }}
      & \multicolumn{2}{c|}{\makecell{$\%C_{95}$ }}
    & \multicolumn{2}{c}{\makecell{NMSE $(\times 10^{-3})$}}
      & \multicolumn{2}{c}{\makecell{$\%C_{68}$}}
      & \multicolumn{2}{c|}{\makecell{$\%C_{95}$ }}
    & Param(M) & Train(hr) & Pred(s) \\
    \midrule
    GP
      & 178.87 & 103.15 & 71.7 & (14.8) & 93.6 & ( 5.5)
      &  97.03 &  59.79 & 64.6 & (13.8) & 89.7 & ( 7.5)
      & -- & -- & 0.0004 \\
    Uncon Supervised
      & 54.66 & 203.76 & \multicolumn{2}{c}{--} & \multicolumn{2}{c|}{--}
      &  4.87 &  12.89 & \multicolumn{2}{c}{--} & \multicolumn{2}{c|}{--}
      & 5.13 & 1.91 & 0.0003 \\
    Cond Supervised
      &  9.52 &  17.02 & \multicolumn{2}{c}{--} & \multicolumn{2}{c|}{--}
      &  3.54 &   6.18 & \multicolumn{2}{c}{--} & \multicolumn{2}{c|}{--}
      & 5.13 & 0.93 & 0.0003\\
    Uncon AE
      & 46.45 & 132.21 & 61.5 & (22.4) & 88.7 & (12.0)
      &  4.85 &  15.95 & 67.6 & (17.8) & 92.6 & ( 8.5)
      & 9.48 & 3.63 & 28.51 \\
    Cond AE
      &  6.13 &  11.70 & 61.6 & (23.4) & 88.2 & ( 8.2)
      &  1.40 &   3.01 & 62.5 & (20.8) & 88.1 & ( 7.9)
      & 9.87 & 2.40 & 10.25 \\
    Uncon Diffusion
      & 37.71 & 107.23 & 64.9 & (25.2) & 89.2 & (12.2)
      &  2.45 &   8.95 & 59.8 & (26.0) & 87.0 & (14.3)
      & 0.99 & 2.75 & 32.15 \\
    Cond Diffusion
      &  7.12 &  17.64 & 62.3 & (25.2) & 89.8 & (12.0)
      &  1.44 &   3.54 & 64.3 & (22.7) & 90.3 & (11.2)
      & 1.09 & 1.76 & 32.21 \\
    \bottomrule
  \end{tabular}%
  }
\end{table}
\begin{figure}[t!]
    \begin{center}
    \includegraphics[scale=0.38]{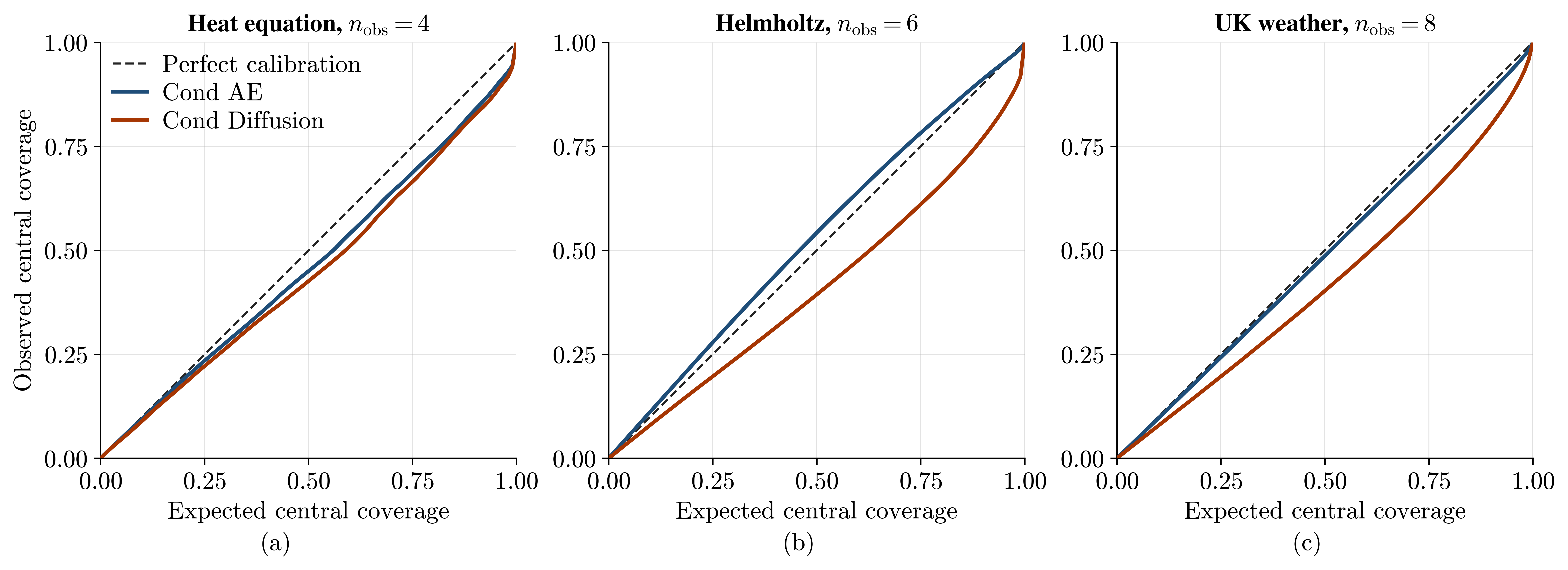}
    \end{center}
    \caption{ \textbf{Performance of posterior calibration.} We plot calibration curves (observed coverage over target central coverage) for three experiments, where the diagonal represents perfect calibration, with more details about calibration curve calculation in Appendix \ref{appen:metrics}.}
    \label{fig:calibration}
\end{figure}
\textbf{Setup:} The first experiment is on the steady-state heat equation on a unit square, with setup in Appendix~\ref{appen:heat_setup}. For each system, we randomly select two of the four sides and assign zero temperature; the remaining two are set to a constant drawn from $\mathrm{Unif}[10,80]$. Accompanying texts describe these boundary temperatures qualitatively, with an example in Figure~\ref{fig:heat_posterior_histograms}. Training uses $6.4k$ samples. \textbf{Results:} After training the autoencoder $\sfE,\sfD$, and the diffusion score network $\sfV$, we perform inference on $200$ test fields with observation noise $\sigma_y=1$, and different numbers of observations $n_{\mathrm{obs}}\in\{2,4,6,8,10\}$. Tables \ref{tab:heat_inference} and \ref{tab:heat_inference_full} compare all models in terms of accuracy, coverage and cost, with metrics detailed in Appendix \ref{appen:metrics}. Generative priors improve substantially from GP and supervised baselines, and each conditional model improves on its unconditional counterpart. Posterior histograms in Figure~\ref{fig:heat_posterior_histograms} also show that unconditional models produce wider posteriors. \textbf{Cond AE} and \textbf{Cond Diffusion} report comparable NMSE and coverages in Table \ref{tab:heat_inference} and Figure~\ref{fig:calibration}(a), but the bracketed mean absolute differences from target coverages are larger for \textbf{Cond Diffusion} ($12.0$ against $8.2$ for $\%C_{95}$ at $n_\text{obs}=4$), and Appendix Figure~\ref{fig:heat_coverage} shows that per-system coverages of \textbf{Cond Diffusion} are less concentrated around target coverages, with its high average coming from $60\%$ of systems reaching $\%C_{95}=100\%$, where the posterior is over-dispersed. This shows that DPS introduces bias into the posterior approximation due to the estimation in (\ref{eq:dps_grad}). More evaluations such as ablations on different $n_{obs}$ (Figures \ref{fig:heat_text_pair},\ref{fig:heat_obs_graph2}) and CRPS (Table \ref{tab:heat_inference_full}) are in Appendix \ref{appen:heat_results}.

\subsection{Helmholtz Equation}

\begin{figure}[t]
    \begin{center}
    \includegraphics[scale=0.4]{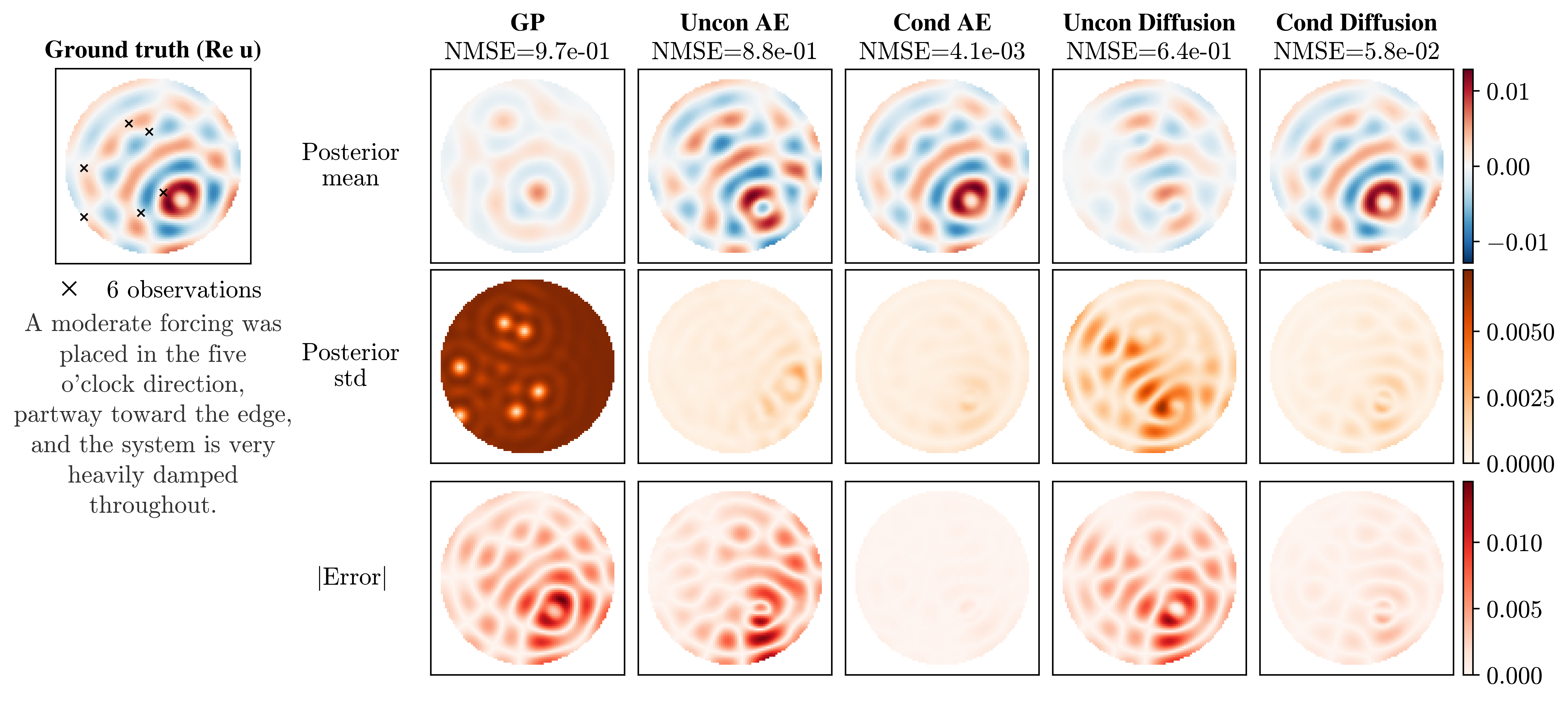}
    \end{center}
    \caption{Inference results of one Helmholtz test system (Real part) at $n_{\mathrm{\text{obs}}}=6$.}
    \label{fig:helm_single_system}
\end{figure}
\begin{table}[t]
  \centering
  \caption{Inference results on 100 Helmholtz test fields, more details in Appendix Tables~\ref{tab:helm_inference_full} and \ref{tab:helm_nmse_quantiles}.}
  \label{tab:helm_inference}
  \setlength{\tabcolsep}{3pt}
  \resizebox{\textwidth}{!}{%
  \begin{tabular}{l |
    S[table-format=2.2]@{\,$\pm$\,}S[table-format=3.2]
    S[table-format=2.1]@{\,}c
    S[table-format=2.1]@{\,}c |
    S[table-format=2.2]@{\,$\pm$\,}S[table-format=2.2]
    S[table-format=2.1]@{\,}c
    S[table-format=2.1]@{\,}c |
    C{1.4cm} C{1.4cm} C{1.4cm}}
    \toprule
    & \multicolumn{6}{c|}{$n_{\text{obs}}=6$}
    & \multicolumn{6}{c|}{$n_{\text{obs}}=10$}
    & \multicolumn{3}{c}{Cost} \\
    \cmidrule(lr){2-7} \cmidrule(lr){8-13} \cmidrule(lr){14-16}
    \textbf{Model}
    & \multicolumn{2}{c}{\makecell{NMSE $(\times 10^{-2})$}}
      & \multicolumn{2}{c}{\makecell{$\%C_{68}$ }}
      & \multicolumn{2}{c|}{\makecell{$\%C_{95}$ }}
    & \multicolumn{2}{c}{\makecell{NMSE $(\times 10^{-2})$}}
      & \multicolumn{2}{c}{\makecell{$\%C_{68}$ }}
      & \multicolumn{2}{c|}{\makecell{$\%C_{95}$ }}
    & Param(M) & Train(hr) & Pred(s) \\
    \midrule
    GP
      & 86.74 &  9.20 & 83.3 & (23.0) & 94.0 & ( 7.4)
      & 77.12 & 11.96 & 82.3 & (23.1) & 93.6 & ( 7.7)
      & -- & -- & 0.005 \\
    Uncon Supervised
      & 25.93 & 23.18 & \multicolumn{2}{c}{--} & \multicolumn{2}{c|}{--}
      & 10.17 &  7.68 & \multicolumn{2}{c}{--} & \multicolumn{2}{c|}{--}
      & 8.22 & 13.64 & 0.002 \\
    Cond Supervised
      & 19.77 & 25.27 & \multicolumn{2}{c}{--} & \multicolumn{2}{c|}{--}
      & 10.84 & 16.48 & \multicolumn{2}{c}{--} & \multicolumn{2}{c|}{--}
      & 8.32 & 10.91 & 0.002 \\
    Uncon AE
      & 66.61 & 338.05 & 45.9 & (28.7) & 68.6 & (28.3)
      & 16.07 &  30.13 & 46.1 & (26.9) & 68.4 & (28.5)
      & 19.71 & 19.37 & 145.25 \\
    Cond AE
      &  7.15 & 17.18 & 71.4 & (13.6) & 94.5 & ( 5.5)
      &  3.16 &  7.25 & 69.7 & (11.8) & 93.8 & ( 5.4)
      & 20.66 & 15.91 & 43.90 \\
    Uncon Diffusion
      & 22.57 & 24.20 & 64.4 & (18.0) & 93.4 & ( 6.7)
      &  5.36 &  9.24 & 67.6 & (19.1) & 92.0 & ( 8.4)
      & 6.02 & 17.33 & 231.47 \\
    Cond Diffusion
      &  7.15 & 12.63 & 54.5 & (23.5) & 83.0 & (15.3)
      &  3.45 &  7.54 & 54.7 & (20.7) & 83.6 & (14.1)
      & 6.31 & 13.17 & 232.18 \\
    \bottomrule
  \end{tabular}%
  }
\end{table}
\textbf{Setup:} This experiment is a damped Helmholtz equation on a circular disc, with uniform forcing over a small circular region. Each system is a complex field, with two channels for the real and imaginary parts. Accompanying text describes the damping intensity, the amplitude and location of forcing, with an example in Figure~\ref{fig:helm_single_system}. The precision of descriptions varies across systems, with details in
Appendix~\ref{appen:helm_setup}. $8k$ training samples are generated. \textbf{Results:} Inference is performed over $100$ test fields with $\sigma_y=6\times 10^{-4}$ and $n_{\text{obs}}=\{4,6,8,10,12\}$. Tables~\ref{tab:helm_inference} and~\ref{tab:helm_inference_full} again show conditional generative priors ahead of other baselines and unconditional alternatives, with \textbf{Cond AE} attaining better posterior coverage and faster inference than \textbf{Cond Diffusion}. In Figure~\ref{fig:calibration}(b), the calibration curve of \textbf{Cond AE} sits closer to the diagonal and \textbf{Cond Diffusion} sits below the diagonal, indicating overconfident posteriors. Figure~\ref{fig:helm_single_system} plots the posterior mean, standard deviation, and absolute error of five probabilistic models at $n_{\mathrm{obs}}=6$, with unconditional models having both a larger error and posterior spread. Additional results are provided in Appendix \ref{appen:helm_results}.

\subsection{UK Weather Reanalysis Data}

\begin{figure}[t]
    \begin{center}
    \includegraphics[scale=0.35]{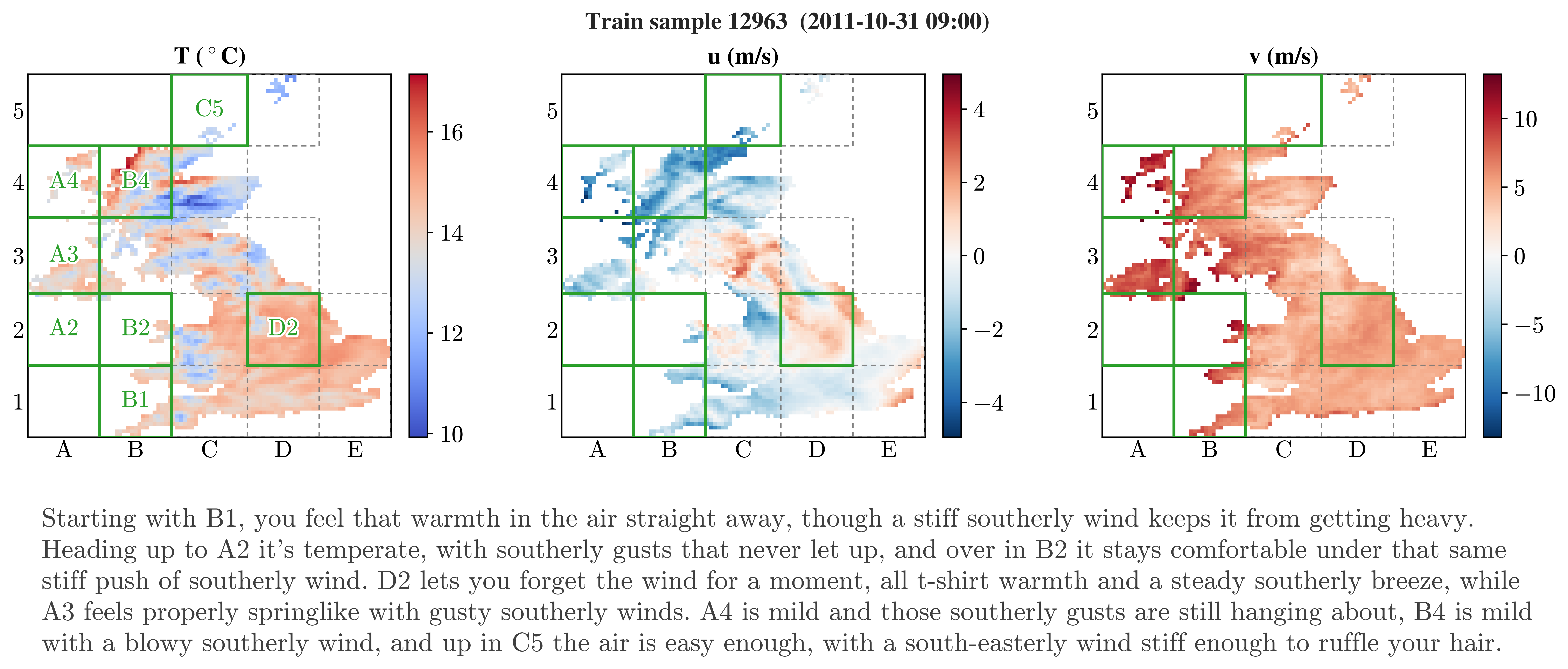}
    \end{center}
    \caption{\textbf{Example of a 3-channel weather data slice with corresponding text description.} Green boxes mark the regions referenced in the text, and grey dashed boxes mark available land regions.}
    \label{fig:uk_visualisation}
\end{figure}
\begin{figure}[t]
    \begin{center}
    \includegraphics[scale=0.4]{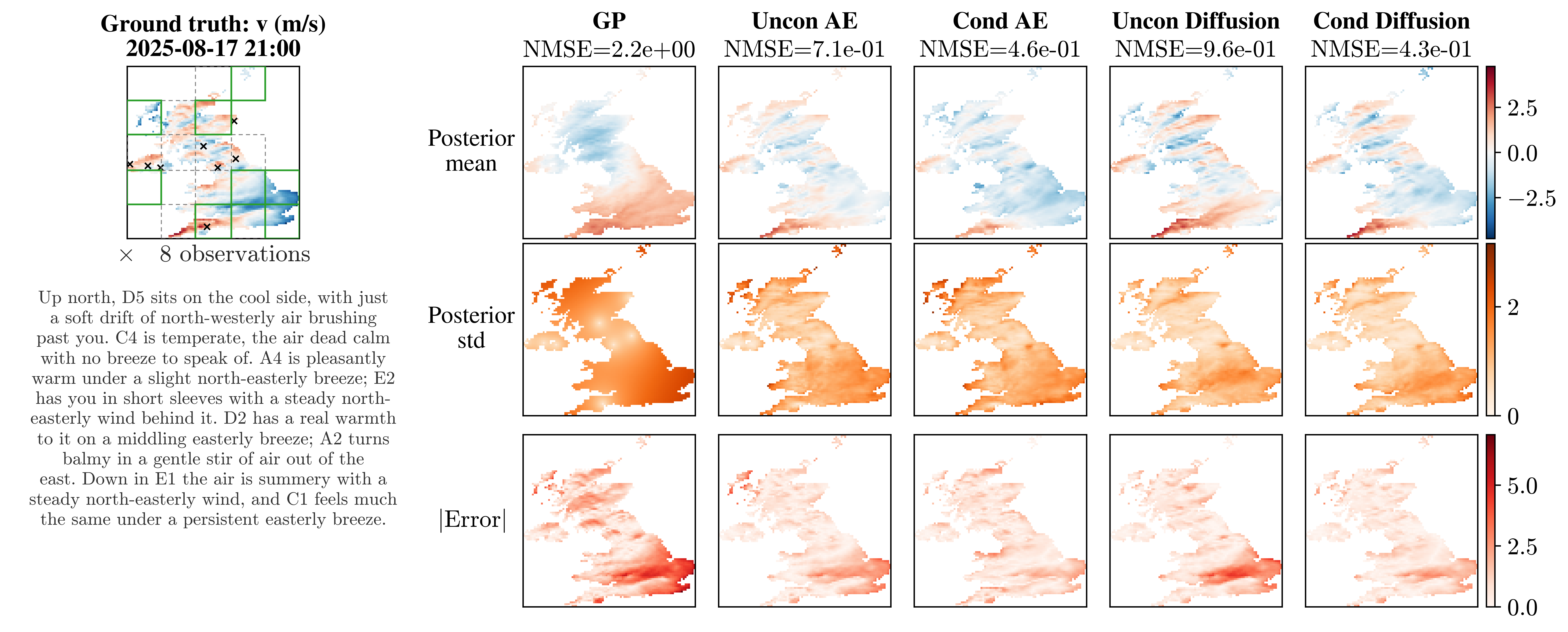}
    \end{center}
    \caption{Inference results of one UK weather test system 
    ($v$) with $n_{\mathrm{\text{obs}}}=8$.}
    \label{fig:uk_single_system_686}
\end{figure}
\begin{table}[H]
  \centering
  \caption{Inference results on 100 UK weather test fields, more details in Appendix Tables~\ref{tab:uk_inference_full} and~\ref{tab:uk_nmse_quantiles}. }
  \label{tab:uk_inference}
  \setlength{\tabcolsep}{3pt}
  \resizebox{\textwidth}{!}{%
  \begin{tabular}{l |
    S[table-format=2.2]@{\,$\pm$\,}S[table-format=2.2]
    S[table-format=2.1]@{\,}c
    S[table-format=2.1]@{\,}c |
    S[table-format=2.2]@{\,$\pm$\,}S[table-format=2.2]
    S[table-format=2.1]@{\,}c
    S[table-format=2.1]@{\,}c |
    C{1.4cm} C{1.4cm} C{1.4cm}}
    \toprule
    & \multicolumn{6}{c|}{$n_{\text{obs}}=8$}
    & \multicolumn{6}{c|}{$n_{\text{obs}}=12$}
    & \multicolumn{3}{c}{Cost} \\
    \cmidrule(lr){2-7} \cmidrule(lr){8-13} \cmidrule(lr){14-16}
    \textbf{Model}
    & \multicolumn{2}{c}{\makecell{NMSE $(\times 10^{-2})$}}
      & \multicolumn{2}{c}{\makecell{$\%C_{68}$}}
      & \multicolumn{2}{c|}{\makecell{$\%C_{95}$}}
    & \multicolumn{2}{c}{\makecell{NMSE $(\times 10^{-2})$}}
      & \multicolumn{2}{c}{\makecell{$\%C_{68}$}}
      & \multicolumn{2}{c|}{\makecell{$\%C_{95}$}}
    & Params(M) & Train(hr) & Pred.(s) \\
    \midrule
    GP
      & 12.12 & 16.06 & 68.3 & ( 8.5) & 91.7 & ( 5.0)
      & 10.16 & 13.80 & 69.7 & ( 8.1) & 92.1 & ( 4.6)
      & -- & -- & 0.002 \\
    Uncon Supervised
      &  7.46 & 11.38 & \multicolumn{2}{c}{--} & \multicolumn{2}{c|}{--}
      &  6.22 & 10.47 & \multicolumn{2}{c}{--} & \multicolumn{2}{c|}{--}
      & 8.23 & 22.34 & 0.002 \\
    Cond Supervised
      &  6.78 &  9.95 & \multicolumn{2}{c}{--} & \multicolumn{2}{c|}{--}
      &  5.89 &  8.81 & \multicolumn{2}{c}{--} & \multicolumn{2}{c|}{--}
      & 8.32 & 14.89 & 0.002 \\
    Uncon AE
      &  8.18 & 11.55 & 67.5 & ( 6.5) & 93.6 & ( 3.3)
      &  6.90 & 10.50 & 65.8 & ( 5.9) & 92.3 & ( 3.6)
      & 16.92 & 14.08 & 47.02 \\
    Cond AE
      &  6.93 &  9.81 & 65.9 & ( 5.6) & 92.5 & ( 3.5)
      &  6.12 &  9.58 & 63.8 & ( 6.6) & 90.9 & ( 4.5)
      & 17.71 & 15.73 & 47.08 \\
    Uncon Diffusion
      &  7.72 & 11.25 & 56.2 & (12.4) & 86.5 & ( 8.7)
      &  6.32 & 11.66 & 56.3 & (12.2) & 86.5 & ( 8.5)
      & 9.19 & 44.29 & 283.38 \\
    Cond Diffusion
      &  6.81 & 10.13 & 56.0 & (12.5) & 86.0 & ( 9.1)
      &  6.05 & 11.02 & 55.7 & (12.6) & 85.8 & ( 9.2)
      & 10.17 & 26.73 & 283.82 \\
    \bottomrule
  \end{tabular}%
  }
\end{table}
\textbf{Setup:} This experiment uses CERRA reanalysis data \citep{ridal2024cerra}, taking temperature and two wind velocity components over UK land as three channels $\{T,u,v\}$. A $5\times5$ grid is laid over the domain and each description covers $8$
randomly chosen regions, reporting the temperature and wind of each with an example in Figure~\ref{fig:uk_visualisation}. Training uses $27{,}396$ fields from 2000-2024. \textbf{Results.} Inference is performed over $100$ test fields, with $\sigma_y=\{0.5,0.3,0.3\}$ and $n_{\text{obs}}=\{4,6,8,10,12,14\}$, with results in Tables~\ref{tab:uk_inference} and~\ref{tab:uk_inference_full}. Deterministic supervised baselines perform better here, but provide no UQ and are tied to the observation process they were trained on. With a different observation operator at inference in Appendix Table~\ref{tab:uk_disjoint}, generative models overtake them. Diffusion models provide better NMSE than autoencoders, showing its expressiveness on more irregular fields, but \textbf{Cond Diffusion} gives worse calibrated posterior than \textbf{Cond AE} in Figure \ref{fig:calibration}(c). Figure~\ref{fig:uk_single_system_686} plots results for one $v$ field. The gain from text is smaller here than in other experiments, since the text carries less information about the field, but conditional models still reach lower NMSE and show lower error inside the regions that the text describes. More evaluations are provided in Appendix \ref{appen:uk_results}.

\section{Conclusion} \label{sec:conclusion}

We propose a framework for learning text-conditional generative priors over physical fields, allowing qualitative natural-language descriptions to be combined with sparse quantitative measurements for posterior inference and uncertainty quantification. We develop a complete pipeline from fine-tuning a text encoder, to training generative priors based on autoencoders and diffusion models, and to posterior inference. Comparing the two, autoencoders achieve better-calibrated posteriors while diffusion models offer greater expressiveness on more irregular fields. Across our experiments, incorporating text into the prior consistently improves reconstruction accuracy over unconditional and non-data-driven priors.

\subsubsection*{Acknowledgments}
MG is supported by a Royal Academy of Engineering Research Chair and EPSRC grants [EP/X037770/1, EP/Y028805/1, EP/V056441/1 and EP/W005816/1]. AV is supported through the EPSRC ROSEHIPS grant [EP/W005816/1].

\subsection*{AI use statement}

In this work, we used generative AI tools to help in the retrieval and discovery of related works, and to help write code to implement methods. It was used to a limited extent to polish specific sections of the writing.
AI was used in providing some limited ingredients, subsequently used and arranged into the theory, for proving mathematical claims.
AI was also used in creating synthetic datasets for experiments.
Prompts used to generate synthetic datasets are provided in the Appendix. 

\subsection*{Reproducibility statement}

Model structures and hyperparameters used in all experiments are provided in Appendix \ref{appen:results}. Code for reproducing all experiments is available at \url{https://github.com/ypzpy/Text-conditional-prior}.

\bibliography{reference}
\bibliographystyle{reference}

\appendix

\clearpage

\section{Notation} \label{appen:notation}

\begin{table}[H]
  \centering
  \caption{Symbols and concepts.}
  \label{tab:notation}
  \setlength{\tabcolsep}{6pt}
  \resizebox{\textwidth}{!}{%
  \begin{tabular}{c l l}
    \toprule
    \textbf{Symbol} & \textbf{Description} & \textbf{Example} \\
    \midrule
    $u \in \cU$ & Quantity of Interest (QoI)
      & Heat distribution over the physical domain \\
    $\theta \in \Theta$ & Human-perceivable quantity related to $u$
      & The left boundary condition \\
    $s \in \mathcal{S}$ & Qualitative data: description of $\theta$ 
      & \texttt{"the left boundary is hot"} \\
    $\tau \in \cT$ & Euclidean description of $s$
      & Text encoding of \texttt{"the left boundary is hot"} \\
    $y \in \Rb^{d_y}$ & Quantitative data: incomplete/indirect measurements of $u$
      & Sparse sensor measurements of heat on domain \\
    \bottomrule
  \end{tabular}%
  }
\end{table}

In this paper, spaces are written in calligraphic type ($\cU,\cT,\cZ$) and maps between them in
sans serif ($\sfD,\sfE,\sfG^\phi,\sfV,\sfH$). Probability measures are written as $\bP,\bQ$, with a subscript naming the variable they act on; $\bP_{u,\tau}$ is a measure on
$(\cU\times\cT,\cB(\cU\times\cT))$, where $\cB(\cdot)$ is the Borel
sigma-algebra, and $\bP_{u\mid\tau}$ is the corresponding conditional. $p$ is reserved for
densities with respect to Lebesgue measure, as in the score $\nabla_u\log p_t(u|\tau)$.

For a measurable function $g: \cZ \rightarrow \cU$, the pushforward $g_\#\bP_z$ is defined by
$g_\#\bP_z(A)=\bP_z(g^{-1}(A))$ for every measurable set $A$. A sample $u\sim g_\#\bP_z$ is obtained by drawing $z\sim\bP_z$ and returning $u=g(z)$. A bullet marks an argument left free, so $\sfD(\bullet,\tau)_\#\bQ_z$ is the
pushforward of $\bQ_z$ through the decoder with the text embedding held fixed at $\tau$. Table~\ref{tab:notation} summarises some symbols and concepts.

\section{Proof}
\subsection{Theorem \ref{thm:obj_zero}} \label{appen:thm_2_1}
\begin{proof}\label{proof:obj_zero}
    For $\bP_{u,\tau}-a.e$ $u,\tau$, if the first reconstruction error term in (\ref{eq:ae_objective}) is $0$, we have
    $$\widetilde{\sfD}\circ\widetilde{\sfE}=(\Id_{\cU}\times\Id_\cT).$$
    Given the joint distribution $\bP_{u,\tau}$, we use notation ``$\otimes$", such that for any measurable set $A$,
    \begin{align}
        \bP_{u,\tau}(A)=(\bP_\tau\otimes\bP_{u|\tau})(A)=\int_\cT \bP_{u|\tau}(\{u: (u,\tau)\in A\})\md \bP_{\tau}(\tau);
    \end{align}
    where, for $\bP_\tau$ being the $\tau$-marginal of $\bP_{u,\tau}$, $\bP_{u|\tau}$ is the $\bP_\tau-a.e$ unique system of conditional measure on the sample space of variable $u$, and is $\bP_\tau$-measurable. 
    Now, invoking existence and uniqueness of conditional measures, if the second term in (\ref{eq:ae_objective}) is also $0$,
    \begin{align*}
        &\widetilde{\sfE}_\#(\Pb_{u,\tau})= \bQ_z\times\bP_\tau\\
        \implies & (\widetilde{\sfD}\circ\widetilde{\sfE})_\#(\Pb_{u,\tau})=\widetilde{\sfD}_\#(\bQ_z\times\bP_\tau)\\
        \implies & \Pb_{u,\tau}=\widetilde{\sfD}_\#(\bQ_z\times\bP_\tau)\\
        {\implies} & \Pb_{u,\tau}=(\bP_\tau\otimes \sfD(\bullet, \tau)_\#\bQ_z)\\
        \implies & \bP_{u|\tau} = \sfD(\bullet, \tau)_\#\bQ_z, \quad \bP_\tau-a.e.
    \end{align*}
     Going from line three to four, involves 
    \begin{align*}
        \widetilde{\sfD}_\#(\bQ_z\times\bP_\tau)(A)
         =&\int_{\cU\times\cT} \bb1_A(u,\tau)\;\md (\widetilde{\sfD}_\#(\bQ_z\times\bP_\tau))(u, \tau)\\
        =&\int_{\cZ\times\cT} \bb1_A(\sfD(z, \tau), \tau)\;\md (\bQ_z\times\bP_\tau)(z, \tau)\\
        =&\int_\cT\int_\cZ \bb1_A(\sfD(z, \tau), \tau)\;\md\bQ_z(z)\,\md\bP_\tau( \tau)\\
        =&\int_\cT\int_\cU \bb1_A(u, \tau)\;\md(\sfD(\bullet, \tau)_\#\bQ_z)(u)\,\md\bP_\tau( \tau)\\
        =&\int_\cT (\sfD(\bullet, \tau)_\#\bQ_z)(\{u:(u,\tau)\in A\})\,\md\bP_\tau( \tau)\\
        =& (\bP_\tau\otimes\sfD(\bullet, \tau)_\#\bQ_z)(A),\quad \text{for all measurable } A;
    \end{align*}
    pushing $\widetilde{\sfD}$ into the integral, and $\sfD$ out of it, invokes  Theorem 3.6.1 in \citet{bogachev2007measure}, 
    separation of integrals and exchange of order follows from Tonelli's theorem. 
\end{proof}

\subsection{Theorem~\ref{thm:sqWass_bound}} \label{appen:thm_2_2}
\begin{proof}\label{proof:sqWass_bound}
    By squaring the triangle inequality
    \begin{align*}
        \sfW_2^2(\bP_{u,\tau}, \widetilde{\sfD}_\#(\bQ_z\times\bP_\tau))\leq
         2\sfW_2^2(\bP_{u,\tau}, &(\widetilde{\sfD}\circ\widetilde{\sfE})_\#\bP_{u,\tau}) + 2\sfW_2^2((\widetilde{\sfD}\circ\widetilde{\sfE})_\#\bP_{u,\tau}, \widetilde{\sfD}_\#(\bQ_z\times\bP_\tau)).
    \end{align*}
    We obtain a bound on the first term by selecting $\gamma\in\Pi(\bP_{u,\tau},(\widetilde{\sfD}\circ\widetilde{\sfE})_\#\bP_{u,\tau})$ to be of the form $\gamma=(\Id \times(\widetilde{\sfD}\circ\widetilde{\sfE}))_\#\bP_{u,\tau})$, which is a valid coupling,
    thus
    \begin{align*}
        \sfW_2^2(\bP_{u,\tau}, (\widetilde{\sfD}\circ\widetilde{\sfE})_\#\bP_{u,\tau})&=\inf_{\gamma\in\Pi}\int_{(\cU\times\cT)^2}\|(u,\tau)-(u', \tau')\|^2\md\gamma\\
        &\leq \int_{(\cU\times\cT)^2}\|(u,\tau)-(u', \tau')\|^2\md((\Id.\times(\widetilde{\sfD}\circ\widetilde{\sfE}))_\#\bP_{u,\tau})\\
        &=\Eb_{u,\tau\sim\bP_{u,\tau}}\|(u,\tau)-\widetilde{\sfD}\circ\widetilde{\sfE}(u,\tau)\|^2\\
        &=\Eb_{u,\tau\sim\bP_{u,\tau}}[\|u-\sfD\circ\widetilde{\sfE}(u,\tau)\|^2 + \|\tau-\tau\|^2]\\
        &=\Eb_{u,\tau\sim\bP_{u,\tau}}\|u-\sfD\circ\widetilde{\sfE}(u,\tau)\|^2.
    \end{align*}
    Bounding the second term involves selecting a coupling $\gamma_\epsilon\in\Pi(\widetilde{\sfE}_\#\bP_{u,\tau},\bQ_z\times\bP_\tau)$ such that for a chosen $\epsilon>0$
    \begin{align*}
       \int_{(\cZ\times\cT)^2} \|(z,\tau) - (z',\tau')\|^2\md\gamma_\epsilon\leq\sfW_2^2(\widetilde{\sfE}_\#\bP_{u,\tau}, \bQ_z\times\bP_\tau)+\epsilon.
    \end{align*}
    Then, $(\widetilde{\sfD}\times\widetilde{\sfD})_\#\gamma_\epsilon$ is an admissible coupling such that 
    \begin{align*}
        \sfW_2^2((\widetilde{\sfD}\circ\widetilde{\sfE})_\#\bP_{u,\tau}, \widetilde{\sfD}_\#(\bQ_z\times\bP_\tau))&\leq \int_{(\cU\times\cT)^2} \|(u,\tau)-(u',\tau')\|^2\md((\widetilde{\sfD}\times\widetilde{\sfD})_\#\gamma_\epsilon)\\
        & = \int_{(\cZ\times\cT)^2} \|\widetilde{\sfD}(z,\tau)-\widetilde{\sfD}(z',\tau')\|^2\md\gamma_\epsilon\\
        & = \int_{(\cZ\times\cT)^2}( \|{\sfD}(z,\tau)-{\sfD}(z',\tau')\|^2 + \|\tau-\tau'\|^2)\md\gamma_\epsilon\\
        & \leq \int_{(\cZ\times\cT)^2} (\|\sfD\|^2_\Lip\|(z,\tau)-(z',\tau')\|^2 + \|\tau-\tau'\|^2 \\
        & \quad \quad \quad \quad \quad \quad \quad \quad \quad \quad \quad \quad \quad \quad + \|z-z'\|^2)\md\gamma_\epsilon\\
        &= \int_{(\cZ\times\cT)^2} (\|{\sfD}\|_{\Lip}^2+1)\|(z,\tau) - (z',\tau')\|^2\md\gamma_\epsilon \\
        &\leq ({\|\sfD\|_{\Lip}^2+ 1})(\sfW_2^2(\widetilde{\sfE}_\#\bP_{u,\tau}, \bQ_z\times\bP_\tau) + \epsilon).
    \end{align*}
    As this is true for any $\epsilon>0$,
    \begin{align*}
        \sfW_2^2((\widetilde{\sfD}\circ\widetilde{\sfE})_\#\bP_{u,\tau}, \widetilde{\sfD}_\#(\bQ_z\times\bP_\tau))\leq ({\|\sfD\|_{\Lip}^2+ 1})\sfW_2^2(\widetilde{\sfE}_\#\bP_{u,\tau}, \bQ_z\times\bP_\tau).
    \end{align*}
    Substituting the two bounds into the squared triangle inequality gives Theorem~\ref{thm:sqWass_bound}.
\end{proof}

\subsection{Corollary \ref{cor:mmd_bound}}\label{appen:coro_2_3}
\begin{proof}\label{proof:mmd_bound}
    This follows directly from Theorem~\ref{thm:sqWass_bound} and the bound, for distributions on compact supports,
    $$\sfW_2(\bP, \bP')\leq C\sfMMD_{\mathrm{RBF}}^\alpha(\bP, \bP')$$
    with $C>0,\alpha \in(0,1]$~\citep{vayer2023controlling}. Note that the compact-support assumption does not hold for $ \bQ_z=\mathcal N(0,\I)$ exactly. This corollary is included to connect the MMD term in (\ref{eq:ae_objective}) to the Wasserstein bound in (\ref{eq:theorem2}) qualitatively.
\end{proof}

\section{Supervised Method} \label{appen:supervised}

The supervised baseline follows \citet{fukami2021global}, a direct
map $(y,\tau)\mapsto u$ that can take any number of sensors at arbitrary locations as input. The observations are first expanded to a dense field following Voronoi tessellation, where every pixel takes the value of its nearest sensor, with this field concatenated with a binary mask marking sensor locations. Figure~\ref{fig:supervised_baseline} shows ground truth fields and the corresponding Voronoi inputs at $n_{\mathrm{obs}}=\{4,8\}$. We use U-Net as the base architecture in our paper, and for \textbf{Cond Supervised} the text embedding is added through FiLM layers. However, the supervised map is tied to the kind of observation process it was trained under. A different range of sensor numbers or locations, different forward operator, or a different noise level, requires retraining. Generative models keep prior and likelihood separate, so likelihood and the corresponding observation process can be changed after training. 

\begin{figure}[H]
    \begin{center}
    \includegraphics[scale=0.5]{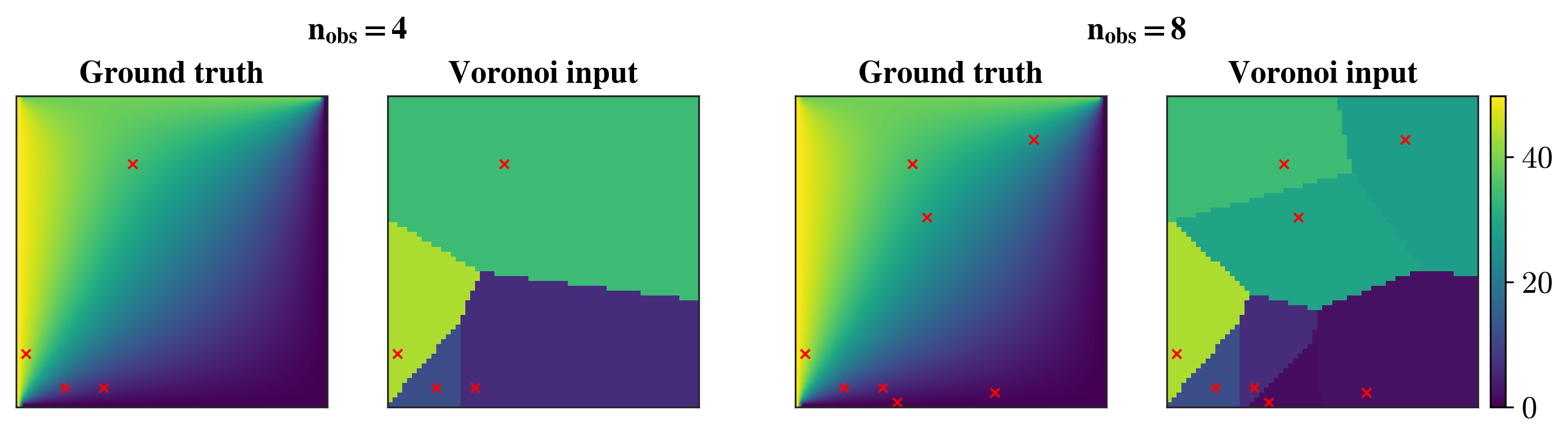}
    \end{center}
    \caption{Ground truth fields and Voronoi inputs to the supervised model at $n_{\mathrm{obs}}=\{4,8\}$.}
    \label{fig:supervised_baseline}
\end{figure}

\section{MMD Implementation in Autoencoders} \label{appen:mmd}

In autoencoders, we use maximum mean discrepancy (MMD) for $\sfd(\bullet,\bullet)$ in (\ref{eq:ae_objective}) to compute the distance in the joint probability space $\cZ\times\cT$. In a batch of size $B$, we have joint samples \smash{$\{(z_i,\tau_i)\}_{i=1}^B\sim\widetilde{\sfE}_\#(\bP_{u,\tau})$} that pair each encoder output $z_i=\sfE(u_i,\tau_i)$ with its own text embedding. We also have reference samples $\{(z_i^{\mathrm{ref}}, \tau_i^{\mathrm{ref}})\}_{i=1}^B$ with $z_i^{\mathrm{ref}}\sim\mathcal{N}(0,\I)=\bQ_z$, and $\tau_i^{\mathrm{ref}} \sim \bP_\tau$ is a text embedding drawn independently from the training pool. We evaluate the MMD between \smash{$\{ (z_i, \tau_i )\}_{i=1}^B$} and $\{(z_i^{\mathrm{ref}}, \tau_i^{\mathrm{ref}})\}_{i=1}^B$, with the kernel being a mixture of five RBF kernels with bandwidths $2^{m}h$, $m=-2,\dots,2$, where $h$ is the mean pairwise squared distance of all samples in the batch. 

In this setup, $z$ and $\tau$ have very different dimensions: $d_\tau=384$ for the text embedding, and $d_z$ ranges from $16$ to $64$ across our experiments, so the normalised text block would account for $86$--$96\%$ of every pairwise distance. Therefore, we downscale normalised $\tau$ by a factor approximately equal to \smash{$\sqrt{d_\tau/d_z}$}, which equalises the two contributions. Figure~\ref{fig:latents} plots latent histograms across all dimensions over the test set in our three experiments, showing that the marginals of $z$ align closely with the ideal $\bQ_z=\mathcal{N}(0,\I)$.

\begin{figure}[H]
    \begin{center}
    \includegraphics[scale=0.45]{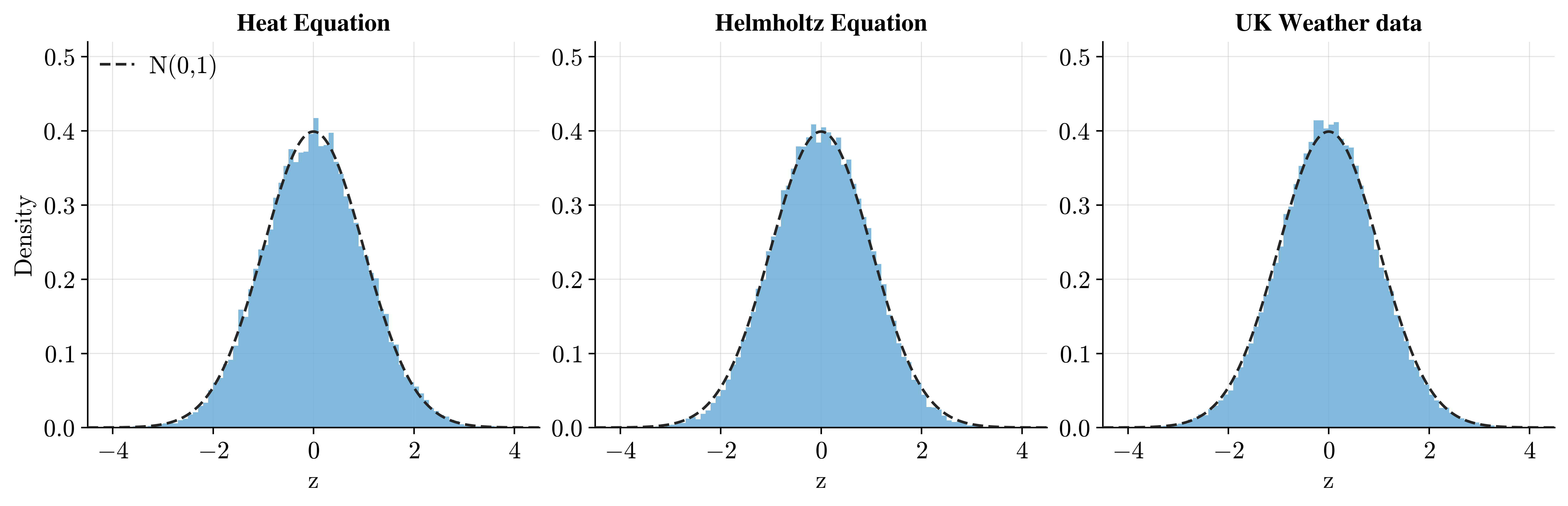}
    \end{center}
    \caption{Latent histograms across all dimensions over the test set in three experiments, compared with the ideal $\bQ_z=\mathcal{N}(0,\I)$. }
    \label{fig:latents}
\end{figure}

\section{Modified DPS Algorithm}\label{appen:DPS}

In Algorithm~\ref{alg:dps_detailed} we provide the modified DPS algorithm, which is an extension of Algorithm~\ref{alg:dps}.

\begin{algorithm}[h]
\caption{Posterior sampling --- diffusion (detailed version of
         Algorithm~\ref{alg:dps})}
\label{alg:dps_detailed}
\begin{algorithmic}[1]
\Require trained $\sfV,\sfG^\phi$; text $s$, observation $y$, operator $\sfH$,
         noise $\sigma_y$; grid $T=t_L>\dots>t_0=\varepsilon$, guidance cap $c$,
         $M$ posterior draws; write $u_i$ for $u_{t_i}$
\State $\tau \gets \sfG^\phi(s)$
\For{$m=1$ \textbf{to} $M$}
  \State $u_L\sim\mathcal{N}(0,\I)$
  \For{$i=L$ \textbf{to} $1$}
    \State $\beta\gets\beta(t_i)$, $\bar\alpha\gets\bar\alpha(t_i)$,
           $\Delta t\gets t_i-t_{i-1}$
    \State $\hat v\gets\sfV(u_i,t_i,\tau)$
      \Comment{conditional prior score}
    \State $\hat u_0\gets\big(u_i+(1-\bar\alpha)\hat v\big)/\sqrt{\bar\alpha}$
      \Comment{Tweedie estimate, eq.~(\ref{eq:tweedie})}
    \State $g\gets -\tfrac{1}{2\sigma_y^{2}}\nabla_{u_i}\|y-\sfH\hat u_0\|_2^{2}$
      \Comment{likelihood score, eq.~(\ref{eq:dps_grad})}
    \State $\delta\gets\beta\,\Delta t\,g$;\quad
           $\delta\gets\delta\cdot\min(1,c/\|\delta\|)$
      \Comment{guidance, norm capped at $c$}
    \State $\Delta \bar w\sim\mathcal{N}(0,\I)$
    \State $u_{i-1}\gets u_i+\beta\big(\tfrac12 u_i+\hat v\big)\Delta t
           +\delta+\sqrt{\beta\Delta t}\,\Delta \bar w$
      \Comment{Euler--Maruyama on eq.~(\ref{eq:rev_sde})}
  \EndFor
  \State $u^{(m)} \leftarrow \bigl(u_0 + (1-\bar\alpha_0)\hat v_0\bigr)/\sqrt{\bar\alpha_0}$ \Comment{$\hat v_0 \leftarrow \mathsf{V}(u_0,t_0,\tau)$, \; $\bar\alpha_0 = \bar\alpha(t_0)$}
\EndFor
\State \Return $\{u^{(m)}\}_{m=1}^{M}$
\end{algorithmic}
\end{algorithm}

\section{Performance Metrics} \label{appen:metrics}

This section defines the performance metrics reported in our experiments.

\paragraph{NMSE.} With $u^{\dagger}$ the ground truth field and $\bar u$ the posterior mean over $M$ draws, NMSE is computed as ${\|u^{\dagger}-\bar u\|_2^2}/{\|u^{\dagger}\|_2^2}$, so that it is comparable across systems.

\paragraph{Coverage.} For each system, we obtain pixel-wise $M$ posterior samples, and report the percentage of pixels whose ground truth falls inside the central $68\%(16\text{th}-84\text{th percentiles})$ and $95\%(2.5\text{th}-97.5\text{th percentiles})$ posterior intervals, denoting as $\% C_{68}$ and $\% C_{95}$. For a calibrated posterior, the target values should be $68\%$ and $95\%$. Coverage above the target value indicates an over-dispersed, under-confident posterior; below it, an over-confident one. In the tables, we also report in parentheses the average absolute difference of the per-system coverage from its target values, which measures how far test fields are from being calibrated, which should be $0$ when every field is exactly calibrated, and smaller is better. 

\paragraph{Calibration Curve.} The calibration curves in Figure~\ref{fig:calibration} illustrate posterior calibration at each coverage level. For every pixel in every test field, we record the percentile rank of the ground truth lying within that pixel's posterior ensemble, i.e. the fraction of $M$ posterior samples at that pixel that fall below the ground truth value, a number in $[0,1]$. At each level $q$ on the horizontal axis, we plot the fraction of pixels whose ground truth falls inside the central $q$ interval of its posterior -- equivalently, whose rank lies in \smash{$[\tfrac{1-q}{2},\tfrac{1+q}{2}]$} -- pooled over the test fields. The diagonal represents perfect calibration, and a curve below the diagonal means the posterior is too narrow and overconfident, while a curve above the diagonal means posterior is too wide and underconfident. Coverages $\% C_{68}$ and $\% C_{95}$ reported in the tables are the values of this curve at $q=0.68$ and $q=0.95$.

\paragraph{Continuous ranked probability score (CRPS). } CRPS is a proper scoring rule that measures how well a full predictive distribution matches the ground truth, by comparing the empirical cumulative distribution function (CDF) of the predicted values with the step-function CDF of the true value. CRPS is a generalisation of the mean absolute error (MAE) to probabilistic forecasts. Following \citet{gneiting2007strictly}, CRPS is estimated from $M$ posterior samples as
\begin{equation}
    \label{eq:crps}
    \mathrm{CRPS} = \frac{1}{M}\sum_{m=1}^{M} \big\|\tilde u^{(m)}-\tilde u^{\dagger}\big\|_{1} - \frac{1}{2M(M-1)}\sum_{m,n=1}^{M} \big\|\tilde u^{(m)}-\tilde u^{(n)}\big\|_{1},
\end{equation}
where $\tilde u^{(m)}$ is one normalised posterior sample and $\tilde u^{\dagger}$ is the normalised ground truth field. Normalised field is used here with $\tilde u = u / \mathrm{RMS}(u^{\dagger})$ as the magnitude of fields varies greatly across systems. Lower CRPS indicates a more accurate and sharper (confident/narrow distributions) posterior ensemble, while higher CRPS indicates ensembles that are either biased or overly diffuse. However, CRPS penalises over-confidence only mildly: as the posterior spread shrinks, the score tends to the MAE of samples, so a posterior whose samples are close to the ground truth attains a low CRPS even if its uncertainty is too narrow to cover the truth.

\paragraph{Train \& Predict Time. } Reported training time is the total wall-clock time of a run, and each model is trained for enough time for the validation loss to stop improving. Note that in our tables, unconditional models take longer to train because their validation loss converges slower than conditional models. Prediction time is the mean running time to perform inference on a single system, for AE it includes NUTS warmup steps.

\section{Additional Experimental Details and Results} \label{appen:results}

All experiments are run on a single RTX4090 GPU. This appendix provides more details and evaluations for each of the three experiments including:

\begin{enumerate}
    \item \textbf{Problem Setup:} the governing equations, the solver or data source, further example fields, and details of synthetic text generation;

    \item \textbf{Experimental Setup:} the choice of
    $\mathrm{score}_{\mathrm{field}}$ at the fine-tuning stage, network architectures, and hyperparameters used for training and inference;

    \item \textbf{Additional Results:} more evaluations including samples from the learned priors, per-system coverages, posterior histograms, ablation studies over $n_{obs}$ and how text contributes at different $n_{obs}$.
\end{enumerate}

\subsection{Steady-State Heat Equation}

\subsubsection{Problem Setup} \label{appen:heat_setup}

\begin{figure}[t!]
    \begin{center}
    \includegraphics[scale=0.5]{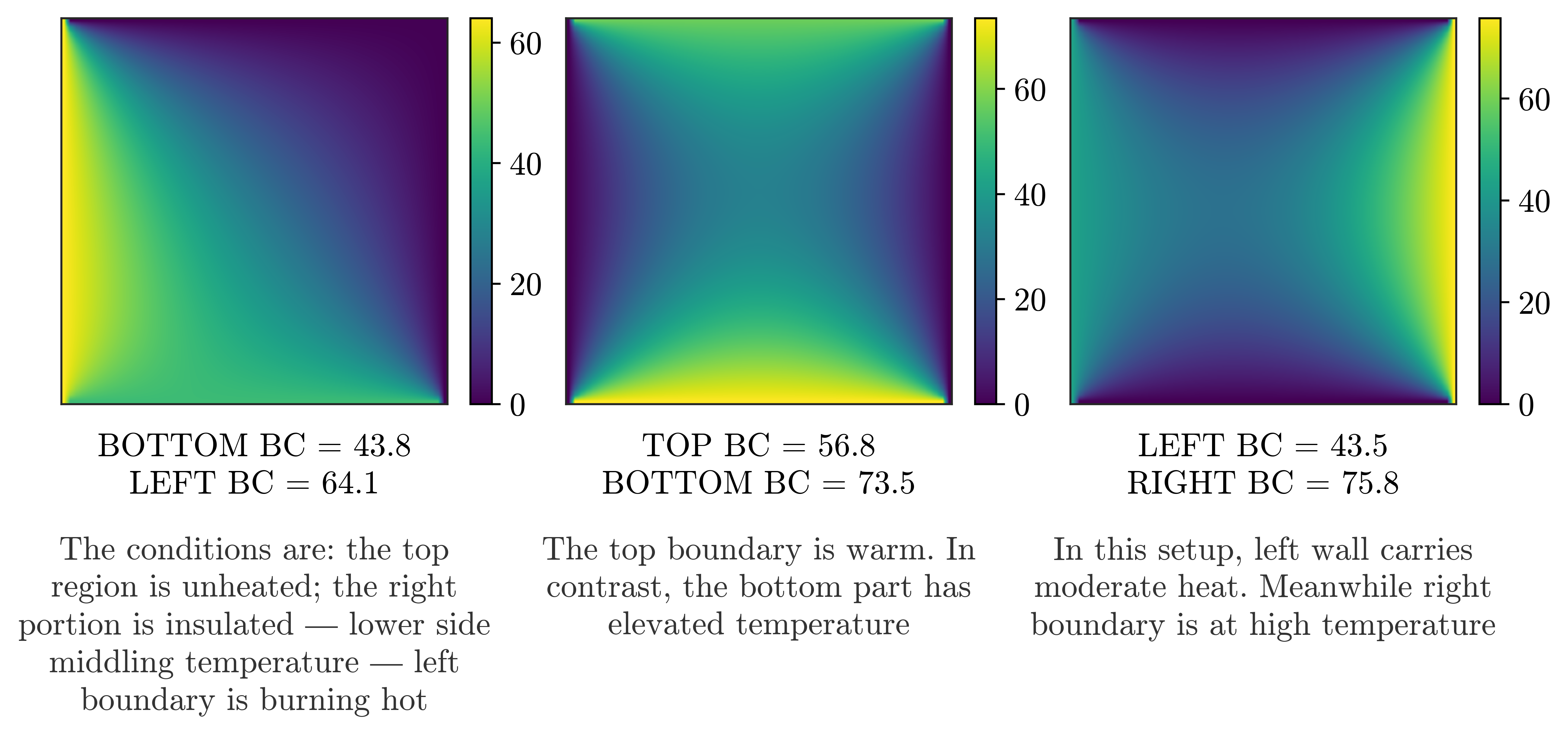}
    \end{center}
    \caption{Examples of heat fields with corresponding BC values and text descriptions.}
    \label{fig:heat_examples}
\end{figure}

The PDE used to generate the dataset is a steady-state heat equation on a unit square with Dirichlet boundary conditions,
\begin{equation}\label{eq:heat}
    \begin{aligned}
    \Delta u(x) &= 1,    &&\quad x\in\Omega=(0,1)^2,\\
    u(x) &= g(x),        &&\quad x\in\partial\Omega,
    \end{aligned}
\end{equation}
where $g$ is piecewise constant on the four sides: a random two of the four sides are held at $g=0$ and the remaining two at random constant temperatures drawn from $\text{Unif}[10,80]$. We use a finite difference solver on a uniform $64\times64$ grid.

To generate synthetic texts from given fields, temperatures are split into three overlapping bands: \textit{hot} $[54,80]$, \textit{warm} $[30,60]$ and \textit{cold} $[10,36]$. The bands overlap, and a temperature falling in two of them is described by a random one of the two. In this first and simplest experiment we did not query an API programmatically, instead we prompt \textit{Claude Sonnet 5} directly through a coding agent to write the descriptions. The prompt requires each description to mention boundaries carrying non-zero temperatures, and express it using varied synonyms of \textit{hot}, \textit{warm} and \textit{cold} with different sentence structures, the two unheated boundaries are mentioned in some descriptions but not all. A few examples are provided in Figure \ref{fig:heat_examples}.

\begin{figure}[t!]
    \begin{center}
    \includegraphics[scale=0.12]{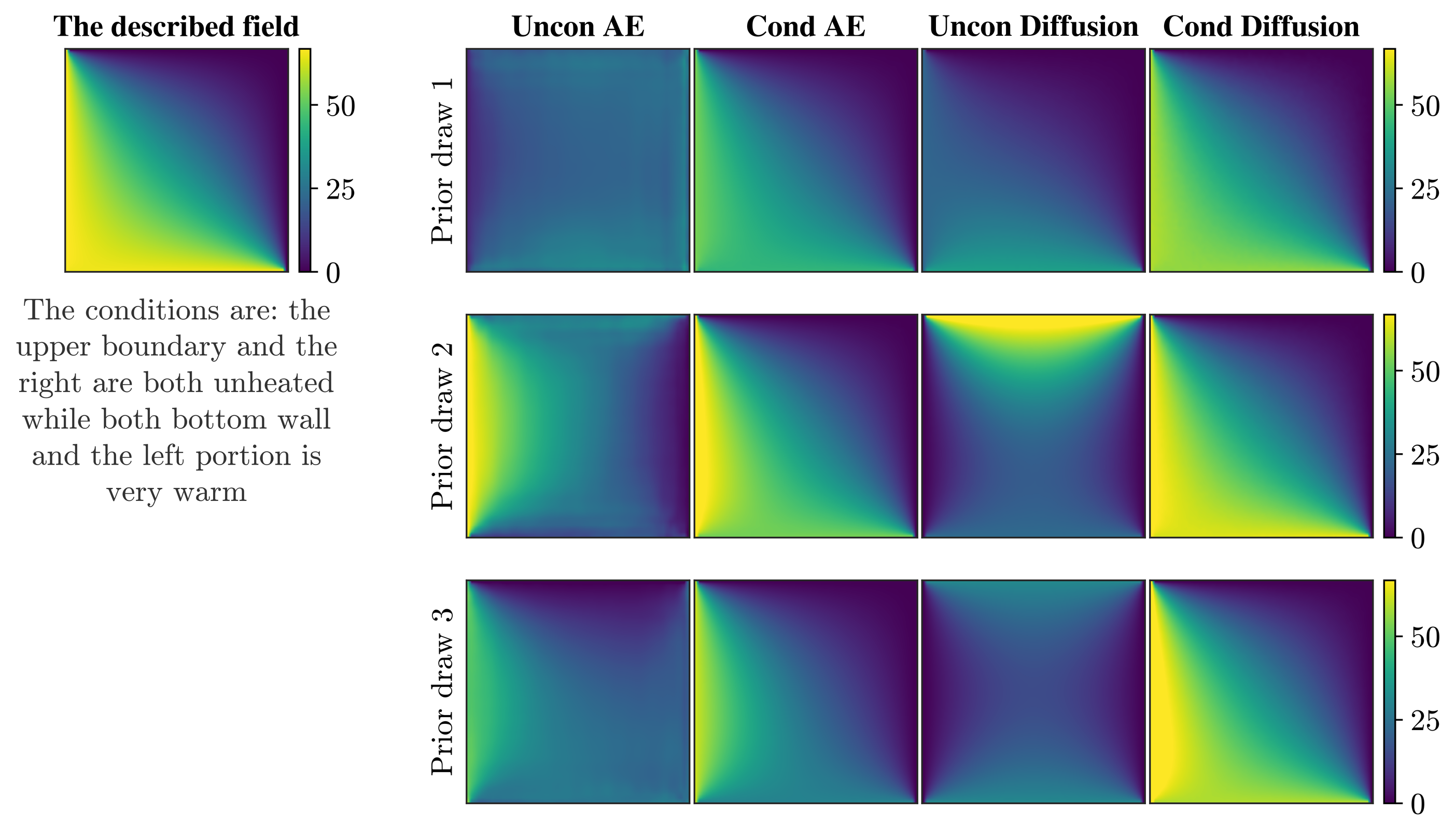}
    \end{center}
    \caption{\textbf{Prior samples generated given text.} The leftmost panel shows the ground-truth field together with the corresponding text. Each of the four rows shows three prior samples drawn using four models, the text is supplied to conditional models and not provided to unconditional ones.}
    \label{fig:heat_prior}
\end{figure}

\begin{figure}[t!]
    \begin{center}
    \includegraphics[scale=0.48]{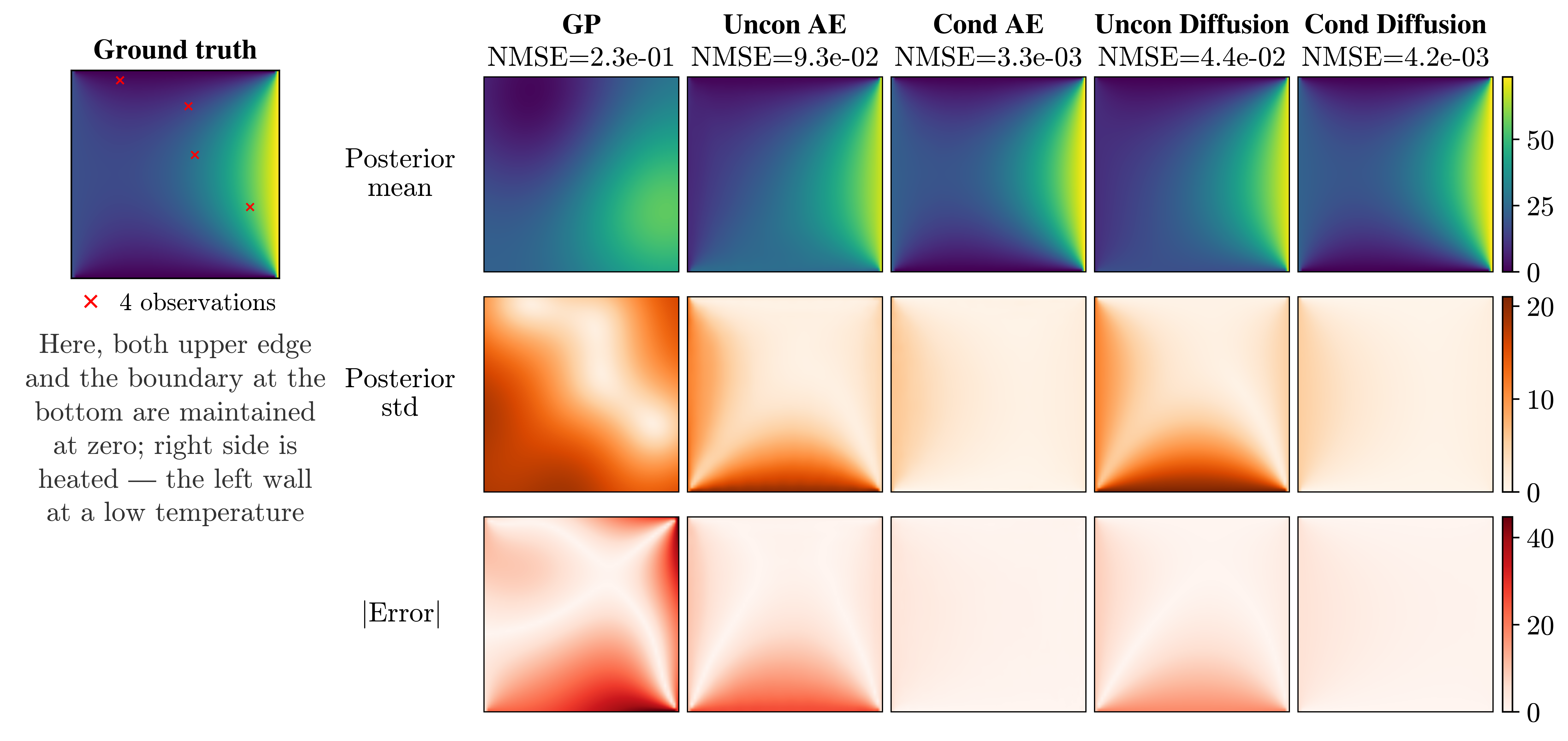}
    \end{center}
    \caption{\textbf{Posterior reconstructions of one heat equation test system at $\mathbf{n_{\mathrm{\textbf{obs}}}=4}$ (same system as in Figure \ref{fig:heat_posterior_histograms}).} Each row indicates posterior mean, posterior standard deviation, $|$error$|$ of the mean respectively. NMSE of posterior mean of each model is reported as well.}
    \label{fig:heat_single_system}
\end{figure}

\subsubsection{Experimental Setup}

\begin{table}[t]
  \centering
  \caption{\textbf{Full inference results on the heat equation test set with $\mathbf{n_{\textbf{obs}}}$ varying from 2 to 10.} This table is the extended version of Table \ref{tab:heat_inference}, and also presents the diffusion inference results obtained via the standard DPS scheme (last two rows). Continuous ranked probability score (CRPS) is also reported, with details in Appendix \ref{appen:metrics}.}
  \label{tab:heat_inference_full}
  \setlength{\tabcolsep}{3pt}
  \resizebox{\textwidth}{!}{%
  \begin{tabular}{@{}l@{}}
  \begin{tabular}{l |
    S[table-format=3.2]@{\,$\pm$\,}S[table-format=3.2]
    S[table-format=2.2]@{\,$\pm$\,}S[table-format=2.2]
    S[table-format=2.1]@{\,}c
    S[table-format=2.1]@{\,}c |
    S[table-format=3.2]@{\,$\pm$\,}S[table-format=3.2]
    S[table-format=2.2]@{\,$\pm$\,}S[table-format=2.2]
    S[table-format=2.1]@{\,}c
    S[table-format=2.1]@{\,}c |
    S[table-format=3.2]@{\,$\pm$\,}S[table-format=3.2]
    S[table-format=2.2]@{\,$\pm$\,}S[table-format=2.2]
    S[table-format=2.1]@{\,}c
    S[table-format=2.1]@{\,}c}
    \toprule
    & \multicolumn{8}{c|}{$n_{\text{obs}}=2$}
    & \multicolumn{8}{c|}{$n_{\text{obs}}=4$}
    & \multicolumn{8}{c}{$n_{\text{obs}}=6$} \\
    \cmidrule(lr){2-9} \cmidrule(lr){10-17} \cmidrule(lr){18-25}
    \textbf{Model}
    & \multicolumn{2}{c}{\makecell{NMSE $\downarrow$\\ $(\times 10^{-3})$}}
      & \multicolumn{2}{c}{\makecell{CRPS $\downarrow$\\ $(\times 10^{-2})$}}
      & \multicolumn{2}{c}{\makecell{$\%C_{68}$ }}
      & \multicolumn{2}{c|}{\makecell{$\%C_{95}$ }}
    & \multicolumn{2}{c}{\makecell{NMSE $\downarrow$\\ $(\times 10^{-3})$}}
      & \multicolumn{2}{c}{\makecell{CRPS $\downarrow$\\ $(\times 10^{-2})$}}
      & \multicolumn{2}{c}{\makecell{$\%C_{68}$ }}
      & \multicolumn{2}{c|}{\makecell{$\%C_{95}$ }}
    & \multicolumn{2}{c}{\makecell{NMSE $\downarrow$\\ $(\times 10^{-3})$}}
      & \multicolumn{2}{c}{\makecell{CRPS $\downarrow$\\ $(\times 10^{-2})$}}
      & \multicolumn{2}{c}{\makecell{$\%C_{68}$ }}
      & \multicolumn{2}{c}{\makecell{$\%C_{95}$ }} \\
    \midrule
    GP
      & 315.79 & 176.38 & 30.68 &  7.54 & 73.1 & (13.0) & 95.6 & ( 5.0)
      & 178.87 & 103.15 & 21.68 &  6.17 & 71.7 & (14.8) & 93.6 & ( 5.5)
      & 135.62 &  82.94 & 18.19 &  5.63 & 68.3 & (14.8) & 91.8 & ( 6.4) \\
    Uncon Supervised
      & 208.16 & 292.41 & \multicolumn{2}{c}{--} & \multicolumn{2}{c}{--} & \multicolumn{2}{c|}{--}
      &  54.66 & 203.76 & \multicolumn{2}{c}{--} & \multicolumn{2}{c}{--} & \multicolumn{2}{c|}{--}
      &  12.99 &  45.38 & \multicolumn{2}{c}{--} & \multicolumn{2}{c}{--} & \multicolumn{2}{c}{--} \\
    Cond Supervised
      &  17.95 &  34.60 & \multicolumn{2}{c}{--} & \multicolumn{2}{c}{--} & \multicolumn{2}{c|}{--}
      &   9.52 &  17.02 & \multicolumn{2}{c}{--} & \multicolumn{2}{c}{--} & \multicolumn{2}{c|}{--}
      &   5.82 &  12.28 & \multicolumn{2}{c}{--} & \multicolumn{2}{c}{--} & \multicolumn{2}{c}{--} \\
    Uncon AE
      & 248.75 & 377.03 & 21.20 & 17.44 & 47.6 & (26.5) & 80.9 & (17.8)
      &  46.45 & 132.21 &  6.63 &  8.83 & 61.5 & (22.4) & 88.7 & (12.0)
      &   9.75 &  29.25 &  3.14 &  3.76 & 65.0 & (19.3) & 92.0 & ( 9.0) \\
    Cond AE
      &  15.42 &  24.84 &  4.80 &  3.76 & 63.2 & (24.5) & 88.3 & ( 8.6)
      &   6.13 &  11.70 &  2.89 &  2.52 & 61.6 & (23.4) & 88.2 & ( 8.2)
      &   3.14 &   8.68 &  1.99 &  2.00 & 62.0 & (22.5) & 87.8 & ( 8.4) \\
    Uncon Diffusion
      & 239.28 & 291.64 & 19.93 & 13.37 & 62.0 & (22.5) & 90.1 & (10.7)
      &  37.71 & 107.23 &  5.02 &  6.40 & 64.9 & (25.2) & 89.2 & (12.2)
      &   7.36 &  22.12 &  2.30 &  2.92 & 59.0 & (27.0) & 86.5 & (14.8) \\
    Cond Diffusion
      &  15.79 &  37.39 &  4.69 &  4.47 & 60.3 & (26.4) & 90.4 & (11.7)
      &   7.12 &  17.64 &  2.88 &  3.02 & 62.3 & (25.2) & 89.8 & (12.0)
      &   3.11 &   8.93 &  1.91 &  2.10 & 63.9 & (24.3) & 89.5 & (12.2) \\
    Uncon Diffusion (DPS)
      & 235.67 & 310.17 & 19.50 & 13.23 & 59.0 & (21.4) & 80.9 & (16.5)
      &  39.47 & 118.32 &  5.34 &  6.72 & 37.4 & (35.8) & 65.2 & (31.3)
      &   7.18 &  20.47 &  2.55 &  2.88 & 29.6 & (41.1) & 54.8 & (41.1) \\
    Cond Diffusion (DPS)
      &  16.46 &  38.82 &  5.07 &  4.91 & 43.0 & (31.9) & 71.4 & (25.2)
      &   7.05 &  18.35 &  3.13 &  3.33 & 37.3 & (35.5) & 67.8 & (29.2)
      &   3.23 &   9.87 &  2.14 &  2.34 & 39.2 & (32.9) & 67.8 & (28.6) \\
    \midrule
  \end{tabular}
  \\
  \begin{tabular}{l |
    S[table-format=3.2]@{\,$\pm$\,}S[table-format=3.2]
    S[table-format=2.2]@{\,$\pm$\,}S[table-format=2.2]
    S[table-format=2.1]@{\,}c
    S[table-format=2.1]@{\,}c |
    S[table-format=3.2]@{\,$\pm$\,}S[table-format=3.2]
    S[table-format=2.2]@{\,$\pm$\,}S[table-format=2.2]
    S[table-format=2.1]@{\,}c
    S[table-format=2.1]@{\,}c}
    & \multicolumn{8}{c|}{$n_{\text{obs}}=8$}
    & \multicolumn{8}{c}{$n_{\text{obs}}=10$} \\
    \cmidrule(lr){2-9} \cmidrule(lr){10-17}
    \textbf{Model}
    & \multicolumn{2}{c}{\makecell{NMSE $\downarrow$\\ $(\times 10^{-3})$}}
      & \multicolumn{2}{c}{\makecell{CRPS $\downarrow$\\ $(\times 10^{-2})$}}
      & \multicolumn{2}{c}{\makecell{$\%C_{68}$ }}
      & \multicolumn{2}{c|}{\makecell{$\%C_{95}$ }}
    & \multicolumn{2}{c}{\makecell{NMSE $\downarrow$\\ $(\times 10^{-3})$}}
      & \multicolumn{2}{c}{\makecell{CRPS $\downarrow$\\ $(\times 10^{-2})$}}
      & \multicolumn{2}{c}{\makecell{$\%C_{68}$ }}
      & \multicolumn{2}{c}{\makecell{$\%C_{95}$ }} \\
    \midrule
    GP
      & 97.03 & 59.79 & 14.75 & 4.77 & 64.6 & (13.8) & 89.7 & ( 7.5)
      & 78.78 & 57.13 & 12.72 & 4.44 & 63.2 & (13.5) & 88.8 & ( 7.8) \\
    Uncon Supervised
      &  4.87 & 12.89 & \multicolumn{2}{c}{--} & \multicolumn{2}{c}{--} & \multicolumn{2}{c|}{--}
      &  2.87 &  5.10 & \multicolumn{2}{c}{--} & \multicolumn{2}{c}{--} & \multicolumn{2}{c}{--} \\
    Cond Supervised
      &  3.54 &  6.18 & \multicolumn{2}{c}{--} & \multicolumn{2}{c}{--} & \multicolumn{2}{c|}{--}
      &  2.90 &  5.41 & \multicolumn{2}{c}{--} & \multicolumn{2}{c}{--} & \multicolumn{2}{c}{--} \\
    Uncon AE
      &  4.85 & 15.95 &  2.32 & 2.94 & 67.6 & (17.8) & 92.6 & ( 8.5)
      &  4.25 & 14.05 &  2.10 & 2.83 & 67.7 & (17.7) & 92.8 & ( 8.3) \\
    Cond AE
      &  1.40 &  3.01 &  1.46 & 1.16 & 62.5 & (20.8) & 88.1 & ( 7.9)
      &  1.03 &  2.20 &  1.28 & 0.99 & 60.4 & (21.3) & 87.1 & ( 8.7) \\
    Uncon Diffusion
      &  2.45 &  8.95 &  1.44 & 1.62 & 59.8 & (26.0) & 87.0 & (14.3)
      &  1.09 &  2.77 &  1.17 & 1.06 & 58.9 & (26.7) & 85.4 & (15.4) \\
    Cond Diffusion
      &  1.44 &  3.54 &  1.40 & 1.29 & 64.3 & (22.7) & 90.3 & (11.2)
      &  1.16 &  3.31 &  1.24 & 1.23 & 64.5 & (23.0) & 90.1 & (11.1) \\
    Uncon Diffusion (DPS)
      &  2.74 &  8.71 &  1.64 & 1.57 & 31.5 & (39.3) & 55.9 & (40.1)
      &  1.40 &  4.67 &  1.40 & 1.25 & 30.5 & (39.0) & 54.8 & (41.1) \\
    Cond Diffusion (DPS)
      &  1.43 &  3.44 &  1.59 & 1.41 & 40.2 & (31.9) & 69.8 & (26.8)
      &  1.14 &  3.05 &  1.41 & 1.29 & 40.5 & (31.3) & 70.2 & (26.2) \\
    \bottomrule
  \end{tabular}
  \end{tabular}%
  }
\end{table}

\begin{figure}[t]
    \begin{center}
    \includegraphics[scale=0.405]{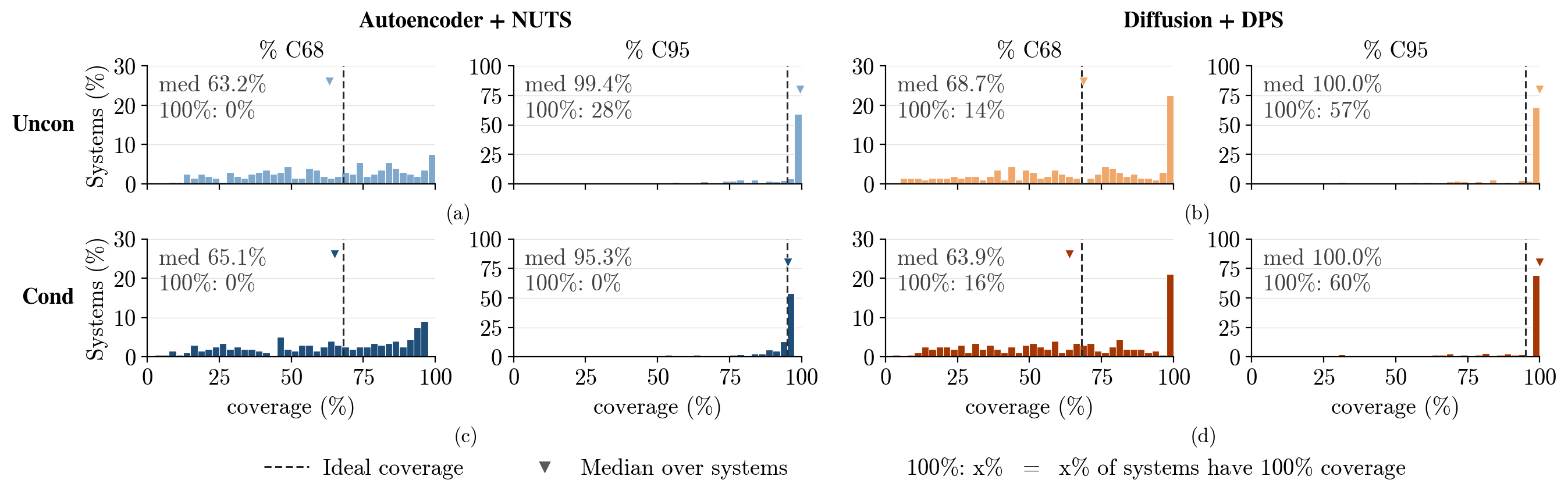}
    \end{center}
    \caption{\textbf{Per-system coverage over 200 steady-state heat equation test systems at $\mathbf{n_{\textbf{obs}}=4}$}. Each panel shows the distribution of coverages across all systems, i.e. the percentage of systems falling in each coverage bin. }
    \label{fig:heat_coverage}
\end{figure}

\paragraph{Fine-tuning.} $\text{score}_\text{field}$ in (\ref{eq:fine-tune}) measures similarity between numerical fields, in this experiment we use a Gaussian kernel on the min-max normalised fields $\tilde u$,
\begin{equation}
    \label{eq:heat_score}
    \mathrm{score}_{\mathrm{field}}(\tilde u,\tilde u')
    =\exp\Bigl(-\frac{\|\tilde u-\tilde u'\|_2^2}{T}\Bigr),
\end{equation}
which returns a value in $[0,1]$. $T=527.16$ controls how quickly the similarity decays with field distance, which is set so that $\mathrm{score}_{\text{field}} = 0.5$
at the median of $\lVert \tilde u - \tilde u'\rVert_2$ over random pairs of training fields; this makes sure that similarity targets are spread across $[0,1]$ rather than concentrated near either end. 3 epochs of fine-tuning on \texttt{all-MiniLM-L6-v2} are performed on $50,000$ training pairs with learning rate $5\times 10^{-5}$.

\paragraph{Train.} Supervised baselines use a U-Net with $64$ base channels, \textbf{Cond Supervised} is trained for $1k$ epochs, while \textbf{Uncon Supervised} for $2k$ epochs with the Adam optimiser \citep{kingma2014adam}, with a learning rate of $10^{-3}$ and batch size of $32$; AE uses FNO with $2$ layers (each for $\sfE$ and $\sfD$), width
$64$ and $12$ modes, latent dimension is $16$, \textbf{Cond AE} is trained for $3k$ epochs and \textbf{Uncon AE} for $5k$, with a learning rate of $10^{-3}$ and batch size $200$. We modify the original FNO structure into an autoencoder architecture, the encoder first applies FNO blocks to the field and then averages each channel over the spatial grid, and then a MLP maps the resulting vector to the latent; the decoder broadcasts the
latent back to a constant field over the grid before its own FNO blocks; the score network is a U-Net with $32$ base channels and a $64$-dimensional conditioning vector, \textbf{Cond Diffusion} is trained for $3k$ epochs and \textbf{Uncon Diffusion} for $5k$, with learning rate of $2 \times 10^{-4}$ and batch size of $32$.

\paragraph{Inference.} During inference, different numbers of $n_\text{obs}$ at random locations for each test system are drawn, with adding Gaussian noise of standard deviation $\sigma_y = 1$. \textbf{GP} regression uses RBF kernel with maximum marginal likelihood estimation of the hyperparameters. AE posteriors are sampled with NUTS in the latent space, using a single chain with $M_o=100$ warmup steps and $M=200$ retained draws. The maximum tree depth is $5$ for \textbf{Cond AE} and $8$ for \textbf{Uncon AE}, as it is harder for \textbf{Uncon AE} to converge. Diffusion posteriors are drawn by integrating the guided reverse SDE over $L=800$ uniform steps and the guidance clip $c=1.0$. The guidance clip is chosen so that around $10 \%$ of initial reverse diffusion steps are clipped, leaving the remainder unaffected and preserving the relative scale of prior and likelihood scores.

\begin{figure}[t]
    \begin{center}
    \includegraphics[scale=0.48]{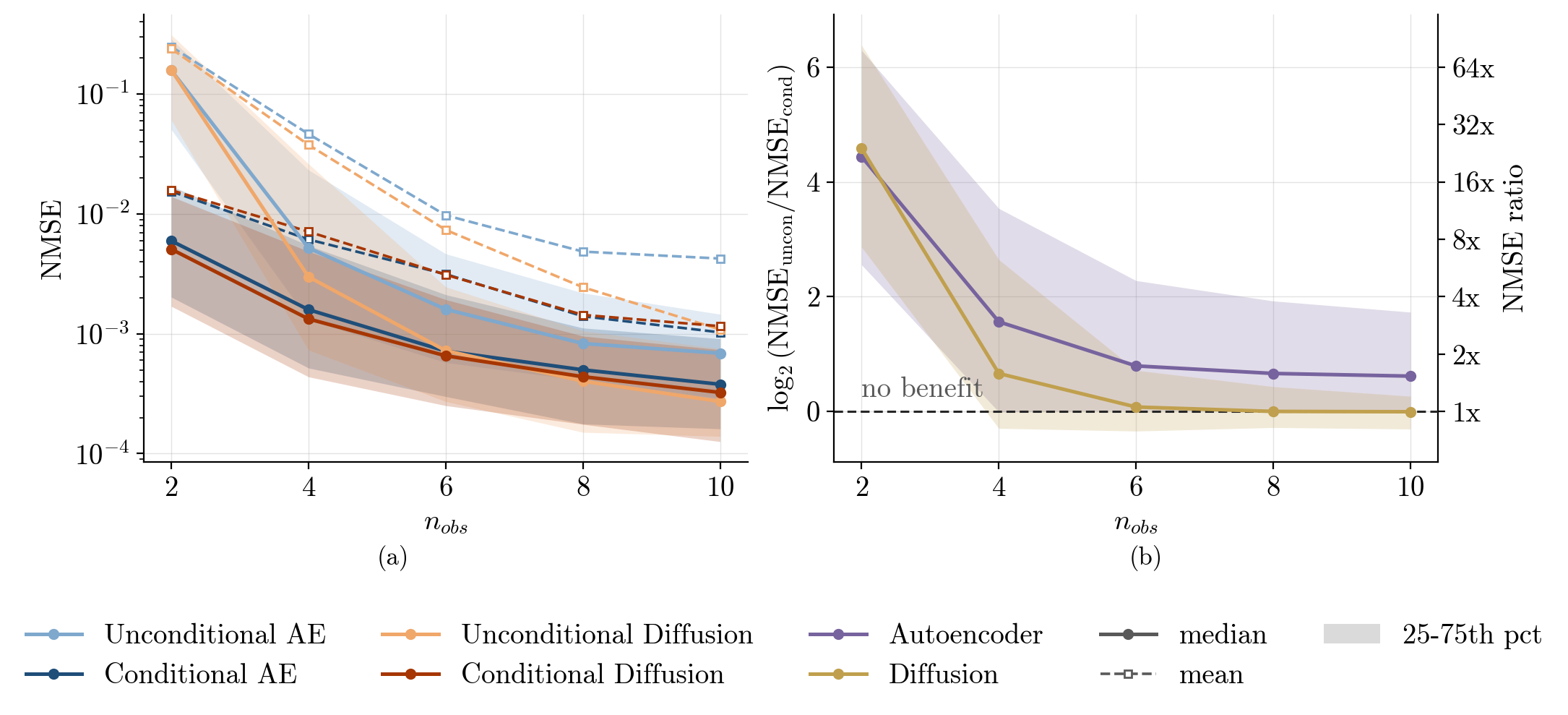}
    \end{center}
    \caption{\textbf{NMSE performance and text gain at different $\mathbf{n_{obs}}$ for  steady-state heat equation.} (a) NMSE across 200 test systems: median (solid), mean (dashed), and 25th–75th percentile bands, see Table~\ref{tab:heat_nmse_quantiles} for exact values. (b) Per-system relative gain from text conditioning, $\log_2(  \mathrm{NMSE}_{\mathrm{uncon}}/\mathrm{NMSE}_{\mathrm{cond}}) $, with median and interquartile bands. Zero values indicate that the text adds nothing beyond the observations.}
    \label{fig:heat_text_pair}
\end{figure}

\begin{table}[t]
  \centering
    \caption{\textbf{Quartiles of per-system NMSE across 200 heat-equation test systems.} $Q_{25}$, median and $Q_{75}$ over the same 200 test systems across different $n_{obs}$ are shown.}
  \label{tab:heat_nmse_quantiles}
  \setlength{\tabcolsep}{4pt}
  \resizebox{\textwidth}{!}{%
  \begin{tabular}{l | ccc | ccc | ccc | ccc | ccc}
    \toprule
    & \multicolumn{3}{c|}{$n_{\text{obs}}=2$}
    & \multicolumn{3}{c|}{$n_{\text{obs}}=4$}
    & \multicolumn{3}{c|}{$n_{\text{obs}}=6$}
    & \multicolumn{3}{c|}{$n_{\text{obs}}=8$}
    & \multicolumn{3}{c}{$n_{\text{obs}}=10$} \\
    \cmidrule(lr){2-4} \cmidrule(lr){5-7} \cmidrule(lr){8-10} \cmidrule(lr){11-13} \cmidrule(lr){14-16}
    \makecell[l]{\textbf{Model}\\ NMSE $(\times 10^{-3})$}
      & $Q_{25}$ & Median & $Q_{75}$
      & $Q_{25}$ & Median & $Q_{75}$
      & $Q_{25}$ & Median & $Q_{75}$
      & $Q_{25}$ & Median & $Q_{75}$
      & $Q_{25}$ & Median & $Q_{75}$ \\
    \midrule
    Uncon AE 
      & 50.64 & 158.02 & 287.87
      & 1.31 & 5.20 & 23.29
      & 0.58 & 1.60 & 4.65
      & 0.41 & 0.83 & 2.19
      & 0.29 & 0.69 & 1.46 \\
    Cond AE 
      & 2.02 & 5.96 & 17.10
      & 0.52 & 1.59 & 5.45
      & 0.30 & 0.72 & 2.12
      & 0.18 & 0.50 & 1.12
      & 0.16 & 0.38 & 0.91 \\
    Uncon Diffusion
      & 60.89 & 159.12 & 312.58
      & 0.73 & 2.99 & 26.18
      & 0.27 & 0.72 & 2.46
      & 0.15 & 0.40 & 1.06
      & 0.14 & 0.27 & 0.75 \\
    Cond Diffusion
      & 1.70 & 5.08 & 13.97
      & 0.44 & 1.33 & 4.87
      & 0.25 & 0.66 & 1.93
      & 0.18 & 0.44 & 0.96
      & 0.13 & 0.32 & 0.74 \\
    \bottomrule
  \end{tabular}%
  }
\end{table}

\begin{table}[t]
  \centering
  \caption{\textbf{Effective sample size (ESS) of NUTS on the heat equation, averaged over the
  200 test systems.} This table reports $\text{ESS}_{\min}$ and $\text{ESS}_{\text{mean}}$, the minimum and the mean over the $16$-dimensional latent vector of $M = 200$ posterior samples.}
  \label{tab:heat_nuts_ess}
  \setlength{\tabcolsep}{4pt}
  \resizebox{\textwidth}{!}{%
  \begin{tabular}{l |
    S[table-format=3.1] S[table-format=3.1] |
    S[table-format=3.1] S[table-format=3.1] |
    S[table-format=3.1] S[table-format=3.1] |
    S[table-format=3.1] S[table-format=3.1] |
    S[table-format=3.1] S[table-format=3.1]}
    \toprule
    & \multicolumn{2}{c|}{$n_{\text{obs}}=2$}
    & \multicolumn{2}{c|}{$n_{\text{obs}}=4$}
    & \multicolumn{2}{c|}{$n_{\text{obs}}=6$}
    & \multicolumn{2}{c|}{$n_{\text{obs}}=8$}
    & \multicolumn{2}{c}{$n_{\text{obs}}=10$} \\
    \cmidrule(lr){2-3} \cmidrule(lr){4-5} \cmidrule(lr){6-7} \cmidrule(lr){8-9} \cmidrule(lr){10-11}
    \textbf{Model}
    & {$\text{ESS}_{\min}$} & {$\text{ESS}_{\text{mean}}$}
    & {$\text{ESS}_{\min}$} & {$\text{ESS}_{\text{mean}}$}
    & {$\text{ESS}_{\min}$} & {$\text{ESS}_{\text{mean}}$}
    & {$\text{ESS}_{\min}$} & {$\text{ESS}_{\text{mean}}$}
    & {$\text{ESS}_{\min}$} & {$\text{ESS}_{\text{mean}}$} \\
    \midrule
    Uncon AE &  53.7 & 116.2 &  62.0 & 124.1 &  64.6 & 127.0 &  65.6 & 128.4 &  63.9 & 124.0 \\
    Cond AE  & 102.9 & 191.7 & 100.5 & 190.2 & 100.4 & 188.7 &  99.7 & 186.2 & 100.8 & 193.4 \\
    \bottomrule
  \end{tabular}%
  }
\end{table}

\subsubsection{Additional Results} \label{appen:heat_results}

\begin{figure}[t]
    \begin{center}
    \includegraphics[scale=0.425]{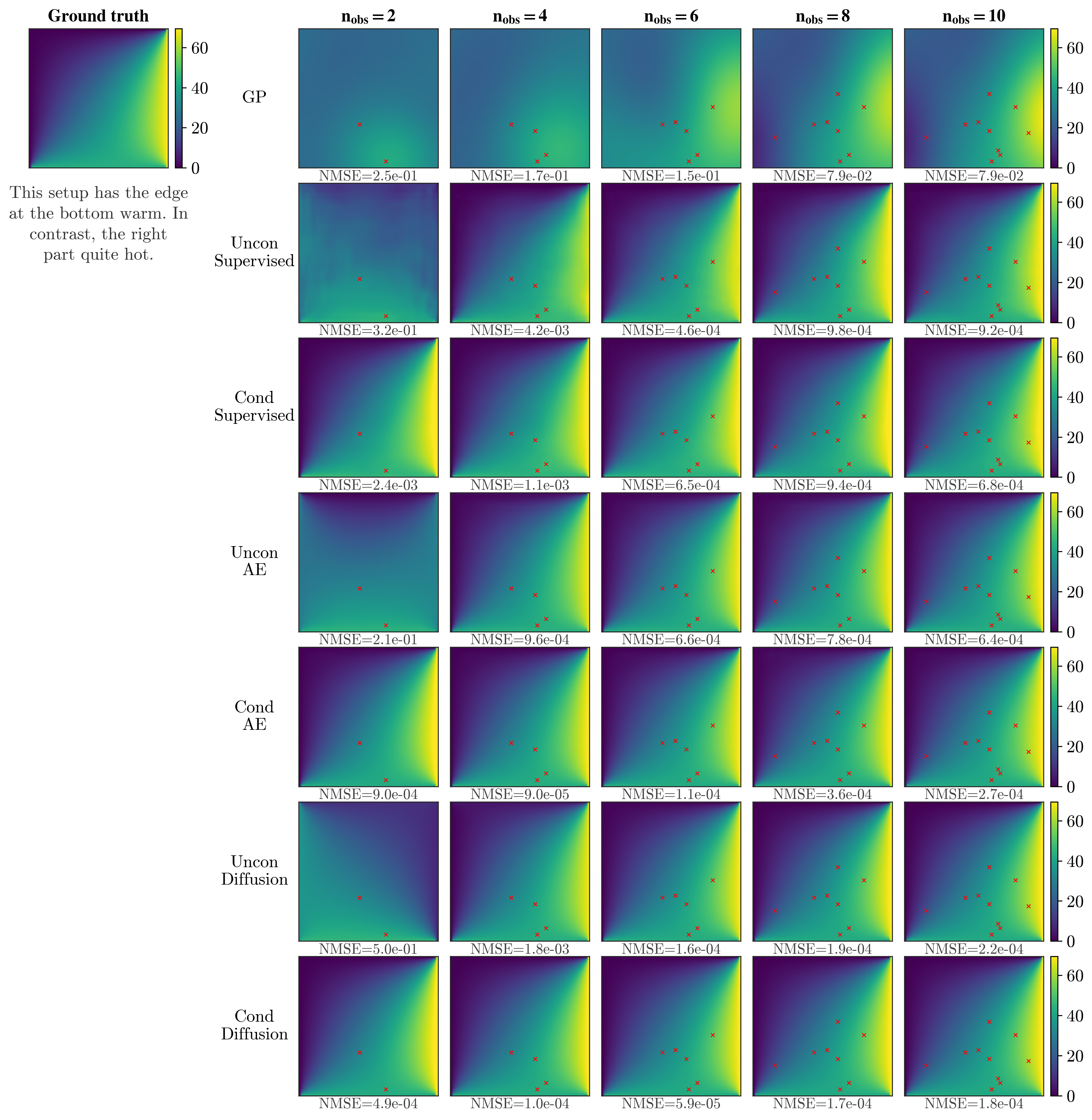}
    \end{center}
    \caption{\textbf{Posterior inference of one steady-state heat equation test system as the number of observation points increases.} The leftmost panel shows the ground truth field together with the text description. Columns give $n_{obs} = \{2,4,6,8,10\}$ with nested observation sets. NMSE is also provided at the bottom of each field.}
    \label{fig:heat_obs_graph2}
\end{figure}

Next, we provide some additional results on the first experiment. 

\paragraph{Inference Results.} Firstly, we present the full table of inference results across $n_{\text{obs}}=\{ 2,4,6,8,10\}$ in Table~\ref{tab:heat_inference_full}. The last two rows report standard DPS, with the guidance strength $\zeta'$ tuned over the range recommended by \citet{chung2022diffusion}, and we take $\zeta'=0.75$. Comparing standard DPS with our modified one in Table~\ref{tab:heat_inference_full}, the two methods achieve similar NMSE but standard DPS attains much lower coverage, meaning that its posteriors are over-confident. Standard DPS has higher CRPS as well. Figure \ref{fig:heat_single_system} shows, for a single system at $n_{\mathrm{obs}}=4$, the posterior inference results including the posterior mean, posterior standard deviation and absolute error of the mean of five probabilistic models. It can be seen that unconditional models produce both a larger posterior variance and a larger error. Figure~\ref{fig:heat_coverage} plots the distribution of coverages across test systems, showing that coverages for diffusion models are less concentrated around targeted coverage level. Table \ref{tab:heat_nuts_ess} reports effective sample size (ESS) of the NUTS chains, showing that \textbf{Cond AE} has larger ESS than \textbf{Uncon AE}.

\paragraph{Prior Generation.} Figure~\ref{fig:heat_prior} provides samples drawn directly from the learned priors, with text provided to conditional models. Priors generated from conditional models are more consistent with the described field, whereas samples from unconditional ones spread over the whole dataset.

\paragraph{Ablation on number of observation points and text contribution. } Figure~\ref{fig:heat_obs_graph2} shows the performance of all seven models over a single system with $n_{\mathrm{obs}}$ sweeping from $2$ to $10$; for supervised
models the network output is shown, and for probabilistic ones the posterior mean is provided. Unconditional models are visibly inaccurate at small $n_{\mathrm{obs}}$, due to the fact that without additional text information, observations alone cannot determine the field. Figure~\ref{fig:heat_text_pair}(a) reports the mean, median and interquartile range of the posterior mean NMSE over $200$ test systems for the four generative
models, with exact values provided in Table~\ref{tab:heat_nmse_quantiles}. The gap between conditional and unconditional models widens as $n_{\mathrm{obs}}$ decreases. Figure~\ref{fig:heat_text_pair}(b) makes this explicit by calculating per-system relative gain from $\log_2(\mathrm{NMSE}_{\mathrm{uncon}}/\mathrm{NMSE}_{\mathrm{cond}})$, showing median and interquartile range across systems. A zero value of text gain indicates that the text adds nothing beyond observations.

\subsection{Helmholtz Equation}

\subsubsection{Problem Setup} \label{appen:helm_setup}

The second experiment is based on a damped Helmholtz equation on the unit disk $\Omega=\{x\in\R^2:\|x\|_2\le 1\}$ with a Neumann boundary,
\begin{equation}
    \begin{aligned}
        \bigl(\Delta+\kappa+i\gamma\kappa\bigr)u(x) &= f(x), && x\in\Omega,\\
        \partial_n u(x) &= 0,                                 && x\in\partial\Omega,
    \end{aligned}
\end{equation}
forced uniformly over a small disk,
\begin{equation}
    f(x)=A\,\mathbbm{1}_{\{\|x-c\|_2\le R\}},\qquad R=0.1 .
\end{equation}
Here $\kappa=200$ is the squared wavenumber, $\gamma$ the damping ratio, and $A$ and $c$ the amplitude and the coordinate of the forcing centre. Each system draws $c$ uniformly over the inner disk of radius $0.8$, and $A$ and $\gamma$ log-uniformly from $[1,10]$ and $[0.02,0.2]$. The solution is complex, so each system consists of a real and an imaginary field. The Finite element method (FEM) in FEniCSx is used on an unstructured triangular mesh of target size $h=0.03$. Then the solution is interpolated onto a $96\times96$ grid over $(-1.1,1.1)^2$ masked to the disk.

About text generation, descriptions are produced by prompting \textit{DeepSeek-V4-Flash} \citep{liu2024deepseek} through its API. Each text consists of three pieces of information in qualitative natural language format: the damping across the whole disk, the magnitude of forcing, and its location, in terms of its angular position and a rough distance from the centre. For forcing magnitude $A$, damping $\gamma$ and radial position $\|c\|_2$, numerical parameters are first binned following Table \ref{tab:helm_bins}. Each bin carries a small pool of anchor words and phrases, one of which is drawn at random and passed to the model, so that numerical values are never shown. Anchor words are different for generating train and test text descriptions. The prompt instructs the model to replace the anchor word with a synonym of the same strength rather than copy it, which preserves the meaning of a description without reusing the same words. The angular position of forcing is reported at one of three different resolutions: a quadrant ($90^\circ$) (eg: in the north-east quadrant), one of eight compass points ($45^\circ$) (eg: in the west), or one of twelve clock directions ($30^\circ$) (eg: one o'clock direction), providing more ambiguity in generated texts. 

\begin{figure}[t!]
    \begin{center}
    \includegraphics[scale=0.45]{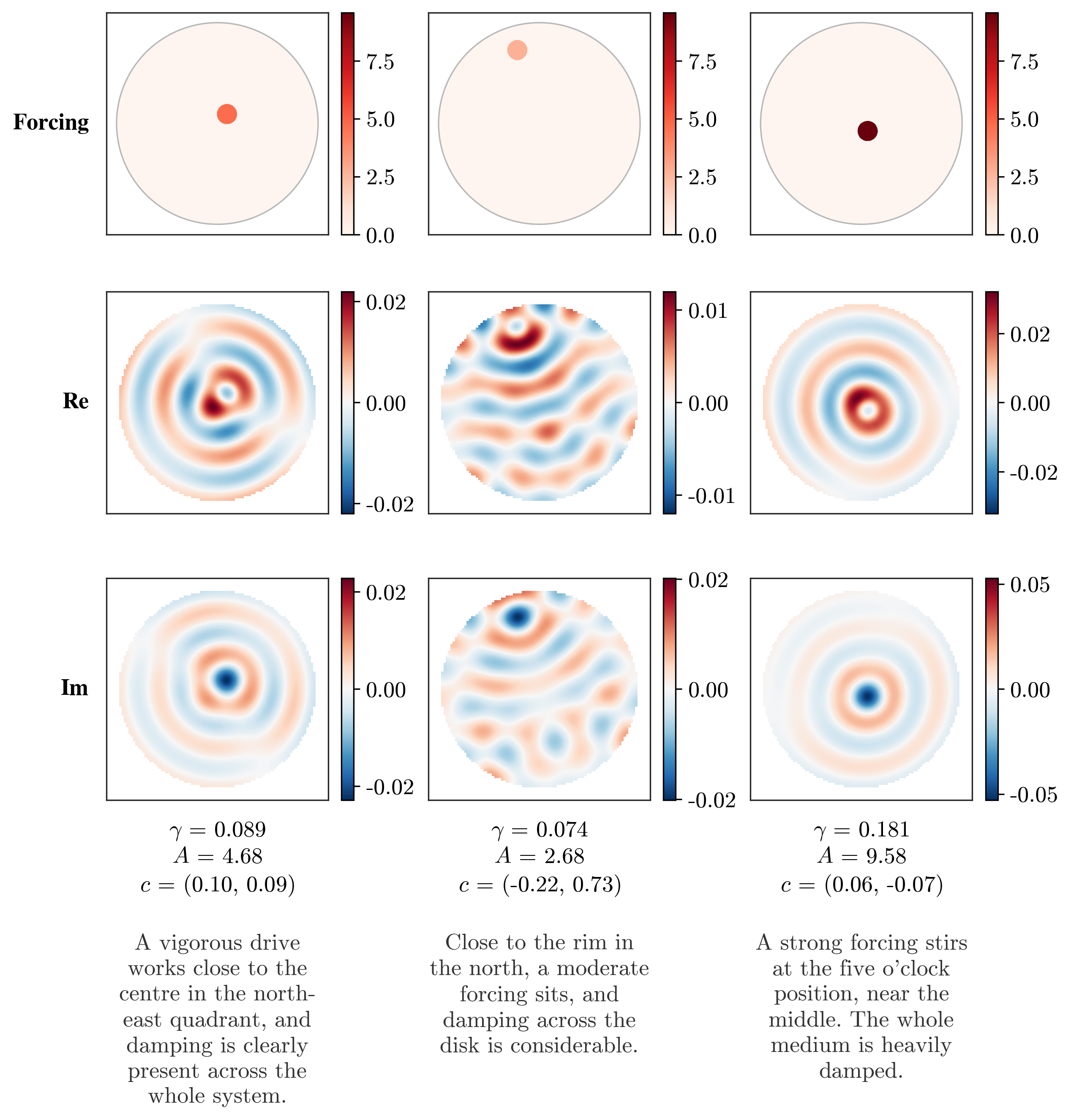}
    \end{center}
    \caption{Examples of Helmholtz fields with forcing fields, related parameters and descriptions.}
    \label{fig:helm_examples}
\end{figure}

\begin{table}[t]
  \caption{\textbf{Binning of the Helmholtz parameters for text generation.} Only the bin
  labels are passed to the language model and the numerical values are never shown.}
  \label{tab:helm_bins}
  \centering
  \begin{tabular}{@{}cc@{\hspace{1.8em}}cc@{\hspace{1.8em}}cc@{}}
    \toprule
    \multicolumn{2}{c}{Forcing magnitude $A$} &
    \multicolumn{2}{c}{Damping $\gamma$} &
    \multicolumn{2}{c}{Radial position $\|c\|_2$} \\
    \cmidrule(lr){1-2} \cmidrule(lr){3-4} \cmidrule(lr){5-6}
    Level & Range & Level & Range & Level & Range \\
    \midrule
        faint    & 1.00--2.15 & very light & 0.020--0.036 & centre & 0.00--0.25 \\
    moderate & 2.15--4.64 & light      & 0.036--0.063 & mid    & 0.25--0.55 \\
    strong   & 4.64--10.0 & moderate   & 0.063--0.113 & rim    & 0.55--0.80 \\
             &            & heavy      & 0.113--0.200 &        &            \\
    \bottomrule
  \end{tabular}
\end{table}

The following context provides the two halves to be provided to API. The system prompt is fixed across the whole dataset, which states the task and rules the descriptions must obey. The user message carries the content, with one line per system, giving its three binned facts. A few examples of Helmholtz fields and corresponding texts are provided in Figure~\ref{fig:helm_examples}.

\begin{figure}[t]
    \begin{center}
    \includegraphics[scale=0.475]{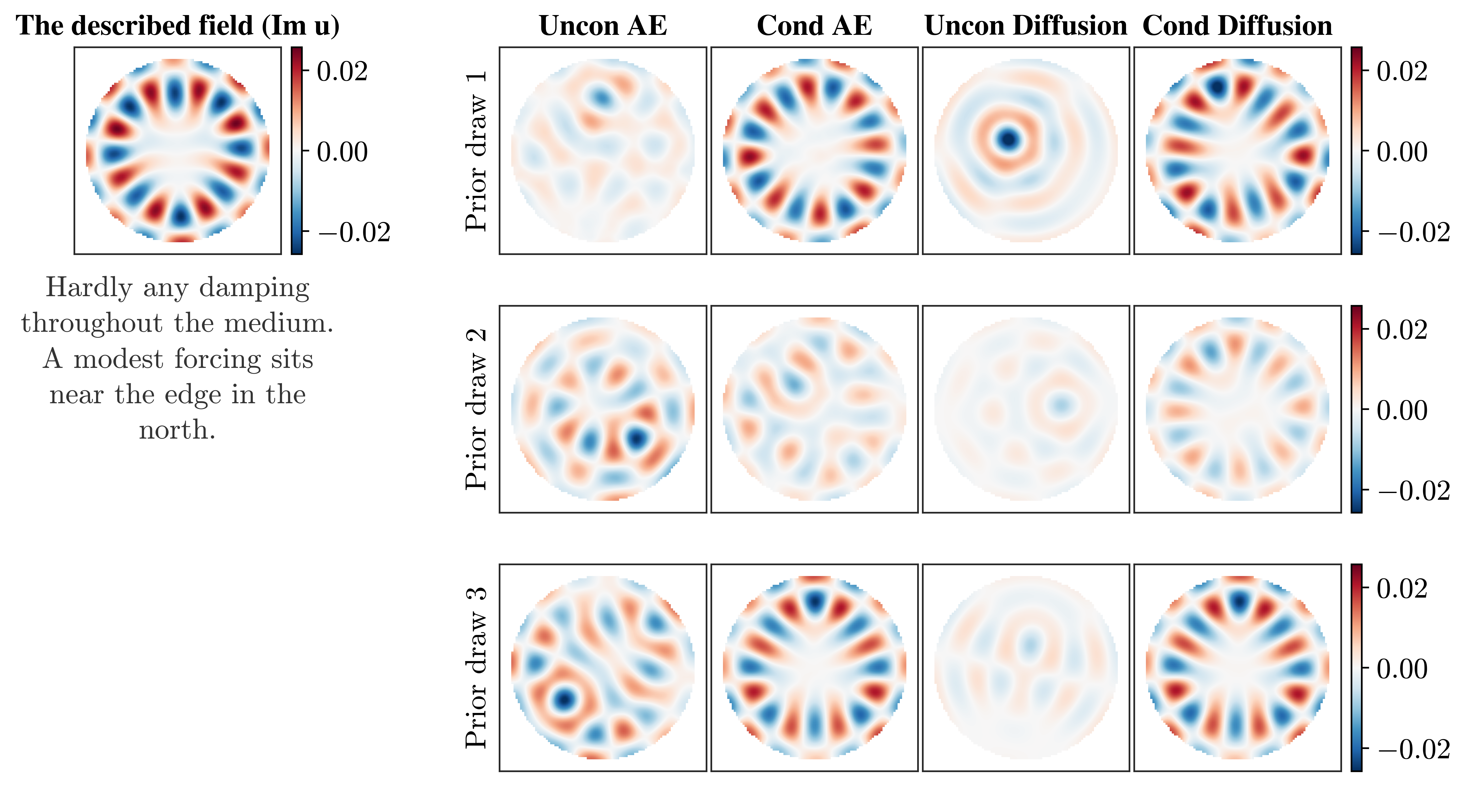}
    \end{center}
    \caption{Prior samples generated from four models (with text provided to conditional models).}
    \label{fig:helm_prior}
\end{figure}

\begin{figure}[t]
    \begin{center}
    \includegraphics[scale=0.44]{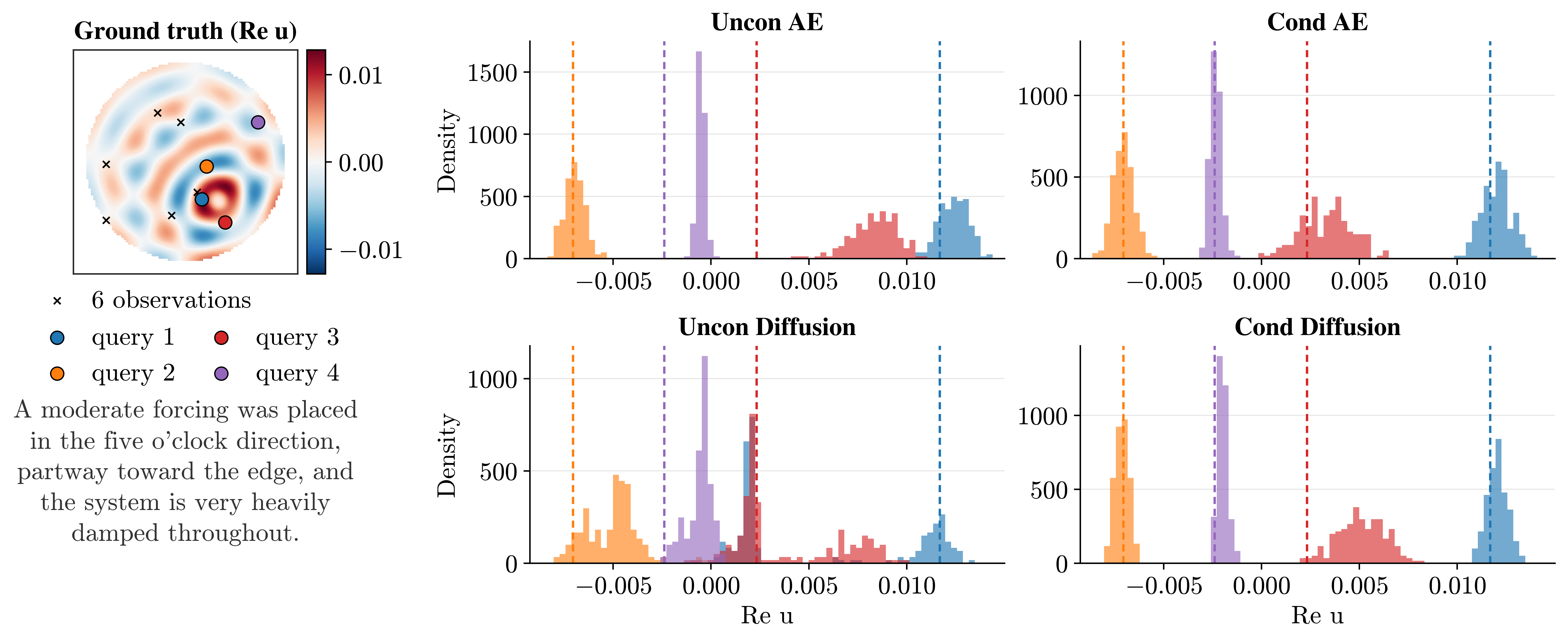}
    \end{center}
    \caption{Posterior marginals of one Helmholtz equation test system (Real part) (same system as in Figure~\ref{fig:helm_single_system}) at four unobserved query locations when $n_{\mathrm{\text{obs}}}=6$.}
    \label{fig:helm_posterior_histograms}
\end{figure}

\begin{table}[t]
  \centering
  \caption{\textbf{Full inference table on the Helmholtz equation test set.} This table extends Table \ref{tab:helm_inference} with $n_{\text{obs}}$ between 4 and 12 over $100$ test fields.}
  \label{tab:helm_inference_full}
  \setlength{\tabcolsep}{3pt}
  \resizebox{\textwidth}{!}{%
  \begin{tabular}{@{}l@{}}
  \begin{tabular}{l |
    S[table-format=2.2]@{\,$\pm$\,}S[table-format=3.2]
    S[table-format=2.2]@{\,$\pm$\,}S[table-format=2.2]
    S[table-format=2.1]@{\,}c
    S[table-format=2.1]@{\,}c |
    S[table-format=2.2]@{\,$\pm$\,}S[table-format=3.2]
    S[table-format=2.2]@{\,$\pm$\,}S[table-format=2.2]
    S[table-format=2.1]@{\,}c
    S[table-format=2.1]@{\,}c |
    S[table-format=2.2]@{\,$\pm$\,}S[table-format=3.2]
    S[table-format=2.2]@{\,$\pm$\,}S[table-format=2.2]
    S[table-format=2.1]@{\,}c
    S[table-format=2.1]@{\,}c}
    \toprule
    & \multicolumn{8}{c|}{$n_{\text{obs}}=4$}
    & \multicolumn{8}{c|}{$n_{\text{obs}}=6$}
    & \multicolumn{8}{c}{$n_{\text{obs}}=8$} \\
    \cmidrule(lr){2-9} \cmidrule(lr){10-17} \cmidrule(lr){18-25}
    \textbf{Model}
    & \multicolumn{2}{c}{\makecell{NMSE $\downarrow$\\ $(\times 10^{-2})$}}
      & \multicolumn{2}{c}{\makecell{CRPS $\downarrow$\\ $(\times 10^{-2})$}}
      & \multicolumn{2}{c}{\makecell{$\% C_{68}$}}
      & \multicolumn{2}{c|}{\makecell{$\%C_{95}$}}
    & \multicolumn{2}{c}{\makecell{NMSE $\downarrow$\\ $(\times 10^{-2})$}}
      & \multicolumn{2}{c}{\makecell{CRPS $\downarrow$\\ $(\times 10^{-2})$}}
      & \multicolumn{2}{c}{\makecell{$\% C_{68}$}}
      & \multicolumn{2}{c|}{\makecell{$\%C_{95}$}}
    & \multicolumn{2}{c}{\makecell{NMSE $\downarrow$\\ $(\times 10^{-2})$}}
      & \multicolumn{2}{c}{\makecell{CRPS $\downarrow$\\ $(\times 10^{-2})$}}
      & \multicolumn{2}{c}{\makecell{$\% C_{68}$}}
      & \multicolumn{2}{c}{\makecell{$\%C_{95}$}} \\
    \midrule
    GP
      & 92.45 &   8.42 & 74.49 &  31.10 & 81.4 & (23.0) & 93.1 & ( 8.1)
      & 86.74 &   9.20 & 74.66 &  33.42 & 83.3 & (23.0) & 94.0 & ( 7.4)
      & 81.62 &  10.47 & 69.13 &  28.71 & 81.4 & (23.0) & 93.0 & ( 8.2) \\
    Uncon Supervised
      & 54.07 & 28.83 & \multicolumn{2}{c}{--} & \multicolumn{2}{c}{--} & \multicolumn{2}{c|}{--}
      & 25.93 &  23.18 & \multicolumn{2}{c}{--} & \multicolumn{2}{c}{--} & \multicolumn{2}{c|}{--}
      & 14.17 &  12.89 & \multicolumn{2}{c}{--} & \multicolumn{2}{c}{--} & \multicolumn{2}{c}{--} \\
    Cond Supervised
      & 28.39 & 34.74 & \multicolumn{2}{c}{--} & \multicolumn{2}{c}{--} & \multicolumn{2}{c|}{--}
      & 19.77 &  25.27 & \multicolumn{2}{c}{--} & \multicolumn{2}{c}{--} & \multicolumn{2}{c|}{--}
      & 13.13 &  17.64 & \multicolumn{2}{c}{--} & \multicolumn{2}{c}{--} & \multicolumn{2}{c}{--} \\
    Uncon AE
      & 94.21 & 186.28 & 45.20 &  36.74 & 37.7 & (33.7) & 60.4 & (35.9)
      & 66.61 & 338.05 & 28.47 &  42.73 & 45.9 & (28.7) & 68.6 & (28.3)
      & 29.20 &  64.45 & 22.53 &  26.28 & 46.5 & (27.6) & 68.4 & (28.4) \\
    Cond AE
      & 23.89 &  46.08 & 18.68 &  20.59 & 67.7 & (18.6) & 89.6 & (10.7)
      &  7.15 &  17.18 &  9.62 &   9.43 & 71.4 & (13.6) & 94.5 & ( 5.5)
      &  4.07 &   8.99 &  7.79 &   6.78 & 70.3 & (12.5) & 94.1 & ( 5.6) \\
    Uncon Diffusion
      & 48.08 &  30.47 & 32.30 &  14.43 & 57.5 & (19.7) & 88.8 & (10.5)
      & 22.57 &  24.20 & 18.26 &  13.38 & 64.4 & (18.0) & 93.4 & ( 6.7)
      &  9.00 &  13.21 &  9.95 &   8.67 & 68.7 & (18.0) & 93.9 & ( 7.0) \\
    Cond Diffusion
      & 20.14 &  30.17 & 18.96 &  16.23 & 50.5 & (26.1) & 76.6 & (21.2)
      &  7.15 &  12.63 & 10.51 &   9.22 & 54.5 & (23.5) & 83.0 & (15.3)
      &  4.30 &   7.95 &  7.90 &   7.09 & 54.8 & (21.7) & 82.2 & (15.9) \\
    Uncon Diffusion (DPS)
      & 47.79 &  27.75 & 32.16 &  13.37 & 56.6 & (19.1) & 88.8 & ( 9.7)
      & 24.18 &  23.35 & 19.23 &  12.81 & 61.5 & (17.3) & 90.6 & ( 8.0)
      & 10.48 &  13.98 & 10.91 &   8.83 & 62.7 & (18.2) & 90.2 & ( 8.7) \\
    Cond Diffusion (DPS)
      & 19.12 &  28.91 & 19.21 &  15.71 & 45.2 & (28.3) & 72.6 & (24.0)
      &  7.19 &  13.22 & 11.40 &   9.50 & 47.4 & (26.8) & 76.9 & (19.7)
      &  4.42 &   8.75 &  7.91 &   7.30 & 46.9 & (24.9) & 75.2 & (20.9) \\
    \midrule
  \end{tabular}
  \\
  \begin{tabular}{l |
    S[table-format=2.2]@{\,$\pm$\,}S[table-format=3.2]
    S[table-format=2.2]@{\,$\pm$\,}S[table-format=2.2]
    S[table-format=2.1]@{\,}c
    S[table-format=2.1]@{\,}c |
    S[table-format=2.2]@{\,$\pm$\,}S[table-format=3.2]
    S[table-format=2.2]@{\,$\pm$\,}S[table-format=2.2]
    S[table-format=2.1]@{\,}c
    S[table-format=2.1]@{\,}c}
    & \multicolumn{8}{c|}{$n_{\text{obs}}=10$}
    & \multicolumn{8}{c}{$n_{\text{obs}}=12$} \\
    \cmidrule(lr){2-9} \cmidrule(lr){10-17}
    \textbf{Model}
    & \multicolumn{2}{c}{\makecell{NMSE $\downarrow$\\ $(\times 10^{-2})$}}
      & \multicolumn{2}{c}{\makecell{CRPS $\downarrow$\\ $(\times 10^{-2})$}}
      & \multicolumn{2}{c}{\makecell{$\% C_{68}$}}
      & \multicolumn{2}{c|}{\makecell{$\%C_{95}$}}
    & \multicolumn{2}{c}{\makecell{NMSE $\downarrow$\\ $(\times 10^{-2})$}}
      & \multicolumn{2}{c}{\makecell{CRPS $\downarrow$\\ $(\times 10^{-2})$}}
      & \multicolumn{2}{c}{\makecell{$\% C_{68}$}}
      & \multicolumn{2}{c}{\makecell{$\%C_{95}$}} \\
    \midrule
    GP
      & 77.12 &  11.96 & 67.76 &  29.42 & 82.3 & (23.1) & 93.6 & ( 7.7)
      & 72.21 &  12.94 & 65.36 &  28.73 & 82.3 & (23.0) & 93.7 & ( 7.6) \\
    Uncon Supervised
      & 10.17 &   7.68 & \multicolumn{2}{c}{--} & \multicolumn{2}{c}{--} & \multicolumn{2}{c|}{--}
      &  8.45 &   8.23 & \multicolumn{2}{c}{--} & \multicolumn{2}{c}{--} & \multicolumn{2}{c}{--} \\
    Cond Supervised
      & 10.84 &  16.48 & \multicolumn{2}{c}{--} & \multicolumn{2}{c}{--} & \multicolumn{2}{c|}{--}
      &  8.84 &  13.23 & \multicolumn{2}{c}{--} & \multicolumn{2}{c}{--} & \multicolumn{2}{c}{--} \\
    Uncon AE
      & 16.07 &  30.13 & 17.42 &  18.26 & 46.1 & (26.9) & 68.4 & (28.5)
      & 12.72 &  29.08 & 14.82 &  17.04 & 46.9 & (26.8) & 69.0 & (27.7) \\
    Cond AE
      &  3.16 &   7.25 &  6.77 &   5.80 & 69.7 & (11.8) & 93.8 & ( 5.4)
      &  2.35 &   4.83 &  6.00 &   4.51 & 68.9 & (12.0) & 93.1 & ( 5.7) \\
    Uncon Diffusion
      &  5.36 &   9.24 &  7.20 &   6.99 & 67.6 & (19.1) & 92.0 & ( 8.4)
      &  3.09 &   6.62 &  5.38 &   5.46 & 66.8 & (17.6) & 92.7 & ( 7.6) \\
    Cond Diffusion
      &  3.45 &   7.54 &  6.75 &   6.64 & 54.7 & (20.7) & 83.6 & (14.1)
      &  2.52 &   4.92 &  5.80 &   5.63 & 54.5 & (20.3) & 83.9 & (13.8) \\
    Uncon Diffusion (DPS)
      &  5.93 &   9.74 &  7.61 &   7.04 & 61.6 & (18.8) & 88.3 & (10.1)
      &  3.95 &   7.94 &  6.02 &   5.96 & 58.2 & (18.3) & 86.5 & (10.5) \\
    Cond Diffusion (DPS)
      &  3.63 &   8.95 &  6.79 &   7.13 & 47.2 & (24.1) & 75.8 & (20.3)
      &  2.71 &   5.94 &  5.85 &   5.97 & 47.0 & (23.9) & 76.4 & (19.3) \\
    \bottomrule
  \end{tabular}
  \end{tabular}%
  }
\end{table}

\begin{figure}[t]
    \begin{center}
     \includegraphics[scale=0.41]{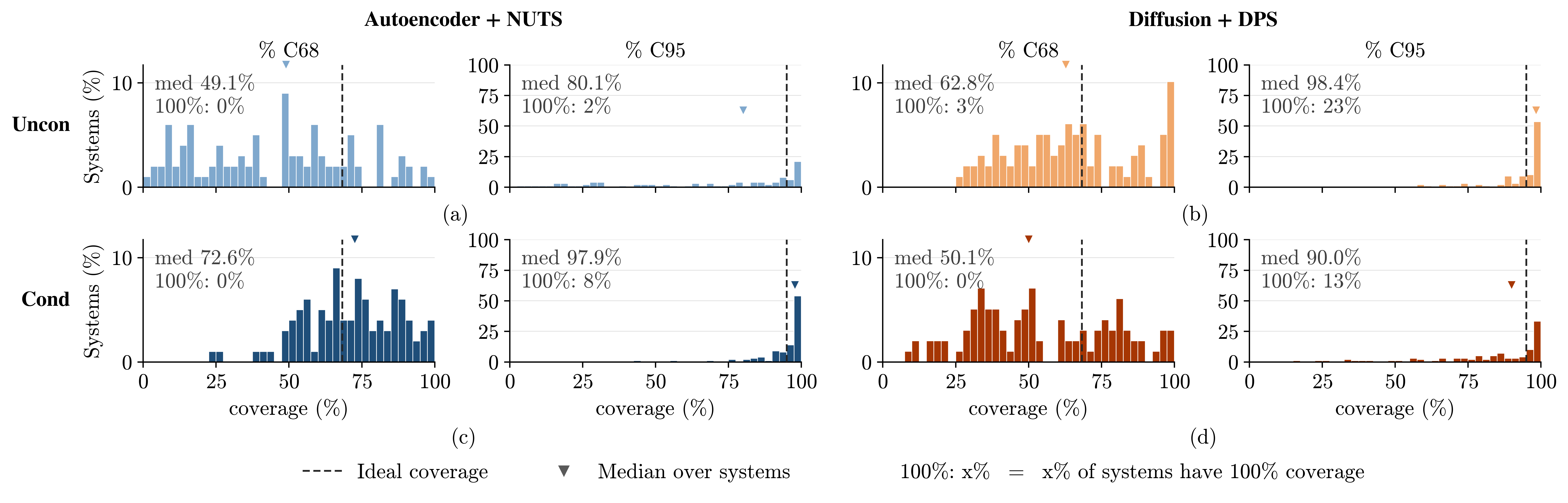}
    \end{center}
    \caption{Per-system coverages over 100 Helmholtz test systems at $n_{\text{obs}}=6.$} 
    \label{fig:helm_coverage}
\end{figure}

\begin{figure}[t]
    \begin{center}
    \includegraphics[scale=0.46]{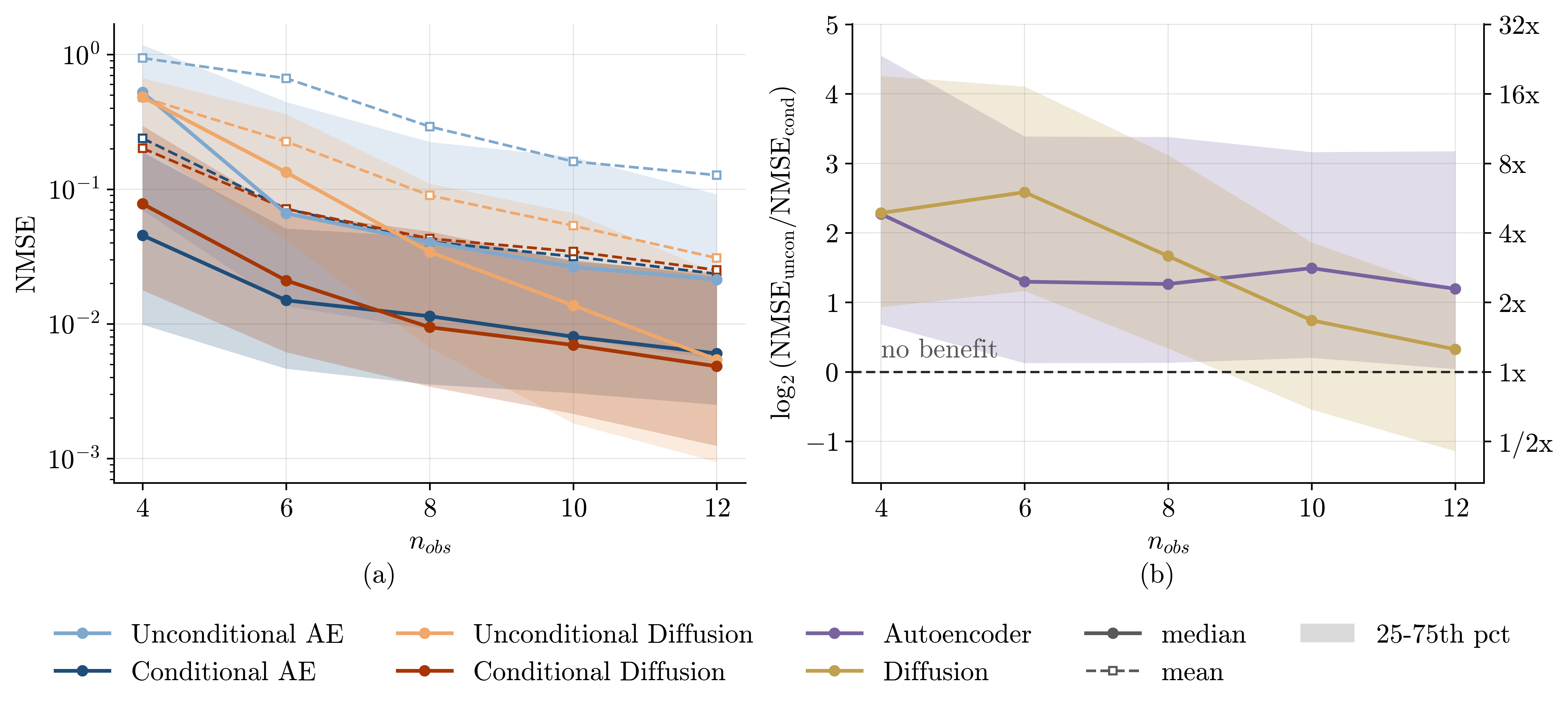}
    \end{center}
    \caption{\textbf{NMSE performance and text gain at different $\mathbf{n_\textbf{obs}}$ for Helmholtz equation.} (a) NMSE across 100 test systems: median (solid), mean (dashed), and 25th–75th percentile bands, see Table~\ref{tab:helm_nmse_quantiles} for exact values. (b) Per-system relative gain from text conditioning, $\log_2(  \mathrm{NMSE}_{\mathrm{uncon}}/\mathrm{NMSE}_{\mathrm{cond}}) $, with median and interquartile bands.}
    \label{fig:helm_text_pair}
\end{figure}

\begin{table}[t]
  \centering
  \caption{Quartiles of per-system NMSE across 100 Helmholtz test systems at different $n_\text{obs}$.}
  \label{tab:helm_nmse_quantiles}
  \setlength{\tabcolsep}{4pt}
  \resizebox{\textwidth}{!}{%
  \begin{tabular}{l | ccc | ccc | ccc | ccc | ccc}
    \toprule
    & \multicolumn{3}{c|}{$n_{\text{obs}}=4$}
    & \multicolumn{3}{c|}{$n_{\text{obs}}=6$}
    & \multicolumn{3}{c|}{$n_{\text{obs}}=8$}
    & \multicolumn{3}{c|}{$n_{\text{obs}}=10$}
    & \multicolumn{3}{c}{$n_{\text{obs}}=12$} \\
    \cmidrule(lr){2-4} \cmidrule(lr){5-7} \cmidrule(lr){8-10}
    \cmidrule(lr){11-13} \cmidrule(lr){14-16}
    \makecell[l]{\textbf{Model}\\ NMSE $(\times 10^{-2})$}
      & $Q_{25}$ & Median & $Q_{75}$
      & $Q_{25}$ & Median & $Q_{75}$
      & $Q_{25}$ & Median & $Q_{75}$
      & $Q_{25}$ & Median & $Q_{75}$
      & $Q_{25}$ & Median & $Q_{75}$ \\
    \midrule
    Uncon AE
      & 6.98 & 52.42 & 118.06
      & 1.36 & 6.62 & 44.52
      & 0.91 & 4.04 & 22.46
      & 0.78 & 2.65 & 17.39
      & 0.55 & 2.14 & 9.19 \\
    Cond AE
      & 0.99 & 4.55 & 18.83
      & 0.46 & 1.50 & 5.11
      & 0.35 & 1.14 & 4.41
      & 0.31 & 0.80 & 2.99
      & 0.25 & 0.60 & 2.26 \\
    Uncon Diffusion
      & 21.18 & 48.36 & 66.97
      & 4.16 & 13.37 & 36.19
      & 0.66 & 3.42 & 11.01
      & 0.18 & 1.37 & 6.66
      & 0.09 & 0.54 & 2.40 \\
    Cond Diffusion
      & 1.78 & 7.79 & 29.61
      & 0.62 & 2.10 & 6.79
      & 0.34 & 0.94 & 4.86
      & 0.21 & 0.70 & 2.91
      & 0.12 & 0.48 & 2.39 \\
    \bottomrule
  \end{tabular}%
  }
\end{table}

{\footnotesize
\textbf{System prompt}
\begin{verbatim}
Suppose you are an expert doing experiments on damped Helmholtz equation
on a disc, where you need to state experimental setups, including 
magnitude of damping, strength and location of force. You write short, 
natural language-style descriptions of experimental setups for each 
system, with three facts provided for each configuration:
- forcing position: where the forcing is applied (with angular 
  directions at different resolutions: a quadrant, a compass direction, 
  or a clock direction, plus a rough distance from the centre)
- forcing magnitude: an anchor word showing how large the forcing is
- damping: an anchor phrase for how strongly damped the WHOLE disk is 
  (one global property of the medium, without specific location)

Write ONE description per item. Rules:
- Every description is one or two COMPLETE sentences, each with a subject
  and a verb.
- Include ALL three facts. Although the items are listed in a fixed order
  (position, magnitude, damping) due to input formatting, please do NOT 
  write them in the same order. Start with describing forcing as well.
- Do NOT use the exact given anchor word. Represent the same level using 
  a different word, but the level must stay unchanged: a low-level anchor 
  must never use words sounding strong, nor a high-level anchor mild.
- Restructure the damping phrase as well, keep the level it represents the 
  same but do not use the same wording.
- Vary your verbs. Do not reuse "force is applied" or "force is present",
  can use verbs like drives, pushes, stirs, sits, hums, acts, hits, works 
  away, was placed, etc.
- Within one request, do not use the same magnitude word twice, use a 
  different synonym instead.
- Damping is GLOBAL, so phrase it as being applied on the whole disk.
- Keep the position EXACTLY as coarse as stated. Never create a more
  precise position, never make it vaguer, never change direction words.
- Never use digits, and say clock hours using words, and ALWAYS add a 
  disambiguating word so it reads like a direction, not a time of day. 
  (eg: use "the two o'clock direction", "toward two o'clock", "at two 
  o'clock on the dial", but "at two o'clock" is WRONG.
\end{verbatim}

\textbf{User message} 
\begin{verbatim}
Facts:
1. forcing position: the north-west quadrant, close to the rim |
   forcing magnitude: weak |
   damping: the disk as a whole is barely damped 
2. forcing position: the four o'clock direction, halfway towards the edge |
   forcing magnitude: medium-strength |
   damping: the system as a whole is clearly damped ...
\end{verbatim}
\label{fig:helm_prompt}}

\subsubsection{Experimental Setup}

\begin{figure}[t]
    \begin{center}
    \includegraphics[scale=0.45]{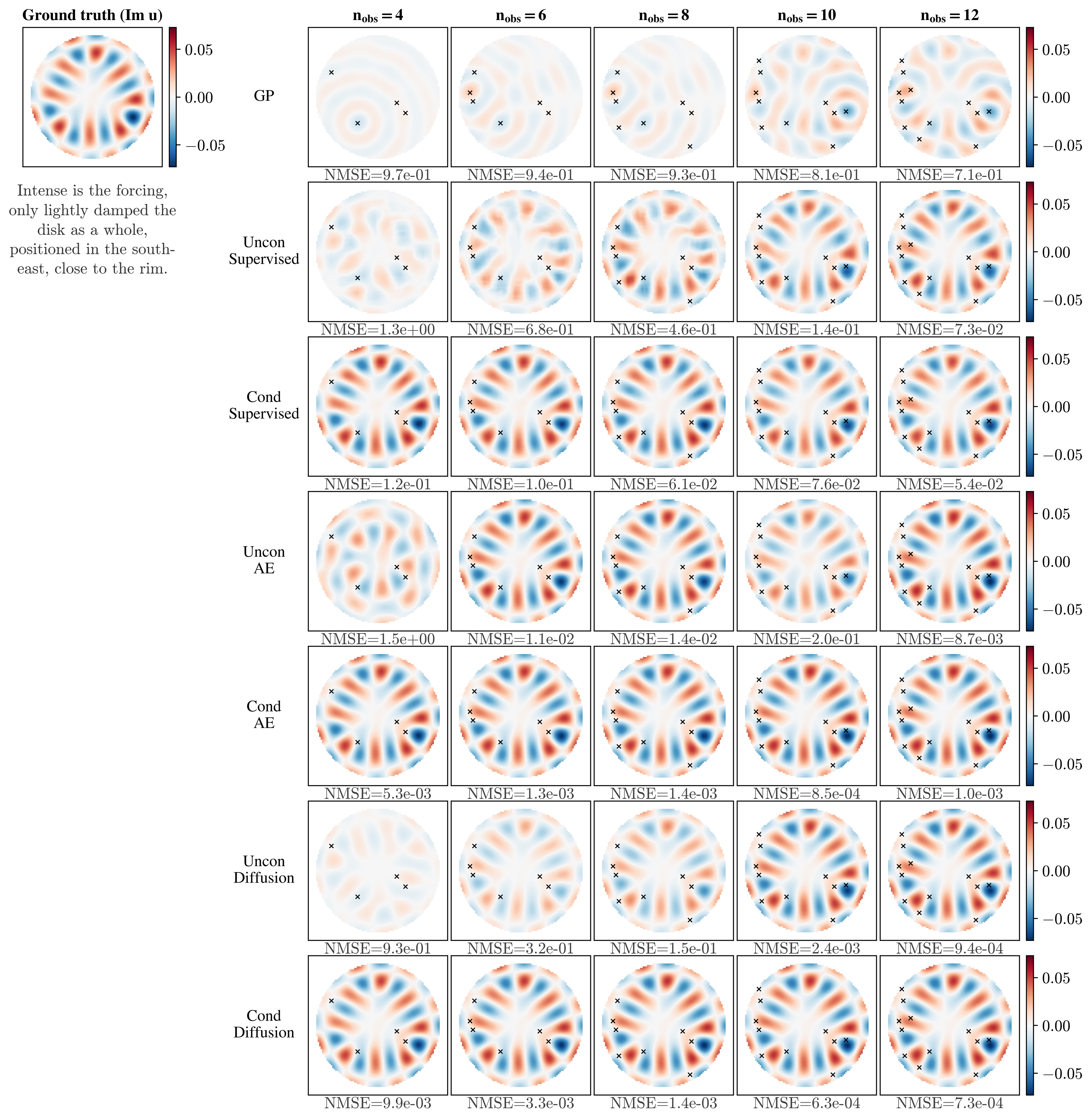}
    \end{center}
    \caption{Posterior inference of one Helmholtz equation test system (Imaginary part) as the number of observation points increases.}
    \label{fig:helm_obs_graph_2}
\end{figure}

\paragraph{Fine-tuning.} Damped Helmholtz fields oscillate, so an $L_2$ distance as in (\ref{eq:heat_score}) is a poor measure of how far apart two fields are. Instead, we compare the pointwise energy $E(x)=|u(x)|^2$ calculated from real and imaginary channels, and use the sliced Wasserstein distance between normalised energy fields \smash{$\tilde E(x)=E(x)/\|u\|_2^2$}, where $\|u\|_2^2=\sum_x E(x)$ is the total energy. Since the normalisation discards the magnitude, we append the difference in log total energy,
\begin{equation}
    \label{eq:helm_score}
    \mathrm{score}_{\mathrm{field}}(u,u')
    =\exp\Bigl(-\frac{c_1\,\mathrm{SW}_2^2(\tilde E,\tilde E')
        +c_2\,\bigl|\log\|u\|_2^2-\log\|u'\|_2^2\bigr|^2}{T}\Bigr),
\end{equation}
where $\mathrm{SW}^2_2(\bullet,\bullet)$ measures the squared sliced Wasserstein distance between two normalised energy fields, $c_1$ and $c_2$ scale the two terms to similar magnitudes, and $T=340.02$ is the bandwidth. 3 epochs of fine-tuning \texttt{all-MiniLM-L6-v2} are performed on $50,000$ training pairs with learning rate $5\times 10^{-5}$.

\begin{figure}[t]
    \begin{center}
    \includegraphics[scale=0.45]{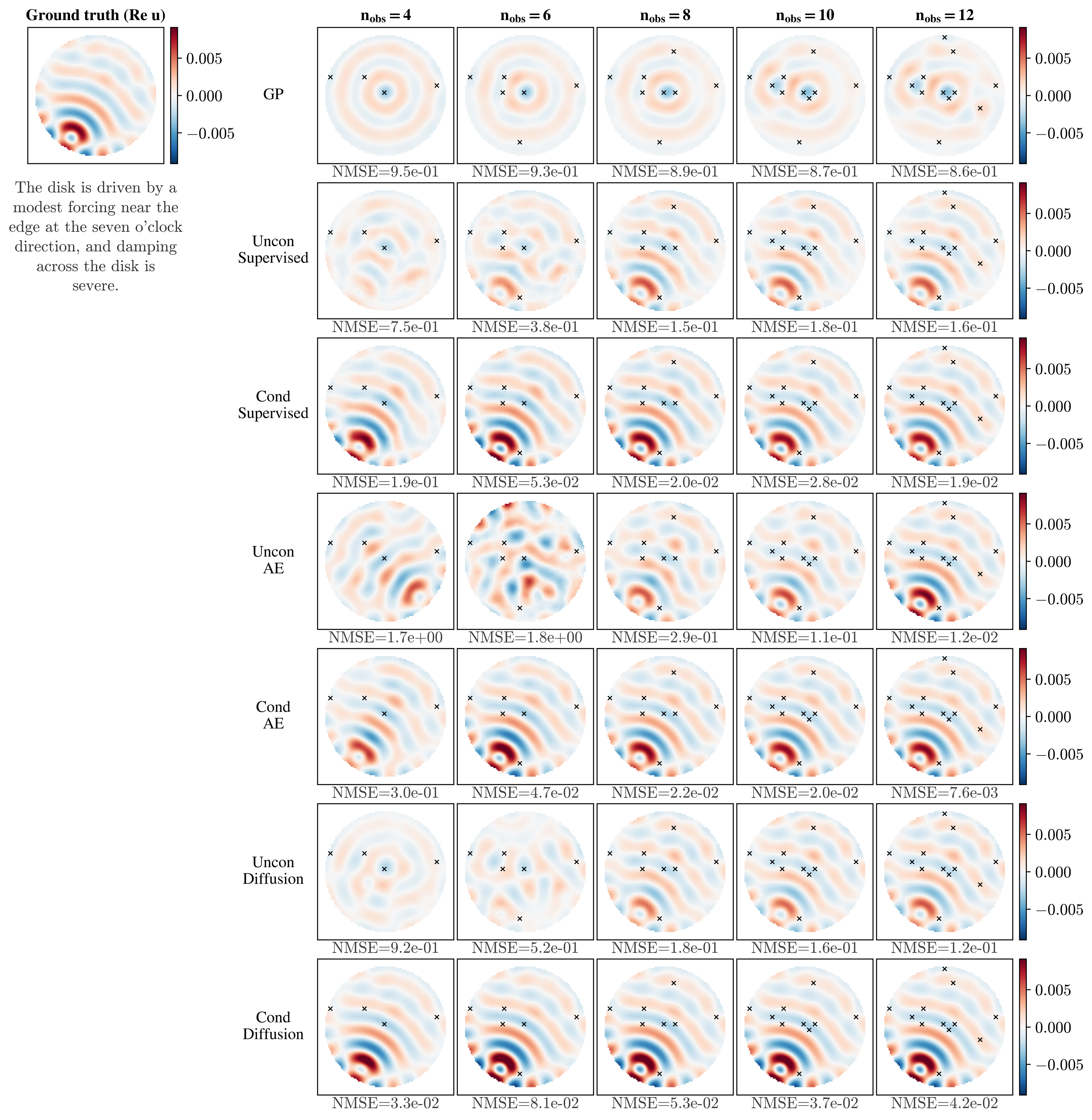}
    \end{center}
    \caption{Posterior inference of one Helmholtz equation test system (Real part) as the number of observation points increases.}
    \label{fig:helm_obs_graph}
\end{figure}

\paragraph{Train.} Supervised baselines use a U-Net with $96$ base channels, \textbf{Cond Supervised} is trained for $2k$ epochs, while \textbf{Uncon Supervised} for $2.5k$ epochs, with a learning rate of $10^{-3}$ and batch size of $32$; AE uses FNO with $4$ layers (each for $\sfE$ and $\sfD$), width
$56$ and $14$ modes, latent dimension is $16$, \textbf{Cond AE} is trained for $4k$ epochs and \textbf{Uncon AE} for $6k$, with a learning rate of $10^{-3}$ and batch size $200$; the score network is a U-Net with $80$ base channels and a $128$-dimensional conditioning vector, \textbf{Cond Diffusion} is trained for $3k$ epochs and \textbf{Uncon Diffusion} for $4k$, with learning rate of $2 \times 10^{-4}$ and batch size of $32$.

\paragraph{Inference.} At inference, observations are drawn randomly from both real and imaginary fields, with the same noise standard deviation $\sigma_y = 6 \times 10^{-4}$. Following \citet{caviedes2021gaussian}, \textbf{GP} uses a modified Bessel kernel
\begin{equation}
    \label{eq:helm_gp_kernel}
    k(x,x')=\sigma_f^2\exp\Bigl(-\frac{\|x-x'\|_2^2}{2\ell^2}\Bigr)
            J_0\bigl(k_0\|x-x'\|_2\bigr),
\end{equation}
where \smash{$\sigma_f^2 J_0\bigl(k_0\|x-x'\|_2\bigr)$} is the Bessel kernel, modulated by an exponential envelope that accounts for the damping in our case, where \smash{$k_0=\sqrt{\kappa}= \sqrt{200}$}, and $\sigma_f$ and $\ell$ are hyperparameters to be learned using maximum likelihood estimation. AE posteriors are sampled using a single NUTS chain with $M_o=100$ warmup steps and $M=200$ retained draws. The maximum tree depth is $10$ for both \textbf{Cond AE} and \textbf{Uncon AE}. Diffusion posteriors are drawn by integrating the guided reverse SDE over $L=800$ uniform steps and the guidance clip $c=1.0$.

\subsubsection{Additional Results} \label{appen:helm_results}

\paragraph{Inference Results.} Table~\ref{tab:helm_inference_full} gives the full inference results across a variety of $n_\text{obs}$. Again, the last two rows report standard DPS, using a tuned guidance strength $\zeta'=0.3$. Standard DPS reaches an NMSE comparable to the modified sampler but provides worse posterior coverage. From both Table~\ref{tab:helm_inference_full} and posterior histograms in Figure~\ref{fig:helm_posterior_histograms}, conditional generative models provide more accurate inference results. Figure~\ref{fig:helm_coverage} shows coverage distributions over $100$ test systems. \textbf{Cond AE} is the best calibrated overall, with the most systems concentrated near the nominal levels (median $\%C_{68}=72.6\%$ and $\%C_{95}=97.9\%$). \textbf{Uncon Diffusion} reports average coverages close to nominal values in Table~\ref{tab:helm_inference_full}, but Figure~\ref{fig:helm_coverage}(b) shows this comes from $23\%$ of systems reaching $\%C_{95}=100\%$, i.e. from posteriors that are far too dispersed rather than from
calibration. \textbf{Uncon Diffusion} also has higher CRPS values than \textbf{Cond Diffusion}.

\paragraph{Prior Generation.} Figure~\ref{fig:helm_prior} provides a few prior draws from the four generative models. With the text supplied, conditional priors concentrate on fields consistent with the description, showing the low damping and the forcing near the northern rim, while unconditional draws vary freely over the whole training distribution.

\paragraph{Ablation on number of observation points and text contribution.} Figures~\ref{fig:helm_obs_graph_2} and~\ref{fig:helm_obs_graph} visualise inference results as $n_\text{obs}$ progresses. \textbf{Cond Supervised}, \textbf{Cond AE} and \textbf{Cond Diffusion} all perform better than unconditional ones, and the margin
is largest when observations alone are least informative. Figure~\ref{fig:helm_text_pair} plots NMSE means and quartiles, and per-system text gain across different $n_\text{obs}$, with the corresponding NMSE values in
Table~\ref{tab:helm_nmse_quantiles}. The gain shrinks as observation data accumulate but does not vanish: even at $n_\text{obs}=12$ the median gain is about $2.3\times$ for the autoencoder and $1.3\times$ for diffusion, so the information from text still contributes.

\subsection{UK Weather Data}

\subsubsection{Problem Setup}

\begin{figure}[t]
    \begin{center}
    \includegraphics[scale=0.4]{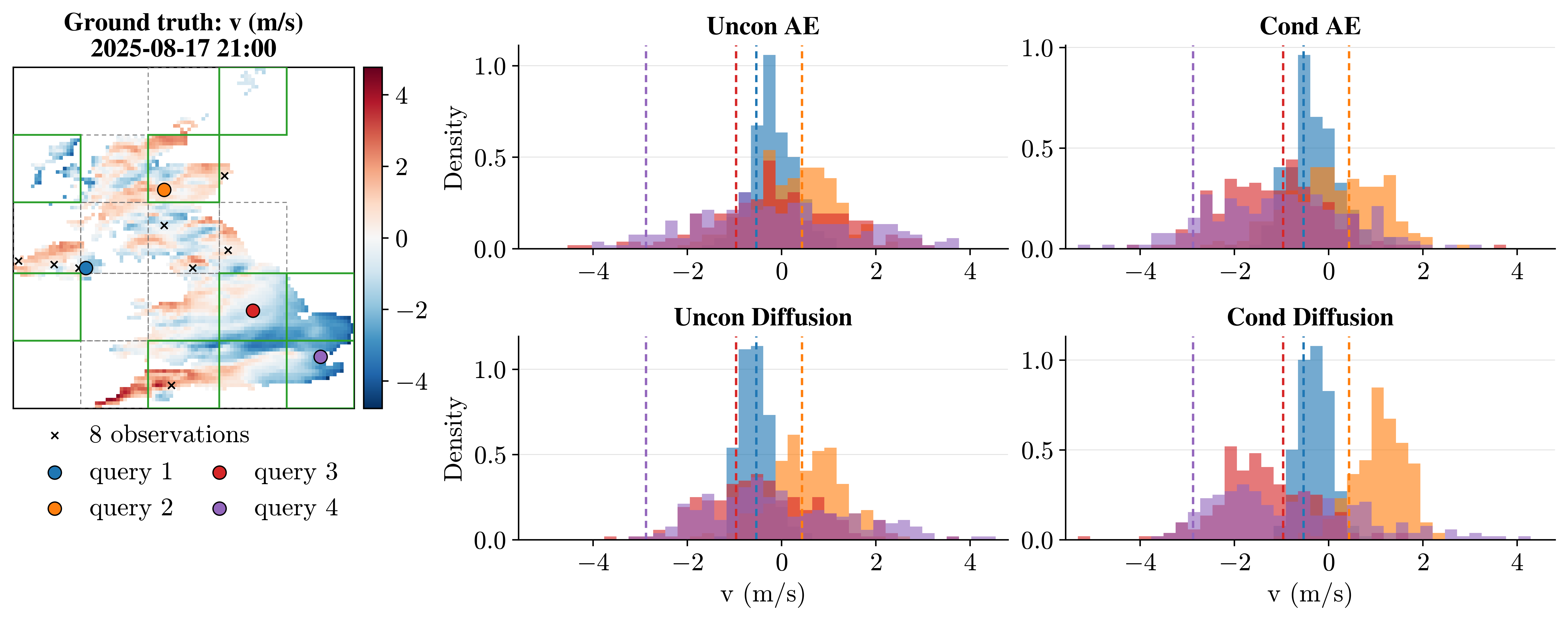}
    \end{center}
    \caption{Posterior marginals of one UK weather test system ($v$) (same system as in Figure~\ref{fig:uk_single_system_686}) at four unobserved query locations when $n_{\mathrm{\text{obs}}}=8$.}
    \label{fig:uk_posterior_histograms_686}
\end{figure}

\begin{table}[t]
  \caption{\textbf{Binning of the regional averages for text generation in UK weather dataset.} Bands overlap at their edges: a value falling in two bands is described by one of them, chosen at random.}
  \label{tab:climate_bins}
  \centering
  \begin{tabular}{@{}C{5.2em}C{4.2em}C{5.2em}C{4.6em}@{}}
    \toprule
    \multicolumn{2}{c}{Temperature (\textdegree C)}
      & \multicolumn{2}{c}{Wind speed (m\,s$^{-1}$)} \\
    \cmidrule(lr){1-2}\cmidrule(lr){3-4}
    Level & Range & Level & Range \\
    \midrule
    freezing  & $<3$        & calm     & $<2.2$       \\
    cold      & $2$--$6$    & light    & $1.8$--$4.0$ \\
    chilly    & $5$--$9$    & moderate & $3.5$--$5.5$ \\
    cool      & $8$--$12$   & breezy   & $5.0$--$7.5$ \\
    mild      & $11$--$15$  & strong   & $>7.0$       \\
    warm      & $14$--$19$  &          &              \\
    very warm & $18$--$24$  &          &              \\
    hot       & $>24$       &          &              \\
    \bottomrule
  \end{tabular}
\end{table}

We download dataset from CERRA \citep{ridal2024cerra}, which is a regional reanalysis dataset for a domain covering all Europe from 1984. The original dataset is at a horizontal resolution of $5.5km$ on a Lambert conformal grid. We restrict it to UK land by intersecting the CERRA land-sea mask with the Natural Earth country boundary, and then resample the retained points onto a regular $96\times96$ latitude-longitude grid spanning $49.99^\circ$--$60.75^\circ$\,N and $8.08^\circ$\,W--$1.74^\circ$\,E. Of the $9216$ grid points, $3135$ fall on land, and the remainder are set to zero and excluded from every loss and metric.

Three channels are obtained from CERRA, \textit{2m temperature}, \textit{10m wind speed} and \textit{10m wind direction}, with the latter two converted to horizontal and vertical wind velocities, making the final three channels $\{ T,u,v\}$. We take three daily analysis times $09{:}00$, $15{:}00$ and $21{:}00$ from 2000 to 2025, giving $28{,}491$ total snapshots, of which the $27{,}396$ from 2000-2024 are used for training.

To generate text for this dataset, we first divide the grid into $25$ regions, A1-E5, of which $18$ contain enough land to be described, as in Figure~\ref{fig:uk_visualisation}. The text corresponding to each field describes the temperature, wind speed and direction in $8$ randomly chosen regions, mimicking a setting in which much of the domain carries no sensors but people can report what the weather feels like. As in the Helmholtz experiment, the regional averages are first binned into levels shown in Table~\ref{tab:climate_bins}. For wind speed, averages are calculated as the mean of pointwise speed \smash{$\sqrt{u^2+v^2}$}. Its direction is calculated from the mean wind velocity $(\bar u,\bar v)$ and reported as one of eight compass points (eg: north-west).

We again query \textit{DeepSeek-V4-Flash} API to generate text, with the system prompt and user message provided below. Same as in Helmholtz experiment, each bin carries a pool of anchor words, one of which is drawn and passed to the model, and the prompt asks for a synonym. 

\begin{table}[t]
  \centering
  \caption{Full inference results on 100 UK weather test systems, which is an extension of Table \ref{tab:uk_inference}.}
  \label{tab:uk_inference_full}
  \setlength{\tabcolsep}{3pt}
  \resizebox{\textwidth}{!}{%
  \begin{tabular}{l |
    S[table-format=2.2]@{\,$\pm$\,}S[table-format=2.2]
    S[table-format=2.2]@{\,$\pm$\,}S[table-format=2.2]
    S[table-format=2.1]@{\,}c
    S[table-format=2.1]@{\,}c |
    S[table-format=2.2]@{\,$\pm$\,}S[table-format=2.2]
    S[table-format=2.2]@{\,$\pm$\,}S[table-format=2.2]
    S[table-format=2.1]@{\,}c
    S[table-format=2.1]@{\,}c |
    S[table-format=2.2]@{\,$\pm$\,}S[table-format=2.2]
    S[table-format=2.2]@{\,$\pm$\,}S[table-format=2.2]
    S[table-format=2.1]@{\,}c
    S[table-format=2.1]@{\,}c}
    \toprule
    & \multicolumn{8}{c|}{$n_{\text{obs}}=4$}
    & \multicolumn{8}{c|}{$n_{\text{obs}}=6$}
    & \multicolumn{8}{c}{$n_{\text{obs}}=8$} \\
    \cmidrule(lr){2-9} \cmidrule(lr){10-17} \cmidrule(lr){18-25}
    \textbf{Model}
    & \multicolumn{2}{c}{\makecell{NMSE $\downarrow$\\ $(\times 10^{-2})$}}
      & \multicolumn{2}{c}{\makecell{CRPS $\downarrow$\\ $(\times 10^{-2})$}}
      & \multicolumn{2}{c}{\makecell{$\% C_{68}$}}
      & \multicolumn{2}{c|}{\makecell{$\%C_{95}$}}
    & \multicolumn{2}{c}{\makecell{NMSE $\downarrow$\\ $(\times 10^{-2})$}}
      & \multicolumn{2}{c}{\makecell{CRPS $\downarrow$\\ $(\times 10^{-2})$}}
      & \multicolumn{2}{c}{\makecell{$\% C_{68}$}}
      & \multicolumn{2}{c|}{\makecell{$\%C_{95}$}}
    & \multicolumn{2}{c}{\makecell{NMSE $\downarrow$\\ $(\times 10^{-2})$}}
      & \multicolumn{2}{c}{\makecell{CRPS $\downarrow$\\ $(\times 10^{-2})$}}
      & \multicolumn{2}{c}{\makecell{$\% C_{68}$}}
      & \multicolumn{2}{c}{\makecell{$\%C_{95}$}} \\
    \midrule
    GP
      & 15.05 & 19.33 & 18.16 & 11.44 & 70.0 & ( 9.5) & 92.6 & ( 4.7)
      & 13.10 & 16.83 & 16.71 & 10.37 & 71.0 & ( 9.4) & 93.0 & ( 4.3)
      & 12.12 & 16.06 & 15.91 & 10.12 & 68.3 & ( 8.5) & 91.7 & ( 5.0) \\
    Uncon Supervised
      & 10.81 & 14.02 & \multicolumn{2}{c}{--} & \multicolumn{2}{c}{--} & \multicolumn{2}{c|}{--}
      &  8.76 & 11.89 & \multicolumn{2}{c}{--} & \multicolumn{2}{c}{--} & \multicolumn{2}{c|}{--}
      &  7.46 & 11.38 & \multicolumn{2}{c}{--} & \multicolumn{2}{c}{--} & \multicolumn{2}{c}{--} \\
    Cond Supervised
      &  8.96 & 12.98 & \multicolumn{2}{c}{--} & \multicolumn{2}{c}{--} & \multicolumn{2}{c|}{--}
      &  7.51 & 10.54 & \multicolumn{2}{c}{--} & \multicolumn{2}{c}{--} & \multicolumn{2}{c|}{--}
      &  6.78 &  9.95 & \multicolumn{2}{c}{--} & \multicolumn{2}{c}{--} & \multicolumn{2}{c}{--} \\
    Uncon AE
      & 11.30 & 15.16 & 15.21 &  9.38 & 69.5 & ( 8.0) & 94.6 & ( 3.5)
      &  9.28 & 12.43 & 13.57 &  8.50 & 68.6 & ( 7.3) & 94.3 & ( 3.2)
      &  8.18 & 11.55 & 12.58 &  8.02 & 67.5 & ( 6.5) & 93.6 & ( 3.3) \\
    Cond AE
      &  9.27 & 13.45 & 13.70 &  8.67 & 66.3 & ( 6.7) & 92.6 & ( 3.9)
      &  7.68 & 10.44 & 12.45 &  7.64 & 66.4 & ( 6.0) & 92.8 & ( 3.4)
      &  6.93 &  9.81 & 11.68 &  7.37 & 65.9 & ( 5.6) & 92.5 & ( 3.5) \\
    Uncon Diffusion
      & 10.49 & 12.87 & 14.79 &  8.85 & 55.6 & (13.5) & 85.9 & ( 9.3)
      &  8.81 & 12.16 & 13.14 &  8.53 & 56.1 & (12.8) & 86.4 & ( 8.7)
      &  7.72 & 11.25 & 12.07 &  8.06 & 56.2 & (12.4) & 86.5 & ( 8.7) \\
    Cond Diffusion
      &  9.31 & 15.15 & 13.55 &  9.37 & 55.8 & (12.9) & 85.9 & ( 9.2)
      &  7.49 & 10.55 & 12.22 &  7.86 & 56.0 & (12.5) & 86.0 & ( 9.0)
      &  6.81 & 10.13 & 11.45 &  7.65 & 56.0 & (12.5) & 86.0 & ( 9.1) \\
    \midrule
    & \multicolumn{8}{c|}{$n_{\text{obs}}=10$}
    & \multicolumn{8}{c|}{$n_{\text{obs}}=12$}
    & \multicolumn{8}{c}{$n_{\text{obs}}=14$} \\
    \cmidrule(lr){2-9} \cmidrule(lr){10-17} \cmidrule(lr){18-25}
    \textbf{Model}
    & \multicolumn{2}{c}{\makecell{NMSE $\downarrow$\\ $(\times 10^{-2})$}}
      & \multicolumn{2}{c}{\makecell{CRPS $\downarrow$\\ $(\times 10^{-2})$}}
      & \multicolumn{2}{c}{\makecell{$\% C_{68}$}}
      & \multicolumn{2}{c|}{\makecell{$\%C_{95}$}}
    & \multicolumn{2}{c}{\makecell{NMSE $\downarrow$\\ $(\times 10^{-2})$}}
      & \multicolumn{2}{c}{\makecell{CRPS $\downarrow$\\ $(\times 10^{-2})$}}
      & \multicolumn{2}{c}{\makecell{$\% C_{68}$}}
      & \multicolumn{2}{c|}{\makecell{$\%C_{95}$}}
    & \multicolumn{2}{c}{\makecell{NMSE $\downarrow$\\ $(\times 10^{-2})$}}
      & \multicolumn{2}{c}{\makecell{CRPS $\downarrow$\\ $(\times 10^{-2})$}}
      & \multicolumn{2}{c}{\makecell{$\% C_{68}$}}
      & \multicolumn{2}{c}{\makecell{$\%C_{95}$}} \\
    \midrule
    GP
      & 10.78 & 14.30 & 14.98 &  9.36 & 69.3 & ( 8.2) & 92.1 & ( 4.6)
      & 10.16 & 13.80 & 14.43 &  9.08 & 69.7 & ( 8.1) & 92.1 & ( 4.6)
      &  9.48 & 13.31 & 13.90 &  8.80 & 69.3 & ( 7.9) & 91.9 & ( 4.5) \\
    Uncon Supervised
      &  6.52 & 10.26 & \multicolumn{2}{c}{--} & \multicolumn{2}{c}{--} & \multicolumn{2}{c|}{--}
      &  6.22 & 10.47 & \multicolumn{2}{c}{--} & \multicolumn{2}{c}{--} & \multicolumn{2}{c|}{--}
      &  5.80 &  9.93 & \multicolumn{2}{c}{--} & \multicolumn{2}{c}{--} & \multicolumn{2}{c}{--} \\
    Cond Supervised
      &  6.26 &  9.53 & \multicolumn{2}{c}{--} & \multicolumn{2}{c}{--} & \multicolumn{2}{c|}{--}
      &  5.89 &  8.81 & \multicolumn{2}{c}{--} & \multicolumn{2}{c}{--} & \multicolumn{2}{c|}{--}
      &  5.67 &  8.72 & \multicolumn{2}{c}{--} & \multicolumn{2}{c}{--} & \multicolumn{2}{c}{--} \\
    Uncon AE
      &  7.07 & 10.04 & 11.66 &  7.37 & 66.9 & ( 6.2) & 93.2 & ( 3.2)
      &  6.90 & 10.50 & 11.38 &  7.41 & 65.8 & ( 5.9) & 92.3 & ( 3.6)
      &  6.51 & 10.17 & 10.99 &  7.22 & 64.8 & ( 5.9) & 91.6 & ( 4.1) \\
    Cond AE
      &  6.22 &  8.94 & 11.07 &  6.90 & 65.0 & ( 6.1) & 91.8 & ( 3.9)
      &  6.12 &  9.58 & 10.82 &  6.93 & 63.8 & ( 6.6) & 90.9 & ( 4.5)
      &  5.87 &  9.01 & 10.56 &  6.76 & 62.5 & ( 6.9) & 90.3 & ( 5.0) \\
    Uncon Diffusion
      &  6.43 &  9.54 & 11.10 &  7.23 & 56.2 & (12.2) & 86.5 & ( 8.5)
      &  6.32 & 11.66 & 10.68 &  7.36 & 56.3 & (12.2) & 86.5 & ( 8.5)
      &  5.97 & 11.83 & 10.25 &  7.22 & 56.3 & (12.1) & 86.6 & ( 8.5) \\
    Cond Diffusion
      &  6.03 &  8.88 & 10.76 &  7.03 & 55.8 & (12.6) & 85.7 & ( 9.3)
      &  6.05 & 11.02 & 10.45 &  7.13 & 55.7 & (12.6) & 85.8 & ( 9.2)
      &  5.69 & 10.45 & 10.05 &  6.90 & 55.9 & (12.4) & 85.8 & ( 9.2) \\
    \bottomrule
  \end{tabular}%
  }
\end{table}

\begin{figure}[t]
    \begin{center}
    \includegraphics[scale=0.4]{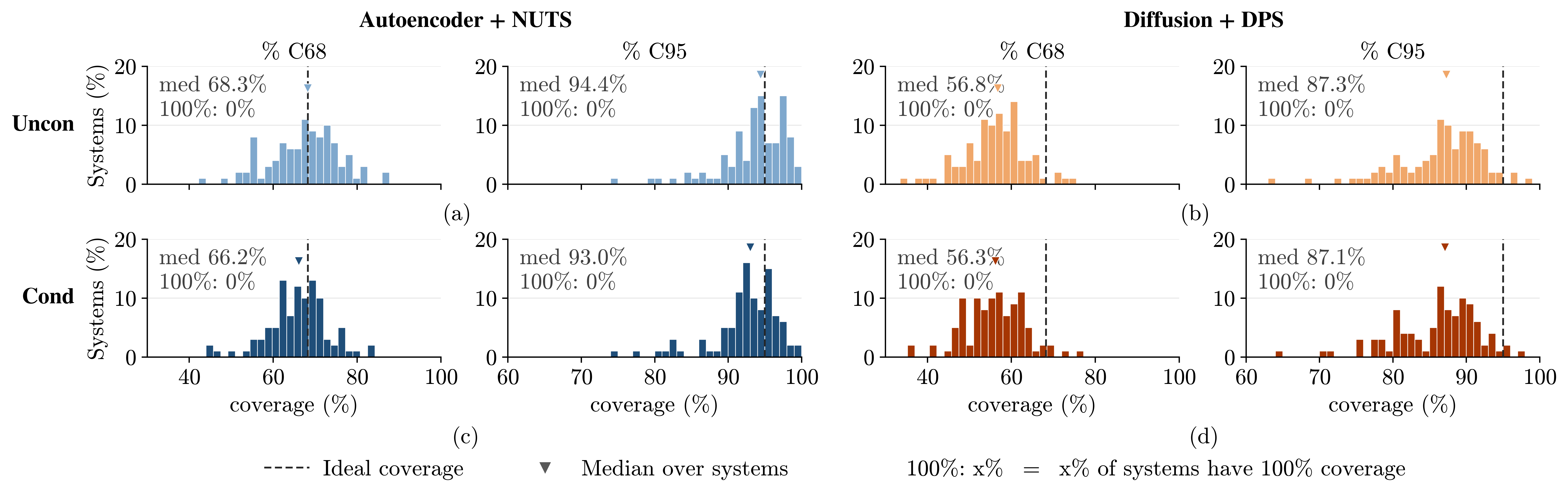}
    \end{center}
    \caption{Per-system coverage over 100 Weather test fields with $n_{\text{obs}}=8$. }
    \label{fig:uk_coverage}
\end{figure}

\begin{figure}[t]
    \begin{center}
    \includegraphics[scale=0.42]{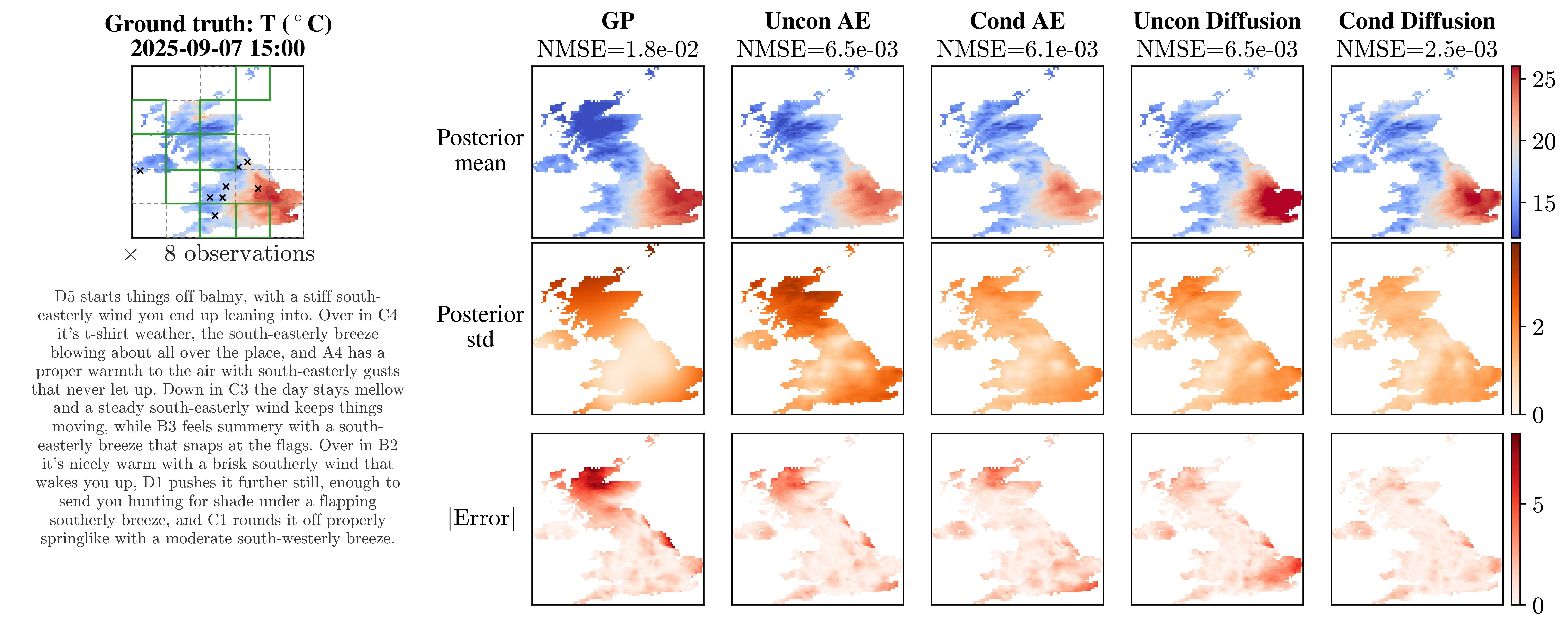}
    \end{center}
    \caption{Posterior reconstructions of one UK Weather test system ($T$) when $n_{\mathrm{\text{obs}}}=8$.}
    \label{fig:uk_single_system_748}
\end{figure}

\begin{figure}[t]
    \begin{center}
    \includegraphics[scale=0.4]{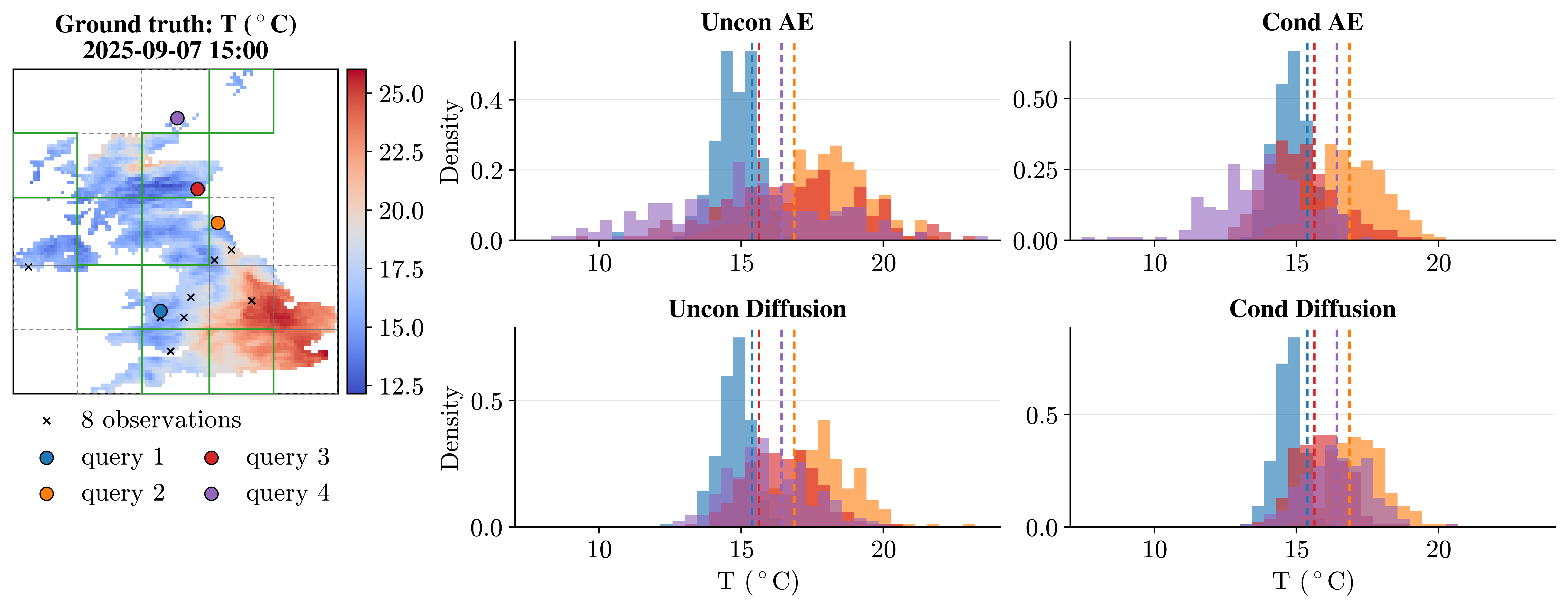}
    \end{center}
    \caption{Posterior marginals of one UK weather test system ($T$) (same system as in Figure~\ref{fig:uk_single_system_748}) at four unobserved query locations when $n_{\mathrm{\text{obs}}}=8$.}
    \label{fig:uk_posterior_histograms_748}
\end{figure}

{\footnotesize
\textbf{System prompt}
\begin{verbatim}
You write short weather descriptions as a PERSON describing how current 
weather feels like in the United Kingdom, and use relaxed, everyday 
language in flowing sentences. NOT a forecast/bulletin: avoid
jargon ('highs/lows of', 'mph', 'outbreaks', 'winds veering').

The UK is divided into a grid of areas, each with a short LABEL: the
letter runs from west to east and the number from south to north. Each 
system below gives a DIFFERENT set of areas, listed in a consistent 
geographic order. For each you get its TEMPERATURE band,
WIND-STRENGTH band, and WIND DIRECTION.

Full ordered band scale (your reference for meaning + neighbours):
  TEMPERATURE bands (approx degC, overlap at edges), coldest -> hottest:
    freezing (below 3C)  <  cold (2-6C)  <  chilly (5-9C)  < ...
  WIND-SPEED bands (approx m/s, overlap at edges):
    calm (below 2.2 m/s)  <  light (1.8-4 m/s)  < ...

For each variant write ONE description covering three facts of EVERY 
listed area. Rules:
  - Describe the areas IN THE ORDER GIVEN in user message so the 
    description flows geographically. NAME each area by its grid label.
  - ONE anchor phrase is provided for each region. The anchor fixes the
    INTENSITY. But please change using synonym, using words which person
    would use. Pasting exactly the anchor word is the BORING fallback. 
    NEVER shift the intensity to a different band.
  - VARY the sentence structure and the fact order from area to area. 
    Never produce a chain of 'X is <word> with a <word> wind'.
  - Wind directions should always end with a NOUN ('wind', 'breeze', 
    'gusts', etc.). A 'variable' wind direction indicates a light or 
    shifting wind.
  - When several areas share the same direction, describe the shared 
    airflow ONCE and then only describe each area by its temperature and 
    wind strength.
  - DO NOT use comparatives (cooler, windier, 'more...') and similarity 
    references ('just as', 'the same wind', 'likewise') to compare 
    between areas.
  - Never create sentences describing rain, sunshine, etc.
\end{verbatim}}

{\footnotesize
\textbf{User message}
\begin{verbatim}
Facts (grid label -- temp band | wind band | direction):
  - B1: agreeable | a persistent breeze | southerly
  - B2: raw | moderate | south-westerly ...
\end{verbatim}}

\subsubsection{Experimental Setup}

\begin{figure}[t]
    \begin{center}
    \includegraphics[scale=0.42]{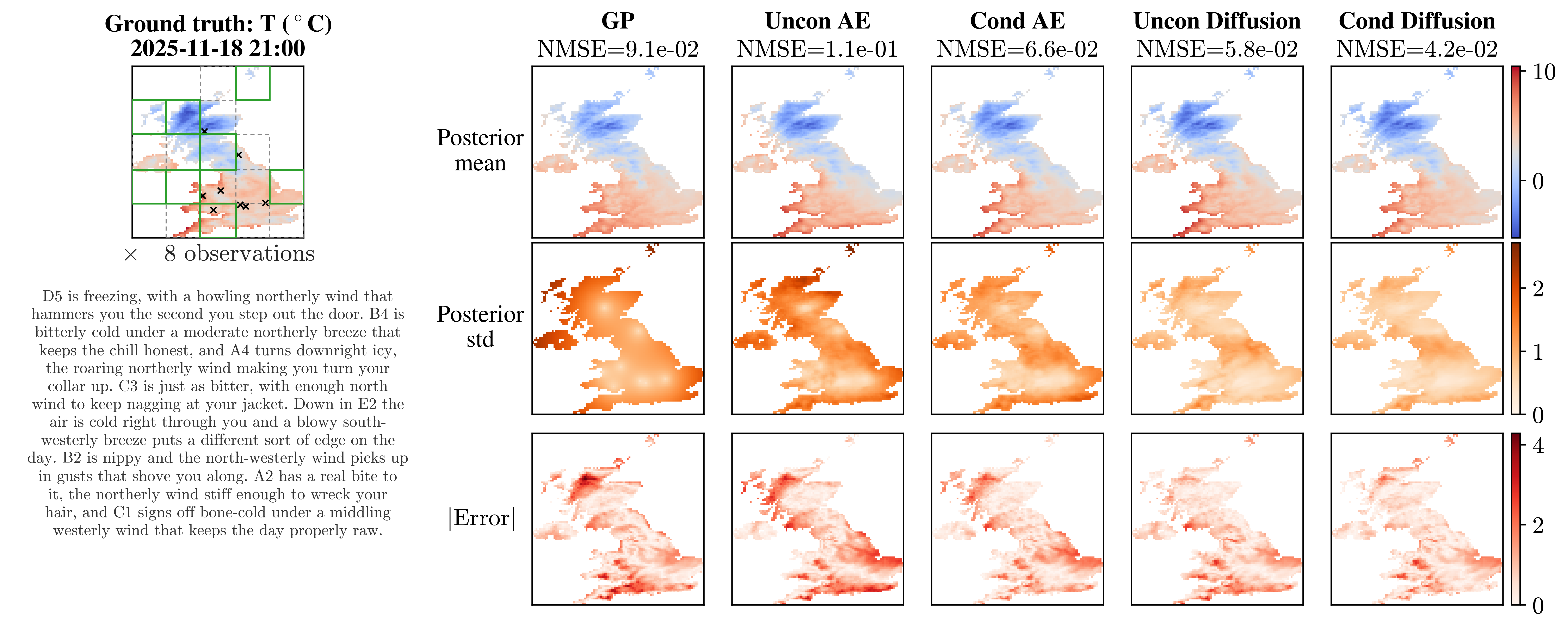}
    \end{center}
    \caption{Posterior reconstructions of one UK Weather test system ($T$) when $n_{\mathrm{\text{obs}}}=8$.} 
    \label{fig:uk_single_system_965}
\end{figure}

\begin{figure}[t]
    \begin{center}
    \includegraphics[scale=0.4]{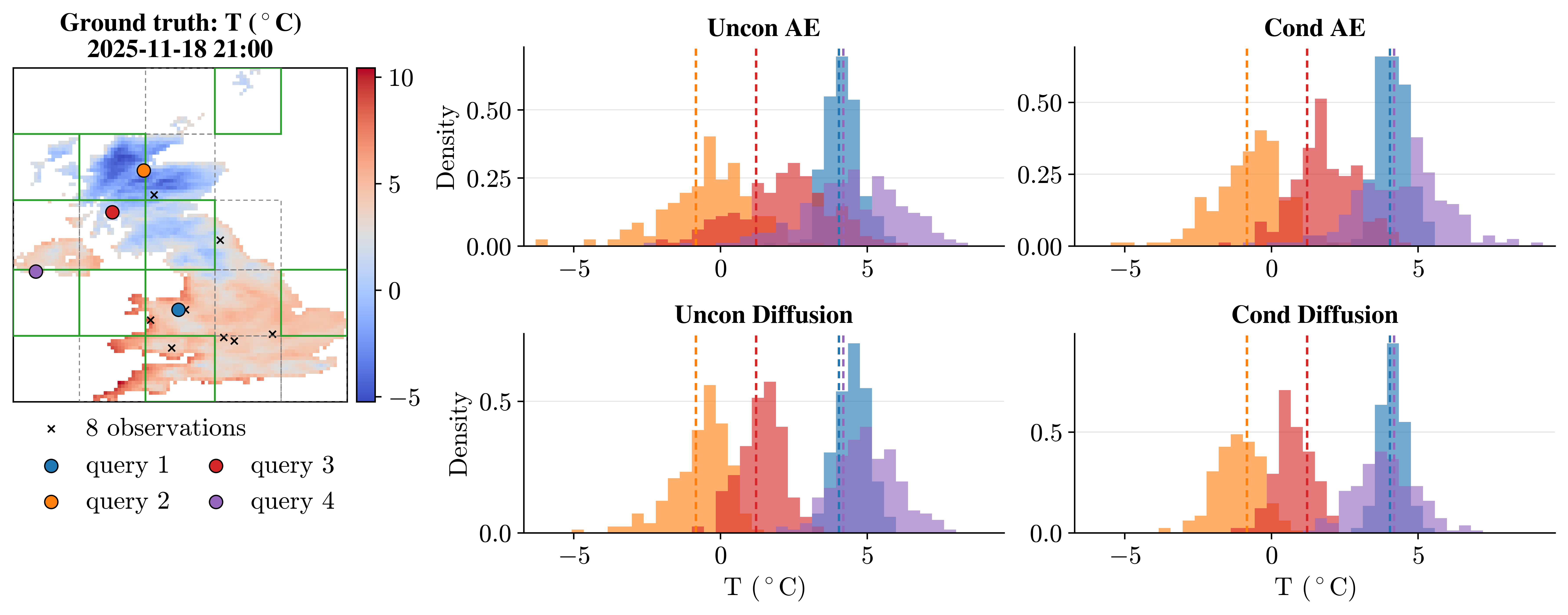}
    \end{center}
    \caption{Posterior marginals of one UK weather system ($T$) (same system as in Figure~\ref{fig:uk_single_system_965}) at four unobserved query locations when $n_{\mathrm{\text{obs}}}=8$.}
    \label{fig:uk_posterior_histograms_965}
\end{figure}

\paragraph{Fine-tuning.} Each text description in this experiment describes a random $8$ of $18$ regions covering UK land, so $L_2$ distance between whole fields is not suitable to be matched to distance between text embeddings. Let $A$ and $B$ be the region subsets described by text $s_A$ and $s_B$, let $\tilde u$ be the stack of the three channels $\{T, u, v\}$ after z-score normalisation, and let $\tilde u|_A$ represent the field values over regions covered in $A$. Suppose the field $\tilde u$ has corresponding text $s_A$ and $\tilde u'$ with $s_B$. Then
\begin{subequations}
    \begin{align}
        \label{eq:uk_score}
        \text{score}_\text{field}(\tilde u, \tilde u') &= \exp\left(-\frac{d^2}{\beta}\right), \\
        d^2 &= \tfrac12\left(d_A^2 + d_B^2\right), \\
        d_A^2 &= \bigl\|\tilde u|_A - \tilde u'|_A\bigr\|^2_2 , \\
        d_B^2 &= \bigl\|\tilde u|_B - \tilde u'|_B\bigr\|^2_2 ,
    \end{align}
\end{subequations}
where $d_A$ calculates the $L_2$ distance between two fields over regions described by $s_A$, and $d_B$ over those $s_B$ describes. $\beta=1.2 \times 10^4$ is the bandwidth. 2 epochs of fine-tuning \texttt{all-MiniLM-L6-v2} are performed on $80,000$ training pairs with learning rate $2\times 10^{-5}$.

\paragraph{Train.} Supervised baselines use a U-Net with $96$ base channels, \textbf{Cond Supervised} is trained for $1k$ epochs, while \textbf{Uncon Supervised} for $1.5k$ epochs, with a learning rate of $10^{-3}$ and batch size of $32$; AE uses CNN-based autoencoder with $96$ base channels, latent dimension is $64$, \textbf{Cond AE} is trained for $3k$ epochs and \textbf{Uncon AE} for $3.5k$, with a learning rate of $10^{-3}$ and batch size $200$; the score network is a U-Net with $96$ base channels and a $256$-dimensional conditioning vector, \textbf{Cond Diffusion} is trained for $1.5k$ epochs and \textbf{Uncon Diffusion} for $2.5k$, with learning rate of $2 \times 10^{-4}$ and batch size of $32$.

\begin{figure}[t]
    \begin{center}
    \includegraphics[scale=0.45]{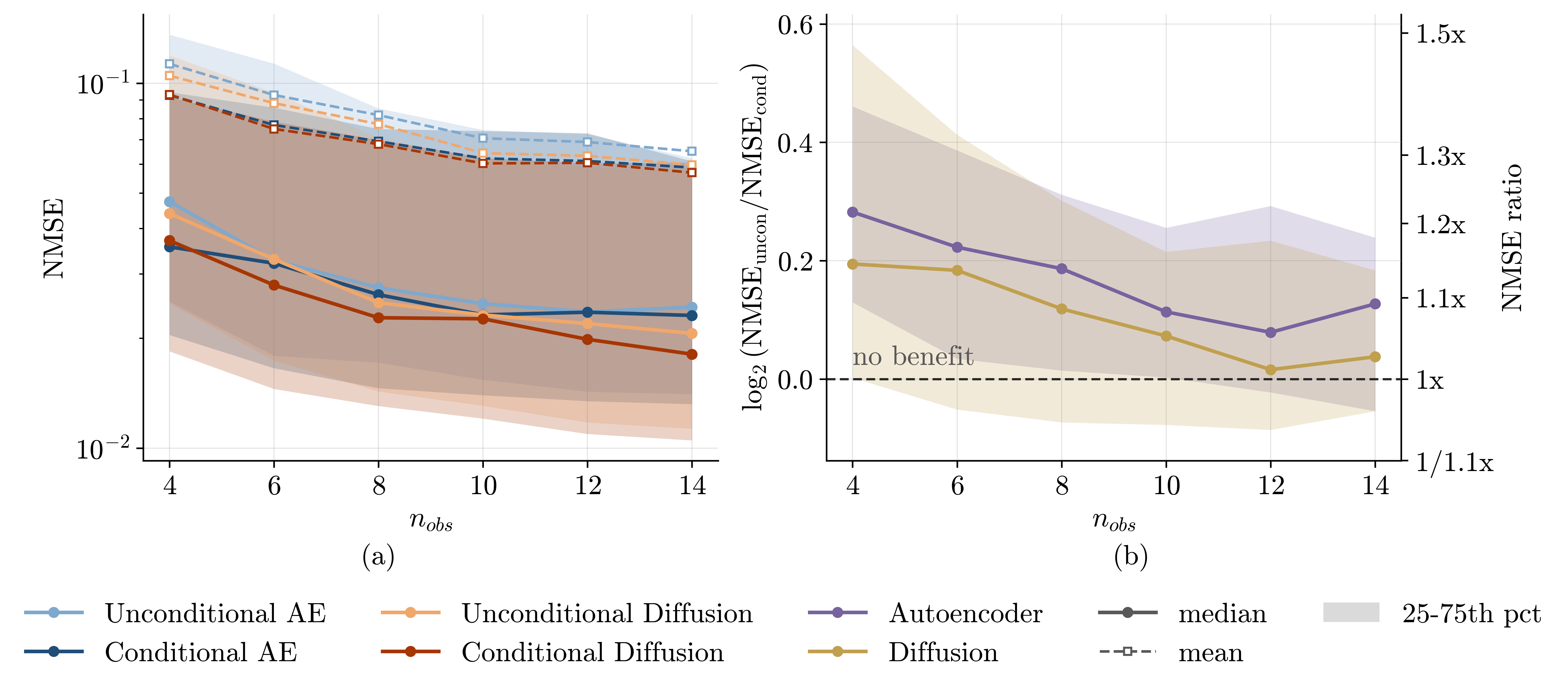}
    \end{center}
    \caption{\textbf{NMSE performance at different $\mathbf{n_{obs}}$ for UK Weather data.} (a) NMSE across 100 test systems: median (solid), mean (dashed), and 25th–75th percentile bands, see Table~\ref{tab:uk_nmse_quantiles} for exact values. (b) Per-system relative gain from text conditioning, $\log_2(  \mathrm{NMSE}_{\mathrm{uncon}}/\mathrm{NMSE}_{\mathrm{cond}}) $, with median and interquartile bands.}
    \label{fig:uk_text_pair}
\end{figure}

\begin{table}[t]
  \centering
  \caption{Quartiles of per-system NMSE across 100 UK weather slices.}
  \label{tab:uk_nmse_quantiles}
  \setlength{\tabcolsep}{4pt}
  \resizebox{\textwidth}{!}{%
  \begin{tabular}{l | ccc | ccc | ccc | ccc | ccc | ccc}
    \toprule
    & \multicolumn{3}{c|}{$n_{\text{obs}}=4$}
    & \multicolumn{3}{c|}{$n_{\text{obs}}=6$}
    & \multicolumn{3}{c|}{$n_{\text{obs}}=8$}
    & \multicolumn{3}{c|}{$n_{\text{obs}}=10$}
    & \multicolumn{3}{c|}{$n_{\text{obs}}=12$}
    & \multicolumn{3}{c}{$n_{\text{obs}}=14$} \\
    \cmidrule(lr){2-4} \cmidrule(lr){5-7} \cmidrule(lr){8-10}
    \cmidrule(lr){11-13} \cmidrule(lr){14-16} \cmidrule(lr){17-19}
    \makecell[l]{\textbf{Model}\\ NMSE $(\times 10^{-2})$}
      & $Q_{25}$ & Median & $Q_{75}$
      & $Q_{25}$ & Median & $Q_{75}$
      & $Q_{25}$ & Median & $Q_{75}$
      & $Q_{25}$ & Median & $Q_{75}$
      & $Q_{25}$ & Median & $Q_{75}$
      & $Q_{25}$ & Median & $Q_{75}$ \\
    \midrule
    Uncon AE
      & 2.52 & 4.74 & 13.58
      & 1.79 & 3.27 & 11.32
      & 1.71 & 2.75 & 8.53
      & 1.54 & 2.49 & 7.45
      & 1.43 & 2.36 & 7.25
      & 1.41 & 2.43 & 6.23 \\
    Cond AE
      & 2.04 & 3.56 & 9.45
      & 1.66 & 3.21 & 8.57
      & 1.46 & 2.63 & 7.49
      & 1.40 & 2.32 & 7.40
      & 1.34 & 2.36 & 7.29
      & 1.32 & 2.31 & 6.11 \\
    Uncon Diffusion
      & 2.49 & 4.39 & 11.91
      & 1.74 & 3.29 & 9.42
      & 1.43 & 2.50 & 7.09
      & 1.31 & 2.32 & 6.20
      & 1.18 & 2.19 & 5.88
      & 1.13 & 2.06 & 5.77 \\
    Cond Diffusion
      & 1.84 & 3.70 & 9.15
      & 1.45 & 2.80 & 7.85
      & 1.30 & 2.28 & 6.99
      & 1.20 & 2.26 & 6.25
      & 1.09 & 1.99 & 6.03
      & 1.05 & 1.81 & 6.07 \\
    \bottomrule
  \end{tabular}%
  }
\end{table}

\paragraph{Inference.} At inference, all three channels are observed at a common set of land locations, with
$\sigma_y = \{0.5, 0.3, 0.3\}$ for $\{T, u, v\}$. More observations are obtained in regions not covered by text. For the \textbf{GP} baseline we use the exponential kernel, i.e.\ Mat\'ern with $\nu=1/2$,
\begin{equation}
    k(x,x') = \sigma_f^2 \exp\!\left(-\frac{\|x-x'\|_2}{\ell}\right),
\end{equation}
where $\sigma_f$ and $\ell$ are hyperparameters learned separately for three channels. AE posteriors are sampled using a single NUTS chain with $M_o=100$ warmup steps and $M=200$ retained draws. The maximum tree depth is $10$ for both \textbf{Cond AE} and \textbf{Uncon AE}. Diffusion posteriors are drawn by integrating the guided reverse SDE over $L=1000$ uniform steps and the guidance clip $c=1.0$.

\subsubsection{Additional Results}  \label{appen:uk_results}
Table~\ref{tab:uk_inference_full} provides full inference results. Figure~\ref{fig:uk_coverage} shows distribution of coverages, which are again better for autoencoders, with medians closer to nominal ($68.3\%$ and $94.4\%$ for \textbf{Uncon AE}, $66.2\%$ and $93.0\%$ for \textbf{Cond AE}), whereas diffusion models yield systematically low coverages at both levels ($\sim56\%$ and $\sim87\%$), indicating that inference is overconfident. Figures~\ref{fig:uk_posterior_histograms_686},~\ref{fig:uk_single_system_748},~\ref{fig:uk_posterior_histograms_748},~\ref{fig:uk_single_system_965},~\ref{fig:uk_posterior_histograms_965} show posterior reconstructions and posterior histograms for several systems. Figure~\ref{fig:uk_text_pair} with Table~\ref{tab:uk_nmse_quantiles} compares NMSE and text gain across various $n_\text{obs}$ as before. The gain is smaller than in the other two
PDE experiments, but the median stays above $1\times$ throughout, roughly $1.1\times$ for the autoencoder at $n_\text{obs}=14$, so the text still contributes when the observations are dense. Figure~\ref{fig:uk_region_gain} separates text gain by region, which is larger inside the regions the text describes (solid) than outside them (dashed).

To test whether supervised models are constrained to the observation process they were trained on, we further design an ablation study which uses a different observation operator at inference. In our main experiments, three channels $\{T,u,v\}$ share the same observation locations in both training and inference, which vary from system to system but are common to all channels. In the ablation, $T$ is observed at $8$ locations and $u,v$ at $8$ different ones, so the
number of observations is unchanged and only the observation operator differs. Table~\ref{tab:uk_disjoint} reports results over $20$ test fields, where both conditional generative models overtake the supervised baselines.

\begin{figure}[t!]
    \begin{center}
    \includegraphics[scale=0.45]{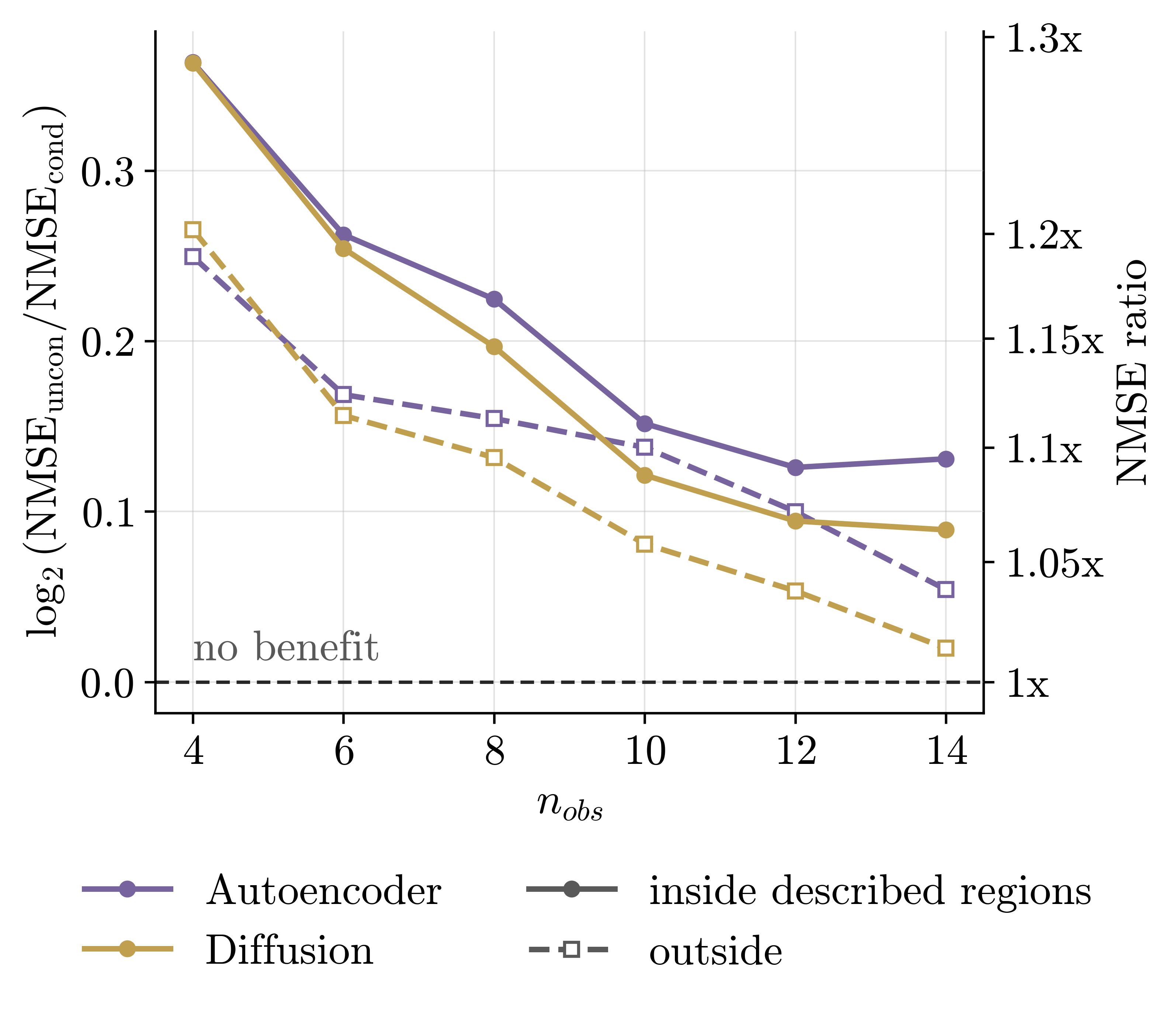}
    \end{center}
    \caption{\textbf{Mean of per-system text gain from conditioning for UK weather test set}. The figure plots the text gain $\log_2(\mathrm{NMSE}_{\text{uncon}}/\mathrm{NMSE}_{\text{cond}})$, averaged over $100$ test systems, inside the text-described regions (solid) and outside them (dashed).}
    \label{fig:uk_region_gain}
\end{figure}

\begin{table}[t]
  \centering
  \caption{\textbf{Ablation on a different observation operator of the UK weather experiment.} Temperature is measured at $8$ locations and the two wind components at $8$ different locations, which is a sensor layout the supervised models were not trained on. Results are over $20$ test systems.}
  \label{tab:uk_disjoint}
  \begin{tabular}{l c c c}
    \toprule
    \textbf{Model} & NMSE $(\times 10^{-2})$
      & $\%C_{68}$ & $\%C_{95}$ \\
    \midrule
    Uncon Supervised & 5.95 $\pm$ 5.31 & --           & --           \\
    Cond Supervised  & 5.11 $\pm$ 4.36 & --           & --           \\
    Uncon AE         & 5.22 $\pm$ 4.86 & 68.9\,(4.6)  & 94.0\,(2.3)  \\
    Cond AE          & 4.41 $\pm$ 4.01 & 65.9\,(5.6)  & 92.3\,(3.6)  \\
    Uncon Diffusion  & 5.04 $\pm$ 4.72 & 56.7\,(11.6) & 86.3\,(9.0)  \\
    Cond Diffusion   & 4.20 $\pm$ 3.70 & 57.2\,(11.5) & 86.7\,(8.5)  \\
    \bottomrule
  \end{tabular}
\end{table}

\end{document}